\documentclass[11pt,reqno]{amsart}

\usepackage[margin=1in]{geometry}

\usepackage[utf8]{inputenc} %
\usepackage[T1]{fontenc}    %

\usepackage{xcolor}
\definecolor{colblue}{HTML}{0B8BD5}
\definecolor{colred}{HTML}{df2026}
\definecolor{colgreen}{HTML}{079348}

 \usepackage[colorlinks=true,
  linkcolor=colred,
  citecolor=colgreen,
  urlcolor=colblue
]{hyperref}      %
\usepackage{url}            %
\usepackage{booktabs}       %
\usepackage{amsfonts,amssymb,amsthm,amsmath,mathtools}       %
\usepackage{nicefrac}       %
\usepackage{microtype}      %
\usepackage{graphicx}       %
\usepackage{xcolor}         %
\usepackage[capitalize]{cleveref}
 
\usepackage{wrapfig}
\usepackage{tikz}
\usetikzlibrary{calc}
\usepackage[noend]{algpseudocode}
\usepackage{algorithm}
\usepackage[shortlabels]{enumitem}
\usepackage{cite}
\usepackage{caption}
\hypersetup{hypertexnames=false}

\title{Sliced Inner Product Gromov--Wasserstein Distances}

\theoremstyle{plain}
\newtheorem{lemma}{Lemma}[section]
\newtheorem{theorem}[lemma]{Theorem}
\newtheorem{corollary}[lemma]{Corollary}
\newtheorem{proposition}[lemma]{Proposition}
\newtheorem{remark}[lemma]{Remark}

\theoremstyle{definition}
\newtheorem{definition}[lemma]{Definition}

\newtheorem{example}[lemma]{Example}

\DeclareMathOperator*{\argmin}{arg\,min}
\DeclareMathOperator*{\argmax}{arg\,max} 

\DeclareMathOperator{\conv}{conv}

\DeclareMathOperator{\supp}{spt}

\newcommand{\R}{\mathbb{R}}

\newcommand{\dd}{{d}} %

\newcommand{\unitsph}{\mathbb{S}^{d-1}}

\newcommand{\unitsphy}{\mathbb{S}^{d_y-1}}

\newcommand{\cD}{\mathcal{D}}

\newcommand{\cN}{\mathcal{N}}

\newcommand{\cP}{\mathcal{P}}

\DeclareMathOperator{\id}{id}

\newcommand{\bA}{\mathbf{A}}

\newcommand{\bD}{\mathbf{D}}

\newcommand{\bG}{\mathbf{G}}

\newcommand{\bI}{\mathbf{I}}
\newcommand{\bJ}{\mathbf{J}}
\newcommand{\bK}{\mathbf{K}}

\newcommand{\bP}{\mathbf{P}}
\newcommand{\bQ}{\mathbf{Q}}
\newcommand{\bR}{\mathbf{R}}
\newcommand{\bS}{\mathbf{S}}

\newcommand{\bU}{\mathbf{U}}
\newcommand{\bV}{\mathbf{V}}

\newcommand{\bsigma}{\mathbf{\Sigma}}
\newcommand{\bdelta}{\mathbf{\Delta}}
\newcommand{\bgamma}{\mathbf{\Gamma}}
\newcommand{\blambda}{\mathbf{\Lambda}}

\newcommand{\sF}{\mathsf{F}}

\newcommand{\St}{\mathrm{St}}

\newcommand{\EE}{\mathbb{E}}

\newcommand{\RR}{\mathbb{R}}

\newcommand{\IGW}{\mathsf{IGW}}
\newcommand{\W}{\mathsf{W}}
\newcommand{\Wp}{\mathsf{W}_p}

\newcommand{\llangle}{\left\langle}
\newcommand{\rrangle}{\right\rangle}

\newcommand{\op}{\mathrm{op}}
\newcommand{\F}{\mathrm{F}}

\newcommand{\Tr}{\mathrm{Tr}} 
\newcommand{\Var}{\mathrm{Var}}

\newcommand{\aIGW}{\underline{\mathsf{IGW}}}
\newcommand{\aW}{\underline{\mathsf{W}}}

\DeclareMathOperator{\diag}{diag}

\newcommand{\bm}[1]{\mathbf #1}

\newcommand{\mathbbm}[1]{\bm #1}

\thanks{\textsuperscript{\dag} Denots equal contribution to this work.}

\thanks{
Z. Goldfeld is partially supported by NSF grants CCF-1947801,  CCF-2046018, and DMS-2210368, and the 2020 IBM Academic Award.}

\author[X. Gong]{Xiaoyun Gong\textsuperscript{\dag} }
\address[X. Gong]{
Center for Applied Mathematics, Cornell University.}
\email{xg332@cornell.edu}
\author[G. Rioux]{Gabriel Rioux\textsuperscript{\dag} }
\address[G. Rioux]{
Department of Mathematics, Imperial College London.}
\email{g.rioux@ic.ac.uk}

\author[Z. Goldfeld]{Ziv Goldfeld}
\address[Z. Goldfeld]{
School of Electrical and Computer Engineering, Cornell University.
}
\email{goldfeld@cornell.edu}

\begin{document}
\begin{abstract}
  The Gromov--Wasserstein (GW) problem provides a framework for aligning heterogeneous datasets by matching their intrinsic geometry, but its statistical and computational scaling remains an issue for  high-dimensional problems. Slicing techniques offer an appealing route to scalability, but, unlike  Wasserstein distances, GW problems do not generally admit closed-form solutions in one-dimension. We resolve this problem for the GW problem with inner product cost (IGW), propose a sliced IGW distance that enjoys a natural rotational invariance property, and comprehensively study its structural and computational properties.  Numerical experiments validating our theory are presented, followed by  applications to heterogeneous clustering of text data and language model representation comparison.
\end{abstract}

\maketitle

\section{Introduction}
The Gromov--Wasserstein (GW) problem enables aligning heterogeneous probability spaces by matching their intrinsic geometry, without the need for shared coordinates or pointwise labels~\cite{memoli2011gromov}. 
This versatility has made GW central to structure-aware matching and inference problems, including network/graph comparison \cite{chowdhury2019gromov}, latent correspondence learning \cite{le2022entropic,rioux2024entropic}, and graph-isomorphism testing \cite{rioux2024limitlaws}, among others. 
However, scaling GW alignment to modern regimes remains challenging as empirical rates exhibit the curse of dimensionality \cite{zhang2024gromov}, while computation generally involves solving  nonconvex quadratic problems \cite{chowdhury2019gromov,memoli2011gromov}.

An effective route to obtaining scalable statistical distances is via \emph{slicing}, i.e., aggregating distance values across one-dimensional projections. 
For Wasserstein distances, this strategy succeeds as each 1D subproblem admits a simple closed-form solution  \cite{santambrogio2015optimal}. 
This tractability underlies broad applications in computing barycenters \cite{rabin2011wasserstein}, generative modeling \cite{deshpande2019max,deshpande2018generative,li2022sliced,nadjahi2019asymptotic}, and representation learning \cite{adhya2025s2wtm,kolouri2018sliced,naderializadeh2025constrained}, along with a growing statistical and algorithmic theory \cite{boedihardjo2025sharp,bonet2025sliced,manole2022minimax,nadjahi2020statistical,nadjahi2021fast,nietert2022statistical} %
More recent work has explored constructing original transport from sliced transport \cite{liu2024expected,nguyen2025fast},
identifying informative slicing directions \cite{liu2025efficient,naderializadeh2025constrained,nguyen2022amortized,nguyen2023markovian}, as well as leveraging sliced Wasserstein distances as part of the objective in other algorithms/tasks \cite{chen2022unsupervised,lu2024slosh,nadjahi2020approximate,shahbazi2025espformer}.
By contrast, GW lacks a closed-form solution  in one dimension, limiting the development of a parallel slicing framework.%

\subsection{Contributions}
We consider the GW problem with inner product cost (abbreviated IGW), between probability distributions $\mu\in\mathcal P(\mathbb R^{d_x})$ and $\nu\in\mathcal P(\mathbb R^{d_y})$, supported in Euclidean spaces:
\begin{equation}
\label{eq:IGW}
   \mathsf{IGW}(\mu,\nu)\coloneqq \left(\inf_{\pi\in\Pi(\mu,\nu)}\iint \left|\llangle x,x'\rrangle-\llangle y,y'\rrangle\right|^2d\pi\otimes\pi(x,y,x',y')\right)^{1/2},   
\end{equation}
where $\Pi(\mu,\nu)$ is the set of all couplings of $\mu$ and $\nu$. IGW quantifies the least amount of inner product distortion achievable when optimizing over all possible correspondences of the two spaces. As such, it captures a meaningful notion of discrepancy between heterogeneous spaces with a natural inner product structure, which motivated its adoption in practice. When $d_x=d_y=1$, we first establish a closed-form solution for this problem which reduces to sorting if $\mu$ and $\nu$ are uniformly distributed on the same (finite) number of points, akin to the Wasserstein case. This unlocks a slicing paradigm for IGW which is the main focus of this work.

The natural construction of the average-sliced IGW (namely, averaging IGW values across one-dimensional projections), is not necessarily invariant to orthogonal transformations, which is a key property of IGW. To address this, inspired by \cite{vayer2019sliced}, we define  a rotationally invariant sliced IGW, denoted $\aIGW$ (see \cref{def:slicedIGW}), 
by introducing an additional outer minimization over orthogonal transformations of one of the marginals. With this, $\aIGW$ defines a pseudometric on the space of probability measures on $\mathbb R^d$ and is topologically equivalent to IGW under certain structural conditions. %
As $\aIGW$ involves an optimization over the Stiefel manifold of an objective defined in terms of an integral over the unit sphere, we  study Monte Carlo averaging as a proxy for integrating over the sphere and propose a subgradient method for optimizing over the Stiefel manifold. Non-asymptotic guarantees are provided for both approximations, resulting in a principled computation framework for average-sliced IGW. The paper concludes with numerical experiments validating the derived theory and providing applications of this framework to comparing embeddings of LLMS and clustering of heterogeneous datasets.

\subsection{Related Literature} The IGW distance, introduced in \cite{vayer2020contribution}, has been applied for various machine learning tasks, including latent correspondence learning \cite{le2022entropic}, LLM representation comparison and heterogeneous clustering of text data \cite{dandapanthula2025optimal}, cross-lingual alignment \cite{aramayo2023uncovering}, among others. Its popularity stems, in part, from its nice analytic properties. Indeed, in contrast to the standard GW distance with quadratic cost, closed-form solutions for IGW with Gaussian marginals are known \cite{dandapanthula2025optimal,delon2022gromov,le2022entropic} and the existence of solutions induced by maps holds under mild conditions \cite{dumont2024existence}. 

The computational burden associated with solving GW problems (in particular, IGW) motivated the study of computationally efficient approximate solvers. This includes approaches based on entropic regularization \cite{karumanchi2025approximation,peyre2016gromov,rioux2024entropic,scetbon2022linear,solomon2016entropic}, convex/conic relaxations \cite{sejourne2021unbalanced,vincent2021semi}, semi-definite relaxations \cite{chen2023semidefinite}, and direct optimization over bi-directional (parametrized) Gromov--Monge maps \cite{zhang2022cycle}. \cite{vayer2019sliced} proposed a sliced quadratic GW distance and studied its metric and topological properties. However, this approach is contingent on a closed-form solution for quadratic GW (i.e., with squared norms in place of inner products); the formula provided in  \cite{vayer2019sliced} was later shown to be invalid via a counterexample \cite{beinert2022assignment}.

The motivation to unlock slicing for GW stems from the practical success of sliced Wasserstein distances \cite{rabin2011wasserstein}, which  overcome the statistical and computational issues of standard Wasserstein distance \cite{nadjahi2020statistical}. Average-slicing is typically implemented via Monte Carlo integration \cite{kolouri2019generalized,nadjahi2020statistical}, for which quantitative error bounds were provided in \cite{nietert2022statistical}. From the statistical standpoint, sliced Wasserstein distances enjoy fast dimension-independent convergence rates  \cite{lin2021projection,manole2022minimax,nadjahi2020statistical,nietert2022statistical,weed2022estimation}, %
as well as a comprehensive limit distribution theory \cite{goldfeld2024statistical,manole2022minimax,xi2022distributional,xu2022central}. Given their attractive algorithmic and statistical profile, sliced Wasserstein distances have been broadly applied for machine learning tasks such as generative modeling \cite{deshpande2019max,deshpande2018generative,li2022sliced,nadjahi2019asymptotic}, barycenter computation \cite{rabin2011wasserstein}, autoencoders \cite{adhya2025s2wtm,kolouri2018sliced}, locality-sensitive hashing \cite{lu2024slosh}, and representation learning \cite{naderializadeh2025constrained}.

\section{Background and Preliminaries}
We briefly review definitions and preliminary results concerning the IGW problem. Let $\|\cdot\|$ and $\llangle \cdot,\cdot\rrangle$ denote the Euclidean norm and inner product. The operator and Frobenius norms of a matrix $\bA\in\RR^{d\times k}$ are denoted by $\|\bA\|_\op$ and $\|\bA\|_\F$, respectively. %
We write $\mathbf I_d\in\mathbb R^{d\times d}$ for the identity matrix. For $d_x\leq d_y$, $\St(d_x,d_y)\coloneqq \left\{\bdelta \in\mathbb R^{d_y\times d_x}:\bdelta^{\intercal}\bdelta=\mathbf{I}_{d_x}\right\}$ is the Stiefel manifold of orthonormal $d_x$-frames in $\RR^{d_y}$. Write $\cP(\RR^d)$ for the space of Borel probability measures over $\RR^d$, and let $\cP_p(\RR^d)$ be its restriction to distributions with finite $p$-th absolute moments. %
For $\mu\in\cP(\RR^d)$, we denote its support by $\supp(\mu)$, while $\bsigma_\mu$, $\mathbf R_{\mu}$ and $M_p(\mu)$ denote the covariance matrix, second moment matrix, and $p$-th absolute moment of $\mu$, respectively. We write $T_\sharp \mu$ for the pushforward measure of $\mu$ through a measurable map $T$. %
For $\theta\in\RR^d$, we identify $\theta^{\intercal}$ with function $x\mapsto \theta^{\intercal}x$, and write $(\theta^{\intercal})_{\sharp}\mu$ for the corresponding pushforward. Weak convergence of probability measures is  denoted as $\nu_n\stackrel{w}{\to}\nu$. %
We use $\lesssim_x$ to denote inequalities up to constants that only depend on $x$; the subscript is dropped when the constant is universal.

\subsection{Inner Product Gromov--Wasserstein Distance}
Throughout, we analyze the GW distance with inner product cost (IGW) between $\mu\in\mathcal P(\RR^{d_x})$ and  $\nu\in\mathcal P(\RR^{d_y})$ as defined in \eqref{eq:IGW}.
Notably, $\mathsf{IGW}$ is invariant under orthogonal transformations (but is not translation-invariant~\cite{le2022entropic}) and defines a pseudometric in the sense that it is symmetric, satisfies the triangle inequality, and is nonnegative with $\IGW(\mu,\nu)=0$ if and only if there exists a $\mu\otimes\mu$-a.s. unitary isomorphism $T:\supp(\mu)\to\supp(\nu)$ with $T_\sharp\mu=\nu$ \cite[Proposition 3.1]{zhang2024gradient}. %
It is known that the IGW distance can be connected to a certain class of parametrized optimal transport (OT) problems as follows; see 
\cite[Appendix E]{rioux2024entropic} and \cite[Lemma 2.1]{zhang2024gradient}. 

\begin{lemma}[IGW duality]\label{lem:F2Dual}
Fix  $(\mu,\nu)\in\cP_2(\RR^{d_x})\times\cP_2(\RR^{d_y})$, and  $M_{\mu,\nu}:=\sqrt{M_2(\mu)M_2(\nu)}$. Then,  
    \begin{equation}
    \mathsf{IGW}^2(\mu,\nu)=\|\mathbf R_{\mu}\|_{\mathrm{F}}^2+\|\mathbf R_{\nu}\|_{\mathrm{F}}^2+\inf_{\mathbf{A}\in\RR^{d_x\times d_y}} 8\|\mathbf{A}\|_\F^2+\mathsf{IOT}_{\mathbf{A}}(\mu,\nu),\label{eq:F2Decomp}
    \end{equation}
    where $\mathsf{IOT}_{\bA}(\mu,\nu)\coloneqq \inf_{\pi\in\Pi(\mu,\nu)}\int c_\bA d\pi$ is the OT problem with cost function $c_{\mathbf{A}}\mspace{-3mu}:\mspace{-3mu}   
    (x,y)\in\RR^{d_x}\times \RR^{d_y}\mapsto\mspace{-3mu}-8x^{\intercal}\mathbf{A}y$ and the infimum is achieved at some $\mathbf{A}^{\star}\mspace{-3mu}\in\mspace{-3mu}\cD_{M_{\mu,\nu}}\mspace{-3mu}\coloneqq\mspace{-3mu}[-M_{\mu,\nu}/2,M_{\mu,\nu}/2]^{d_x\times d_y}$. 
\end{lemma}

\subsection{Comparison and Topological Properties}
The IGW problem is comparable to the 2-Wasserstein distance, $\W_2$, as stated in \cite[Lemma 3.2]{zhang2024gradient}; 
\begin{equation}
    \IGW(\mu,\nu) \leq \big(2 M_2(\mu)+2M_2(\nu) \big)^{\frac{1}{2}} \W_2(\mu,\nu),\quad\forall \mu,\nu\in\cP_2(\RR^d),\label{eq:W2_upperbound}
\end{equation}
where $\W_2(\mu,\nu)^2\coloneqq \inf_{\pi\in\Pi(\mu,\nu)}\int \|x-y\|^2d\pi(x,y)$. 
The cited lemma  also provides a lower bound in terms of $\W_2$ (up to an orthogonal transformation of one of the marginals) when $\bsigma_\mu$ and $\bsigma_\nu$ are nonsingular. Our subsequent derivations require a stronger lower bound that removes this nonsingularity condition; see \cref{appen:lem:equivalence_igw_w_proof} for the~proof. 
\begin{proposition}[IGW and Wasserstein comparison]\label{lem:equivalence_igw_w}
     Suppose that $\mu\in\cP_{2}(\RR^{d_x})$ and $\nu\in\cP_{2}(\RR^{d_y})$ are not point masses. Letting $\mathbf P\mathbf \Lambda_{\pi^{\star}}\mathbf Q^{\intercal}$ be the singular value decomposition of $\int xy^{\intercal}d\pi^{\star}(x,y)$ for some solution $\pi^{\star}$ of $\mathsf{IGW}(\mu,\nu)$ with  $\bP\in\mathbb R^{d_x\times d_x}$, and $\bQ\in\mathbb R^{d_y\times d_y}$, we have that 
     \[
     \sqrt{\kappa}\mathsf{W}_2((\bm\Delta\bP^{\intercal})_{\sharp} \mu,(\bm Q^{\intercal})_{\sharp}\nu)
   \leq \mathsf{IGW}(\mu,\nu),
   \]
    where  $
    \bdelta = \begin{bmatrix}
                \mathbf I_{d_x}&\mathbf 0 
        \end{bmatrix}^{\intercal}\in \mathbb R^{d_y\times d_x} 
$, and 
\[
\begin{aligned}
\kappa \coloneqq \frac 12 \min_{i\in\mathcal N}\left\{\int x_i^2d(\bdelta\bP^{\intercal})_{\sharp}\mu(x)+\int y_i^2d(\bQ^{\intercal})_{\sharp}\nu(y)\right\}>0,\text{  and}
\\
\mathcal N\coloneqq \left\{i\in\{1,\dots, d_y\}:\int x_i^2d(\bdelta\bP^{\intercal})_{\sharp}\mu(x)+\int y_i^2d(\bQ^{\intercal})_{\sharp}\nu(y) \neq 0\right\}.
\end{aligned}
\]
\end{proposition} 
\cref{lem:equivalence_igw_w} and its proof enable us to characterize convergence under IGW. 
To our knowledge, this result is new and stands in analogy to the equivalence between $p$-Wasserstein convergence and weak convergence of measures plus convergence of $p$-th moments, see  \cite[Theorem 6.9]{villani2008optimal}).

\begin{proposition}[IGW and weak convergence]\label{prop:IGW_weak_conv}
Fix $(\mu_n)_{n\in\mathbb N}\subset\cP_2(\RR^d)$ and $\mu\in\mathcal P_2(\RR^d)$. Then, $\IGW(\mu_n,\mu)\to 0$ if and only if there exists $\bdelta_n\in \mathrm{St}(d,d)$ such that $(\bdelta_n)_\sharp\mu_n\stackrel{w}{\to} \mu$ and $M_2\big(\mu_n)\to M_2(\mu)$. Furthermore, there exists $\bdelta\in\mathrm{St}(d,d)$, such that $\bdelta_\sharp\mu_n\stackrel{w}{\to}\mu$ along~a~subsequence.
\end{proposition}
     \cref{prop:IGW_weak_conv} is proved in \cref{proof:prop:IGW_weak_conv}. Observe that the sequence $\mu_n=\left((-1)^n\mathrm{Id}\right)_{\sharp}\mu$ satisfies $\mathsf{IGW}(\mu_n,\mu)=0$ for every $n\in\mathbb N$, but $\mu_{2n}$ and $\mu_{2n+1}$ converge weakly to $\mu$ and $(-\id)_{\sharp}\mu$ respectively. As such, it is necessary to pass to a subsequence to guarantee that $\bdelta_\sharp\mu_n\stackrel{w}{\to}\mu$ for some $\bdelta\in\mathrm{St}(d,d)$. 

\section{Closed-Form Solutions for Univariate IGW Alignment}\label{sec:univariate_igw}

Computation of a sliced IGW distance requires a means to efficiently solve univariate IGW problems. By comparison, the $p$-Wasserstein distance on the line has $\Wp(\mu,\nu)=\|F_\mu^{-1}-F_\nu^{-1}\|_{L^p([0,1])}$, where $F_\mu^{-1}$ is the quantile function associated with~$\mu\in\mathcal{P}(\R)$. When $\mu$ and $\nu$ are uniform distributions~on $n$ points, this reduces to a sorting problem. While Theorem~4.2.4 in \cite{vayer2020contribution} derives a solution for univariate IGW, their result assumes absolute continuity of one of the marginals, which is violated in the discrete setting. Our next result, proven in \cref{proof:thm:UnivariateIGW}, provides a fully general treatment. %

\begin{theorem}[Univariate IGW]\label{thm:UnivariateIGW}
Fix $\mu,\nu\in\cP_2(\RR)$ and  $F_{\mu}^{-1},F_{\nu}^{-1}$ as their quantile functions. Then,  
\[
\IGW(\mu,\nu)^2
=
M_2(\mu)^2+M_2(\nu)^2
-2\max\left\{\left( \int xy\,\dd\pi^{\star}(x,y)\right)^2,\left(\int xy\,\dd\pi_{\star}(x,y)\right)^2\right\},
\]
where $\pi^{\star} = \big(F_{\mu}^{-1},F_{\nu}^{-1}\big)_{\sharp}\upsilon$ and
$\pi_{\star} = \big(-F_{(-\id)_{\sharp}\mu}^{-1},F_{\nu}^{-1}\big)_{\sharp}\upsilon$,
with $\upsilon=\mathrm{Unif}([0,1])$ and the optimal coupling for $\mathsf{IGW}(\mu,\nu)$ is that which achieves the maximum. 

Furthermore, if $\mu,\nu$ are uniformly distributed on
$\supp(\mu)=\{x^{(i)}\}_{i=1}^{N}$ and $\supp(\nu)=\{y^{(j)}\}_{j=1}^{N}$ with
$x^{(1)}<x^{(2)}<\dots <x^{(N)}$ and $y^{(1)}<y^{(2)}<\dots <y^{(N)}$,
then the optimal coupling is either $\pi_{\id}=(\id,T_{\mathrm{id}})_{\sharp}\mu$ or $\pi_{\check{\mathrm{id}}}=(\id,T_{\check{\mathrm{id}}})_{\sharp}\mu$  where  $T_{\mathrm{id}}:x^{(i)}\mapsto y^{(i)}$ and $T_{\check{\mathrm{id}}}:x^{(i)}\mapsto y^{(N-i+1)}$ are the identity and anti-identity permutations.
\end{theorem}
To our knowledge, the above result has not appeared previously in the literature and depends strongly on the structure of IGW. Indeed, a similar formula for finitely supported marginals with uniform weights was claimed to hold in \cite{vayer2019sliced} for the quadratic GW problem, but was shown to fail via a counterexample in \cite{beinert2022assignment}.

\section{Sliced Inner Product Gromov--Wasserstein Distance}\label{sec:slicedIGW}

The closed-form solution to the univariate IGW problem enables slicing methods for high-dimensional IGW alignment. To define the average-sliced IGW distance, write $\unitsph$ for the unit sphere in $\RR^d$ and let $\sigma_d$ be the Haar measure on $\mathbb S^{d-1}$. Here and in the sequel we assume, without loss of generality, that $d_x\leq d_y$.

\begin{definition}[Sliced IGW distance]\label{def:slicedIGW}
    The average--sliced IGW distance between $\mu\in \cP_2(\R^{d_x})$ and $\nu\in \cP_2(\R^{d_y})$ is
    $
        \aIGW(\mu,\nu)\coloneqq \inf_{\bdelta\in \St(d_x,d_y)} \left( \int_{\mathbb S^{d_y-1}} \IGW\left((\theta^{\intercal}\bdelta)_{\sharp}\mu,(\theta^{\intercal})_{\sharp}\nu\right)^2  d\sigma_{d_y}(\theta) \right)^{1/2}.%
    $
\end{definition}

The above definition differs from standard sliced Wasserstein distances by the outside minimization over the Stiefel manifold. This step is used to retain the orthogonal invariance of the original IGW distance as noted in \cite{vayer2019sliced} (see \cref{prop:metric}). We underscore that this infimum is always attained.

\begin{proposition}[Attainment]
\label{lem:min_attained_slicing}
   Fix $\mu\in\cP_2(\RR^{d_x})$ and $\nu\in\cP_2(\RR^{d_y})$. Then, there exists $\bar\bdelta\in\St(d_x,d_y)$ for which $\aIGW(\mu,\nu)^2=\int_{\mathbb S^{d_y-1}}\IGW((\theta^{\intercal}\bar\bdelta)_{\sharp}\mu,(\theta^{\intercal})_{\sharp}\nu)^2d\sigma_{d_y}(\theta)$.  
\end{proposition}
\cref{lem:min_attained_slicing} follows by leveraging compactness of $\St(d_x,d_y)$ and continuity of the objective in $\bdelta$ as described in \cref{proof:lem:min_attained_slicing}. As an example, we present a closed-form solution for $\aIGW$ between centered Gaussians, which we denote by $\mathcal N(0,\bsigma)$ where $\bsigma$ is the covariance matrix,  see \cref{app:aIGWGaussians} for details. This characterization provides a  baseline for subsequent numerical~experiments.

\begin{example}[$\aIGW$ between Gaussians]
\label{ex:aIGWGaussian} For $\bsigma_{\mu}\in\mathbb R^{d_x\times d_x}$ and $\bsigma_{\nu}\in\mathbb R^{d_y\times d_y}$, 
\begin{equation}
 \aIGW(\cN(0,\bsigma_\mu),\cN(0,\bsigma_\nu))^2 = \frac{\big(\Tr( \bsigma_\mu) - \Tr( \bsigma_\nu)\big)^2+2\sum_{i=1}^{d_y}\big(\lambda_i(\bsigma_\mu)-\lambda_i(\bsigma_\nu)\big)^2}{d_y(d_y+2)},\label{eq:Gauss_sliced_IGW}
\end{equation}
where $\lambda_i(\bsigma_{\mu})=0$ for $i>d_x$. Furthermore, the infimum over $\St(d_x,d_y)$ is achieved at  
 $\bdelta =\bV \bJ\bU^{\intercal}$ where $\bJ = [\bI_{d_x}\  \mathbf{0}]^\intercal$ and $\bU$, $\bV$ are obtained by diagonalizing  $\bsigma_{\mu}=\bm U\bm D_{\mu}\bm U^{\intercal}$ and $\bsigma_{\nu}=\bm V\bm D_{\nu}\bm V^{\intercal}$.
\end{example} 

We now move to a comprehensive study of metric, topological, and computational properties of the sliced IGW distance.

\subsection{Metric and Topological Structure}\label{subsec:metric_topology}

We first show that, like IGW itself, the sliced variant defines a pseudometric 
which, under primitive conditions, nullifies if and only if the high-dimensional IGW distance is zero. To state this result, we recall that the characteristic function of $\mu\in\mathcal P(\RR^d)$ is given by $\Phi_\mu(t)\coloneqq \EE_{\mu}[e^{it^\intercal X}],\ t\in\RR^d$. %
We say that $\Phi_{\mu}$ is analytic if, for each $t\in \mathbb R^d$ there exists a neighborhood $U$ of $t$ on which $\Phi_{\mu}$ admits an absolutely and uniformly convergent power series representation \cite[Definition 2.3.2]{krantz2001function}. 
\begin{proposition}[Pseudometric]\label{prop:metric}
$\big(\cP_2(\RR^d),\aIGW\big)$ is a pseudometric space. Furthermore 
\begin{enumerate}
    \item  there exist $\mu,\nu\in\cP_2(\RR^d)$ for which $\aIGW(\mu,\nu)=0$, but $\IGW(\mu,\nu)\neq 0$,
    \item if~$\mu,\nu\in\cP_2(\RR^d)$ have analytic characteristic functions, $\aIGW(\mu,\nu)=0\iff \IGW(\mu,\nu)=0$.
\end{enumerate}
\end{proposition}

It is easy to see that if $\IGW(\mu,\nu)=0$, then $\aIGW(\mu,\nu)=0$. The converse requires more care and is shown to fail without the analyticity assumption. Intuitively, if $\aIGW(\mu,\nu)=0$, this only gives information about (almost every) one-dimensional projection of the measures. Lifting this information to the original measures requires certain structural assumptions such as analyticity of the characteristic function as illustrated in the proof (\cref{proof:prop:metric}). This condition holds, for instance, if $\mu,\nu$ are compactly supported, see \cite[Theorem 7.1.14]{hormander1983analysis}. We further show in \cref{sec:rem:lattice} that when $\mu,\nu$ are supported on a lattice, we again have $\aIGW(\mu,\nu)=0\iff\IGW(\mu,\nu)=0$. 

Next we consider the (quotient) topology induced by the sliced IGW distance and establish its equivalence to that of IGW under suitable conditions.

\begin{proposition}[Topology]\label{prop:topology}
Fix  $\mu\in\cP_2(\RR^d)$ and $(\mu_n)_{n\in\mathbb N}\subset \cP_2(\RR^d)$. If $\IGW(\mu_n,\mu) \rightarrow 0$, then  $\aIGW(\mu_n,\mu) \rightarrow 0$. Further, if  $\aIGW(\mu_n,\mu)\to 0$ and every weak subsequential limit of $\mu_n$ has an analytic characteristic function, then $\mathsf{IGW}(\mu_n,\mu)\to 0$ if $\mu$ has an analytic characteristic function.
\end{proposition}

As in the proof of \cref{prop:metric}, the difficulty stems from  showing that $\aIGW(\mu_n,\mu) \rightarrow 0$ implies that $\IGW(\mu_n,\mu) \rightarrow 0$, see \cref{proof:prop:topology} for complete details.

We conclude  with a comparison inequality between sliced IGW and the sliced $2$-Wasserstein distance, defined as $\aW_2(\alpha,\beta)^2
\coloneqq
\int_{\unitsphy}\W_2\big((\theta^{\intercal})_{\sharp}\alpha,(\theta^{\intercal})_{\sharp}\beta\big)^2 d\sigma_{d_y}(\theta)$ for $\alpha,\beta\in\mathcal P_2(\mathbb R^d)$.
\begin{proposition}[Comparison]
\label{prop:slicedComparison}
 Fix $\mu\in\cP_2(\RR^{d_x})$, $\nu\in\cP_2(\RR^{d_y})$ and  $\bdelta=[
\bI_{d_x}\; \bm 0
]^{\intercal}\in\mathbb R^{d_y\times d_x}$,
\begin{equation}\label{eq:IGW_vs_W2}
\aIGW(\mu,\nu)^2\mspace{-3mu}
\le\mspace{-3mu}
\int_{\unitsph}\IGW\big((\theta^{\intercal}\bdelta)_{\sharp}\mu,(\theta^{\intercal})_{\sharp}\nu\big)^2\,d\sigma_{d_y}(\theta)\mspace{-3mu}\leq\mspace{-3mu} 2\big(M_2(\mu)\mspace{-1mu}+\mspace{-1mu}M_2(\nu)\big)\,\aW_2((\bdelta)_{\sharp}\mu,\nu)^2.\mspace{-1mu}
\end{equation} 
\end{proposition}
This result follows essentially by applying  \eqref{eq:W2_upperbound} slice-wise and enables transferring known empirical convergence rates for $\aW_2$ to $\aIGW$, as we may control $\mathbb E[\aIGW(\hat \mu_n,\mu)]$ by $\sqrt{\mathbb E [\aW_2(\hat {\mu}_n,\mu)^2]}$,
 where $\hat\mu_n\coloneqq \frac{1}{n}\sum_{i=1}^n \delta_{X_i}$ is the  empirical measure from $n$ independent and identically distributed (i.i.d.) samples 
$X_1,\ldots,X_n \stackrel{i.i.d.}\sim\mu$.
 
 For instance,
the following result follows by applying \cite[Corollary~2]{nadjahi2020statistical}, see \cref{proof:prop:slicedComparison} for the proof of \cref{prop:slicedComparison} and this corollary. 
\begin{corollary}[Empirical convergence rates]\label{prop:emp_rates_aigw}
Let $\mu\in\cP_q(\RR^{d_x})$, $\nu\in\cP_q(\RR^{d_y})$ for some $q>2$. There exists a constant $C_{q}>0$ depending only on $q$ such that 
\[
\begin{aligned}
\EE\big[\aIGW(\hat\mu_n,\mu)\big]
&\leq
C_q\,\sqrt{M_2(\mu)}\,M_q(\mu)^{1/q} r_q(n),\text{ and}
\\
\EE\Big[\big|\aIGW(\hat\mu_n,\hat\nu_n)-\aIGW(\mu,\nu)\big|\Big]
&\leq
C_q\Big(\sqrt{M_2(\mu)}\,M_q(\mu)^{1/q}+\sqrt{M_2(\nu)}\,M_q(\nu)^{1/q}\Big) r_q(n),
\end{aligned}
\]
where $r_q(n)$ is defined by 
\[
r_q(n)=\begin{cases}
  n^{-(q-2)/(2q)}, &\text{if } q\in(2,4),
  \\
  r_q(n)=n^{-1/4}(\log n)^{\mathbbm 1_{\{q=4\}}1/2}, &\text{if }q=4,
  \\
  r_q(n)=n^{-1/4},&\text{if }q>4.
\end{cases}
\]
\end{corollary} 

Faster rates are attainable under additional assumptions. To wit, if  $\mu$ is log-concave, $\mathbb E[\aIGW(\hat \mu_n,\mu)]=O(\sqrt{\log (n)/n})$, which is optimal up to $\log$ factors  \cite[Theorem 1, Remark 11]{nietert2022statistical}. 

\subsection{Averaging via Monte Carlo Integration}\label{subsec:MC_slicedIGW}

Since $\aIGW$ requires integrating the one-dimensional IGW costs over $\unitsphy$, a natural approximation is Monte Carlo (MC) integration over random directions as in the sliced Wasserstein case \cite{nadjahi2019asymptotic,nietert2022statistical}. Fixing $\mu\in\cP_2(\RR^{d_x})$ and $\nu\in\cP_2(\RR^{d_y})$ with $d_x\le d_y$,  we write $\aIGW(\mu,\nu)=\inf_{\bdelta\in\St(d_x,d_y)} \aIGW_\bdelta(\mu,\nu)$, where 
$
    \aIGW_\bdelta(\mu,\nu)^2
    \coloneqq
    \int_{\unitsphy} g_\bdelta(\theta)\,d\sigma_{d_y}(\theta),$ for $
    g_\bdelta(\theta)\coloneqq \IGW\big((\theta^{\intercal}\bdelta)_{\sharp}\mu,(\theta^{\intercal})_{\sharp}\nu\big)^2.
$
To approximate this integral for a fixed $\bdelta\in\St(d_x,d_y)$ given $\Theta_1,\ldots,\Theta_m\stackrel{i.i.d.}{\sim}\sigma_{d_y}$, we characterize the error of the MC estimator of $\aIGW_\bdelta(\mu,\nu)^2$, defined by
$   \widehat{\aIGW}_{\bdelta,m}(\mu,\nu)^2
    \coloneqq
    \frac{1}{m}\sum_{j=1}^m g_\bdelta(\Theta_j).
$ This will allow us to characterize the MC error of $
    \inf_{\bdelta\in\St(d_x,d_y)}  \widehat{\aIGW}_{\bdelta,m}(\mu,\nu)^2$.
\begin{proposition}[MC error]\label{prop:MC_slicedIGW}
Fix $\mu\in\cP_2(\RR^{d_x})$, $\nu\in\cP_2(\RR^{d_y})$ and
 $\bdelta\in\St(d_x,d_y)$. Then,
 \[
\EE\Big[\Big|\widehat{\aIGW}_{\bdelta,m}(\mu,\nu)^2-\aIGW_\bdelta(\mu,\nu)^2\Big|\Big]
    \leq
    \frac{4\sqrt{2}(2\pi)^{1/4}(M_2(\mu)+M_2(\nu))^2}{\sqrt{md_y}},\]
 and, setting  $\widehat{\aIGW}_{m}(\mu,\nu)^2
    \coloneqq
    \inf_{\bdelta\in\St(d_x,d_y)}  \widehat{\aIGW}_{\bdelta,m}(\mu,\nu)^2$ and $L=4M_2(\mu)^2 + 4M_2(\mu)M_2(\nu)$, 
\[\EE\Big[\Big|\widehat{\aIGW}_{m}(\mu,\nu)^2-\aIGW(\mu,\nu)^2\Big|\Big] \lesssim
\big(M_2(\mu)^2+M_2(\nu)^2\big)\sqrt{\left(d_xd_y - \frac{d_x(d_x-1)}{2}\right)\frac{\log \left(L^2 m\right)}{m}}.\]
\end{proposition}
The proof of this result relies on Lipschitz continuity in $\theta$ and $\bm \Delta$ of $\IGW\big((\theta^{\intercal}\bdelta)_{\sharp}\mu,(\theta^{\intercal})_{\sharp}\nu\big)^2$ and employs a Rademacher symmetrization argument. Complete details are included in \cref{proof:prop:MC_slicedIGW}. 

To account for both the empirical error and the MC error, note that the fully data-driven estimator
$
    \widehat{\aIGW}_{m,n}
    \coloneqq
    \left(\inf_{\bdelta\in\St(d_x,d_y)}
    \frac{1}{m}\sum_{j=1}^m \IGW\big((\Theta_j^{\intercal}\bdelta)_{\sharp}\hat\mu_n,(\Theta_j^{\intercal})_{\sharp}\hat\nu_n\big)^2\right)^{1/2}
$
approximates $\aIGW(\mu,\nu)$ with error at most $\big|\widehat{\aIGW}_{m,n}-\widehat{\aIGW}(\hat \mu_n,\hat \nu_n)\big|
+
\big|\widehat{\aIGW}(\hat \mu_n,\hat \nu_n)-\aIGW(\mu,\nu)\big|$ by the triangle inequality. These terms can be controlled (in expectation) by appealing to \cref{prop:emp_rates_aigw} and
  \cref{prop:MC_slicedIGW}, noting that $|a-b|\leq \sqrt{|a^2-b^2|}$ for $a,b\geq 0$ and applying Jensen's inequality.  

\begin{remark}[Comparison to sliced Wasserstein]
Compared with the Wasserstein case \cite{nietert2022statistical}, whose the rate matches that where $\bdelta$ is fixed   in \cref{prop:MC_slicedIGW}, the bound for IGW incurs an additional $\sqrt{\log(m)}$ factor in the MC error. This extra logarithmic term arises from the covering argument used to control the discrepancy between infima over the Stiefel manifold.
\end{remark}

\subsection{Optimization over the Stiefel Manifold}\label{subsec:Stiefel_slicedIGW}

To this point, we have quantified the MC error in \cref{prop:MC_slicedIGW} and derived a closed-form solution for univariate IGW in \cref{thm:UnivariateIGW}. %
To account for the minimization over $\St(d_x,d_y)$, suppose we are given fixed $(\theta^{(i)})_{i=1}^R$ (realizations of samples from $ \sigma_{d_y}$) and aim to solve
\begin{equation}
\label{eq:stiefelOptimization}
\inf_{\bdelta\in\mathbb R^{d_y\times d_x}} F(\bdelta)\coloneqq  \frac{1}{R}\sum_{i=1}^R\mathsf{IGW}\left(((\theta^{(i)})^{\intercal}\bdelta)_{\sharp}\mu,((\theta^{(i)})^{\intercal})_{\sharp}\nu\right)^2 \text{ s.t. } \left\| \bdelta^{\intercal}\bdelta - \mathbf I_{d_x}\right\|_{\mathrm{F}}^2 = 0, 
\end{equation}
where the sum is a proxy for the integral, following \cref{prop:MC_slicedIGW}.
Letting $\mathbf C_{\pi}= \int xy^{\intercal}d\pi(x,y)$ for $\pi\in\Pi(\mu,\nu)$, \cref{lem:pushforwardMeasures} yields that
$
\mathsf{IGW}\left((\theta^{\intercal}\bdelta)_{\sharp}\mu,(\theta^{\intercal})_{\sharp}\nu\right)^2 =\left(\theta^{\intercal}\bm \Delta \bR_{\mu}\bm \Delta^{\intercal} \theta\right)^2+\left(\theta^{\intercal}\bR_{\nu} \theta\right)^2
+\inf_{\pi\in \Pi(\mu,\nu)}-2\left(\theta^{\intercal}\bdelta\mathbf C_{\pi}\theta\right)^2, 
$
so that 
the objective in \eqref{eq:stiefelOptimization} is a sum of nonsmooth and nonconvex functions (in $\bdelta$). Indeed, $\bm \Delta \mapsto \inf_{\pi\in \Pi(\mu,\nu)}-2\left(\theta^{\intercal}\bdelta\mathbf C_{\pi}\theta\right)^2$ is not necessarily differentiable (see \cref{prop:subgradient}) and is concave (see Proposition 2.9 and Exercise 2.20 (a) in \cite{rockafellar1998variational}). 

Given the particular structure of $F$ (nonconvex, nonsmooth, and not clearly weakly-convex), standard approaches to  optimization over the Stiefel manifold   \cite{zhu2019linearly,li2021weakly,huang2022riemannian} are not subject to formal convergence guarantees. To obtain provably convergent algorithms in the current setting, we thus follow the alternative approach of \cite{xiao2024dissolving,hu2024constraint} which propose to bypass the intricacies of Riemannian optimization methods via a  ``constraint dissolving'' approach. Simply put, these works establish a correspondence between solving a problem over a Riemannian manifold and solving an unconstrained optimization problem. In the particular case of the Stiefel manifold, \cite{hu2024constraint} proposes to minimize 
\begin{equation}
    H(\mathbf \Delta)=  F(A(\mathbf{\Delta})) + \frac{\beta}{4}\left\| \mathbf{\Delta}^{\intercal}\mathbf{\Delta}-\mathbf{I}_{d_x}\right\|_{\mathrm{F}}^2,\text{ for } 
    A(\mathbf{\Delta}) = \frac{1}{8}\mathbf{\Delta}\left(  15\mathbf{I}_{d_x} - 10\mathbf{\Delta}^{\intercal}\mathbf{\Delta}+3(\mathbf{\Delta}^{\intercal}\mathbf{\Delta})^2 \right)\label{eq:cont_disslve_proxy}
\end{equation}
over the entire space $\mathbb R^{d_y\times d_x}$ for a suitable choice of $\beta>0$. The composition with $A(\cdot)$ aims to correct for the curvature of the manifold and the added term penalizes deviations from $\St(d_x,d_y)$. As the new objective, $H$, inherits the nonconvexity and nonsmoothness of $F$, we minimize it via a standard subgradient method, \cref{alg:sgm}, whose convergence properties were studied in \cite{davis2020stochastic}. Throughout, the subdifferential of a locally Lipschitz function is understood in the sense of Clarke \cite{clarke1990optimization}.  
\begin{proposition}[Subdifferential]
\label{prop:subgradient}
    For any $\bm \Delta\in\mathbb R^{d_y\times d_x}$, the Clarke subdifferential of $H$ is
    \[
\begin{gathered}
 \partial H(\mathbf \Delta)=  DA_{[\bdelta]}\left(\partial F(A(\mathbf{\Delta}))\right) + \beta\bdelta\left( \mathbf{\Delta}^{\intercal}\mathbf{\Delta}-\mathbf{I}_{d_x}\right)\text{ for }   
\\
DA_{[\mathbf \Delta]}(\bm \Xi)
\mspace{-2mu}=\mspace{-2mu}\frac{1}{8}\bm \Xi\left(15 \bm I_{d_x}\mspace{-2mu}-\mspace{-2mu}10\bm\Delta^{\intercal}\bm \Delta\mspace{-2mu}+\mspace{-2mu} 3(\bm \Delta^{\intercal}\bm \Delta)^2\right)\mspace{-2mu}-\mspace{-2mu}\bm \Delta\Phi\left( \mathbf \Delta^{\intercal}\mathbf \Xi \right)\mspace{-2mu} +\mspace{-2mu} \frac 32 \mathbf \Delta\Phi\left(\Phi(\mathbf \Delta^{\intercal}\mathbf \Xi) \left(\mathbf \Delta^{\intercal}\mathbf \Delta\mspace{-2mu}-\mspace{-2mu}\mathbf I_{d_x}\right)\right)
\end{gathered}
\]
 where, for a matrix $\bm M\in\mathbb R^{d_y\times d_x}$, we set $\Phi(\bm M)=\frac{1}2 (\mathbf M+\mathbf {M}^{\intercal})$, and 
 \[
        \partial F(\bdelta)= \frac{1}{R}\sum_{i=1}^{R}\left(4\left((\theta^{(i)})^{\intercal} \bdelta \mathbf R_{\mu}\bdelta^{\intercal} \theta^{(i)}\right) \left(\mathbf{R}_{\mu}\bdelta^{\intercal} \theta^{(i)}(\theta^{(i)})^{\intercal}\right)^{\intercal} + \partial G^{\theta^{(i)}}(\bdelta)\right),
    \]
    where %
$        \partial G^{\theta}(\bdelta) \mspace{-2mu}=\mspace{-2mu} 
            \conv\mspace{-2mu}\left\{ -4  \theta^{\intercal}\bdelta \mathbf C_{\pi_{\theta,\bdelta}}\theta  \left(\mathbf C_{\pi_{\theta,\bdelta}}\theta\theta^{\intercal}\right)^{\intercal}\mspace{-2mu}:\mspace{-2mu}
            \pi_{\theta,\bdelta}\in\argmin_{\pi\in \Pi(\mu,\nu)}\mspace{-2mu}-2\left( \theta^{\intercal}\bdelta \mathbf C_{\pi} \theta \right)^2 \right\}
            $
           for each $\theta\in\mathbb S^{d_y-1}$ and $\bdelta\in\mathbb R^{d_y\times d_x}$, 
            with $\conv(\cdot)$ denoting the closed convex hull of a set.
\end{proposition}
The proof of \cref{prop:subgradient} follows by applying a known chain rule for $H$, Proposition 3.6 in \cite{hu2024constraint}, and by computing the subdifferential of $F$, see  \cref{proof:prop:subgradient} for details.

\begin{algorithm}[!t]
\caption{Subgradient method for sliced IGW}\label{alg:sgm}
\begin{algorithmic}[1]
\Statex Fix $\bm{\Delta}_1$ with $\|\bm{\Delta}_1 ^{\intercal}\bm{\Delta}_1-\bm I_{d_x}\|_{\mathrm{F}}\leq 1/6$, $0\leq K\in\mathbb N$, $\beta>0$, and a step sequence $(\eta_k)_{k\in\mathbb N}$,
\For{$k=1,\dots, K$} 
\State $\mathbf G_k\gets DA_{[\bm \Delta_k]}(\bm S_k)+\beta \bdelta_k(\bdelta_k^{\intercal}\bdelta_k-\bI_{d_x})$\text{ for some }$\mathbf S_k\in\partial F(A(\bm \Delta_k))$ 
\State $\bm \Delta_{k+1}\gets\bm \Delta_k-\eta_k\bm G_k$
\EndFor
\end{algorithmic}
\end{algorithm}

Having characterized the subdifferential, we employ \cref{alg:sgm} to obtain  critical points of \eqref{eq:cont_disslve_proxy}. Namely, the algorithm aims to recover a point $\bar \bdelta\in\mathbb R^{d_y\times d_x}$ for which $0\in\partial H(\bar \bdelta)$ which, by Proposition 2.3.2 in \cite{clarke1990optimization}, is a necessary first order condition for optimality. By contrast, a Riemannian critical point for the original problem \eqref{eq:stiefelOptimization}  is a point $\bdelta\in \St(d_x,d_y)$ for which $0\in\partial_R F( \bdelta)$, where $\partial_R F( \bdelta)$ is the Riemannian Clarke subdifferential \cite[Theorem 5.1]{yang2014optimality},
$
\partial_{{R}}F(\mathbf \Delta)=
\left\{ \mathbf Z-\frac{1}{2}\mathbf \Delta(\mathbf \Delta^{\intercal}\mathbf Z +\mathbf Z^{\intercal} \mathbf \Delta ) : \mathbf Z \in \partial F(\mathbf \Delta)\right\}.$
Here, the term $\mathbf Z-\frac{1}{2}\mathbf \Delta(\mathbf \Delta^{\intercal}\mathbf Z +\mathbf Z^{\intercal} \mathbf \Delta)$ corresponds to the projection of $\mathbf Z\in\partial F(\mathbf \Delta)$ onto the tangent space to $\St(d_x,d_y)$ at $\bm\Delta$   
(cf. e.g., \cite{absil2009optimization} Example 3.6.2), and the above formula follows from Theorem 5.1 in \cite{yang2014optimality}, noting that $F$ is regular\footnote{$F$ is regular if, for every $\bdelta\in\mathbb R^{d_y\times d_x}$,
 $F'(\bdelta;\bm\Xi)\coloneqq\lim_{t\downarrow 0}\frac{F(\bdelta+t\bm\Xi)-F(\bdelta)}{t}$ exists for every $\bm \Xi\in\mathbb R^{d_y\times d_x}$ and
    $F'(\bdelta;\bm\Xi)=\max_{\bm G\in \partial F(\bm \Delta)}\langle\bm G,\bm \Xi\rangle_{\mathrm{F}}$, see Definition 2.3.4 in \cite{clarke1990optimization}.} as follows from the proof of \cref{prop:subgradient} and Proposition 2.3.6 in \cite{clarke1990optimization}. Remarkably, despite this difference, the cluster points of \cref{alg:sgm} correspond to Riemannian critical points for $F$.

\begin{theorem}[Convergence of \cref{alg:sgm}]
\label{thm:convergenceSGM}
Suppose that the step sequence $\eta_k$ satisfies $\sum_{k=1}^{\infty}\eta_k=\infty$, $\sum_{k=1}^{\infty}\eta_k^2<\infty$, and $0\leq \eta_k\leq \frac{1}{2\beta}$ where $\beta\geq \max\{162\alpha,2L_1\}$ for $\alpha$ and $L_1$ as in \cref{prop:boundedIterates} and \cref{cor:LipschitzConstant}. Then, the iterates $(\bdelta_k)_{k\in\mathbb N}$ generated by \cref{alg:sgm} are bounded and each cluster point of $(\bdelta_k)_{k\in\mathbb N}$ is a Riemannian critical point of $F$ over the Stiefel manifold.   
\end{theorem}

The proof of this result relies on the convergence of subgradient methods for certain generic classes of functions \cite{davis2020stochastic} and the properties of the modified objective $H$; see \cref{sec:alg} for details. In particular, we leverage the fact that if $0\in\partial H(\bdelta)$ for some $\bdelta\in\St(d_x,d_y)$, then $0\in\partial_{R} F(\bdelta)$. This framework is compatible with the constraint-dissolving approach we adopt because of the correspondence it provides between stationary points of $H$ and Riemannian stationary points of~$F$. Extending our analysis to the nonasymptotic regime is of interest, but convergence rate results for subgradient methods \cite{zhang2020complexity,davis2022gradient,kong2024cost} only guarantee a weaker version of stationarity which forfeits the correspondence to Riemannian stationary~points.  The per-iteration complexity of \cref{alg:sgm} is characterized in the case of uniform distributions on the same (finite) number of points in \cref{app:sec:validation} under the assumption that the projected distributions are still uniformly supported on the same number of points; beyond this case the complexity of computing $\partial F$ appears difficult to characterize.

\section{Experiments}\label{sec:experiments}

\subsection{Validation of Monte-Carlo error and convergence rate.}\label{sec:validation}

\begin{figure}[!t]
   \centering
    \includegraphics[width=0.9\linewidth]{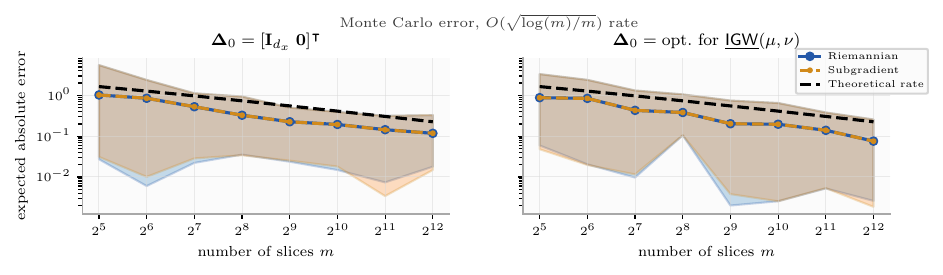}
\vspace{-1em}
    \caption{Validation of the MC error from \cref{prop:MC_slicedIGW}. The estimated value of $\EE[|\widehat{\aIGW}_{m}(\mu,\nu)^2-\aIGW(\mu,\nu)^2|]$ follows the theoretical rate $O(\sqrt{\log(m)/m})$ across optimization methods and initializations. The shaded area corresponds to the maximum and minimum values from the $25$ realizations.}
    \vspace{-.8em}
    \label{fig:MCOptim}
\end{figure}
\begin{figure}[!t]
    \centering
    \includegraphics[width=0.9\linewidth]{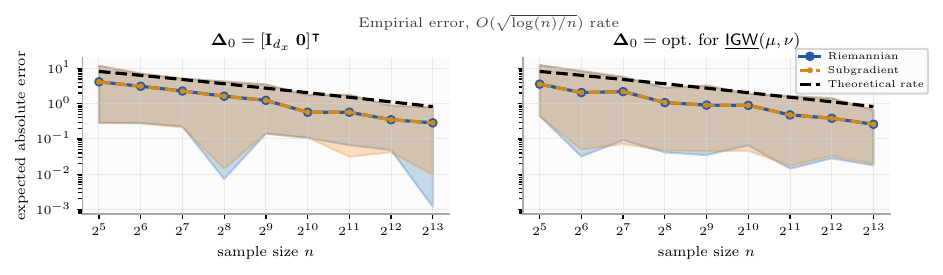}
\vspace{-1em}
    \caption{Validation of the empirical convergence rate from \cref{prop:emp_rates_aigw}. For fixed $m$, the estimated value of $\EE[|\widehat{\aIGW}_{m,n}-\aIGW(\mu,\nu)|]$ scales as $C_1+C_2\sqrt{\log(n)/n}$. The shaded area corresponds to the maximum and minimum values from the $25$ realizations.}
    \vspace{-.8em}
    \label{fig:sampleCompOptim}
\end{figure}
These experiments empirically validate \cref{prop:emp_rates_aigw,prop:MC_slicedIGW} up to the optimization error. Since this requires access to the population value of $\aIGW(\mu,\nu)$, we use the Gaussian setting from \cref{ex:aIGWGaussian}. Specifically, we take $\mu=\mathcal N(0,\bsigma_{\mu})$ and $\nu=\mathcal N(0,\bsigma_{\nu})$ with $d_x=5$ and $d_y=10$, where the covariance matrices are randomly generated. In this setting, both $\aIGW(\mu,\nu)$ and the projected one-dimensional costs can be computed directly from covariance matrices, without discretizing $\mu$ and $\nu$, see Appendix \ref{app:sec:validation} for details and the expressions for $\bsigma_{\mu},\bsigma_{\nu}$.

We first illustrate the $\sqrt{\log(m)/m}$ MC rate for
$
    \EE\big[\big|\widehat{\aIGW}_{m}(\mu,\nu)^2-\aIGW(\mu,\nu)^2\big|\big],$
where the expectation is over the i.i.d. projection directions $\Theta_1,\ldots,\Theta_m\sim\sigma_{d_y}$, recall \cref{prop:MC_slicedIGW}. For each $m\in\{2^r\}_{r=5}^{13}$, we draw a fresh collection of $m$ directions, approximately compute
$
\widehat{\aIGW}_{m}(\mu,\nu)^2
=
\inf_{\bdelta\in\St(d_x,d_y)}
\widehat{\aIGW}_{\bdelta,m}(\mu,\nu)^2
$,
and repeat this procedure $25$ times. We compare the subgradient method from \cref{alg:sgm} with a practical Riemannian subgradient method described in \cref{sec:Riemannian}; only the former is covered by \cref{thm:convergenceSGM}, but the two methods produce similar values in this experiment. We also compare two initializations: $\bdelta_0=[\mathbf I_{d_x}\;\mathbf 0]^{\intercal}$ and the population optimizer from \cref{ex:aIGWGaussian}. The results in \cref{fig:MCOptim} are consistent with the theoretical rate.

We next validate the empirical convergence rate from \cref{prop:emp_rates_aigw}. Fixing $m=3000$ and $n\in\{2^r\}_{r=5}^{12}$, we construct $\hat\mu_n,\hat\nu_n$ from $n$ samples, sample $m$ projection directions, compute $\widehat{\aIGW}_{m,n}$, and average $\big|\widehat{\aIGW}_{m,n}-\aIGW(\mu,\nu)\big|$ over $25$ repetitions with new samples. Since $\mu$ and $\nu$ are Gaussian, the discussion after \cref{prop:emp_rates_aigw} yields a rate of $C_1+C_2\sqrt{\log(n)/n}$ for $C_1,C_2>0$. The results in \cref{fig:sampleCompOptim} again agree with the theory. See Appendix \ref{app:sec:validation} for further details including runtime.

\subsection{Comparison of LLM representations.}

\begin{figure}[t]
    \centering
    \includegraphics[width=0.95\linewidth]{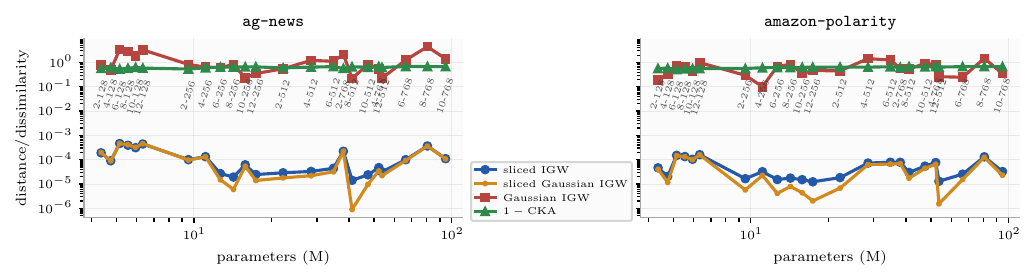}
    \vspace{-1em}
    \caption{Comparison of distances between embeddings of sentences from \texttt{ag-news} and \texttt{amazon-polarity}. Each point corresponds to the distance between the embeddings generated by the student and teacher models with the various metrics. The labels indicate the number of layers $L$ and embedding dimension $H$ for the student architecture, abbreviated as $L$-$H$.}
    \label{fig:LLM}
\end{figure}

Following the representation-comparison framework of~\cite{dandapanthula2025optimal}, we compare embeddings generated by a teacher model, \texttt{bert-base-uncased}, and the $23$ distilled BERT student models from \cite{turc2020wellread}. The teacher has $L=12$ transformer layers and embedding dimension $H=768$, while the students vary over depths $L\in\{2,4,6,8,10,12\}$ and embedding dimensions $H\in\{128,256,512,768\}$, excluding the teacher-sized model. Intuitively, the students are trained to recover the teacher's predictions, and hence are expected to learn representations with similar geometry, although these representations need not be expressed in a common basis. This renders invariant distances like IGW and sliced IGW particularly suitable for this comparison.

We first embed $800$ text passages from the \texttt{amazon-polarity} \cite{MTEB} and \texttt{ag-news} \cite{agnews} datasets using the teacher and student models. For a student-teacher pair, $\mu^S$ denotes the centered empirical distribution over the student embeddings and $\nu^T$ is defined similarly for the teacher. We next compute $\aIGW(\mu^S,\nu^T)$ using the Riemannian subgradient method with $m=3000$ slices initialized at the Gaussian alignment from \cref{ex:aIGWGaussian}. We compare this with $\aIGW(\mathcal N(0,\bsigma_S),\mathcal N(0,\bsigma_T))$ and $\IGW(\mathcal N(0,\bsigma_S),\mathcal N(0,\bsigma_T))$ (which is available in closed form, see \cref{app:sec:LLM}), where $\bsigma_S,\bsigma_T$ are empirical covariance matrices, and $1-\mathrm{CKA}(X^S,Y^T)$ \cite{kornblith2019similarity}, the centered kernel alignment of the embeddings $X^S$ and $Y^T$ from the student and teacher, see \eqref{eq:CKA} for a definition. The results are shown in \cref{fig:LLM}, demonstrating that the empirical sliced IGW distance is consistent with the trends of the other metrics while forgoing restrictive Gaussian assumptions. {Dataset sources, licenses, and runtime details are reported in \cref{app:sec:LLM}.}

\subsection{Clustering heterogeneous user text data.}

\begin{figure}[t]
    \centering
    \includegraphics[width=.95\linewidth]{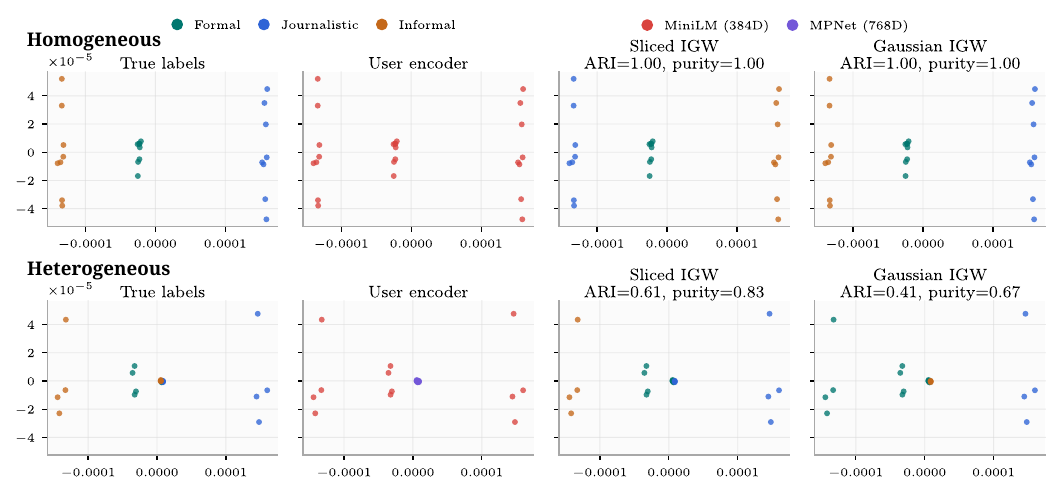}
    \vspace{-1em}
    \caption{Results of user clustering experiment. The 2d representations of the users are obtained via 2d MDS applied to the sliced IGW distance matrix. The top row corresponds to the homogeneous experiment whereas the bottom row is the heterogeneous one. The columns are organized as follows: \textbf{Left:} ground truth clusters, \textbf{Center left:} model used to encode the user texts, \textbf{Center right:} result of clustering based on sliced IGW pairwise distances, \textbf{Right:} clustering from Gaussian IGW pairwise distances. Cluster colors in the two right columns are for illustrative purposes.}      
    \label{fig:hetero_clustering}
    \vspace{-1em}
\end{figure}

Our final application is to  clustering users which are represented by a distribution of text embeddings. Precisely, we form $8$ synthetic users for each of the primary styles: Formal (\texttt{arXiv-summarization} \cite{arxiv}), Journalistic (\texttt{ag-news} \cite{agnews}), and Informal (\texttt{IMDB} \cite{imdb}) based on $120$ sentence embeddings per user for a total of $24$ users. We consider both a homogeneous setting, where all user text are encoded with \texttt{all-MiniLM-L6-v2} in dimension $384$, and a heterogeneous setting, where half of the users of each style are encoded with \texttt{all-MiniLM-L6-v2} and the rest with \texttt{all-mpnet-base-v2} in dimension $768$. Each user is built from $110$ texts in their primary style and $10$ documents from the $2$ non-primary styles (the number of each of the $2$ styles is chosen at random) to complicate the task.

To cluster the users according to their primary style, we compute the pairwise IGW between Gaussian approximations of the user data embeddings as in the previous experiment and the pairwise sliced IGW between empirical measures (again with the Riemannian subgradient method). Importantly, sliced IGW forgoes the restrictive modeling assumptions of Gaussian IGW. 
For each distance matrix $\bD$, we form a self-tuning affinity matrix with entries $\bK_{ij}=\exp(-\bD_{ij}^2/(\sigma_i\sigma_j))$, where $\sigma_i$ is the distance from the $i$-th user $i$ to its third nearest neighbor and then apply spectral clustering \cite{zelnik2004self} with $3$ {clusters} to estimate the true clusters. 

To visualize the results of the clustering in  \cref{fig:hetero_clustering}, we apply multidimensional scaling (MDS) to obtain 2d embeddings of the users using the sliced-IGW pairwise distance matrix.  %
We repeat the process described above three times with new users generated by randomly permuting the embdedded source documents. In all three repetitions, both methods  recover the true clusters exactly in the homogeneous setting, achieving an Adjusted Rand Index (ARI) \cite{hubert1985comparing} of $1$ and purity score of $1$. In the heterogeneous setting, the clustering from sliced IGW achieves ARI values of $(0.61,0.44,0.52)$  and purity scores of $(0.83,0.75,0.79)$, compared with an ARI values of $(0.41,0.17,0.17)$  and purity scores of $(0.67,0.67,0.67)$ for Gaussian IGW where the $i$-th coordinate corresponds to the $i$-th experiment. %
We attribute this improved perfomance to the fact that sliced IGW captures the underlying structure of the data without enforcing a Gaussianity modeling assumption. Dataset sources, licenses, and runtime details are reported in \cref{app:sec:clustering}.

\section{Proofs of Main Results}
\subsection{Proof of \texorpdfstring{\cref{lem:equivalence_igw_w}}{Proposition 2.2}}\label{appen:lem:equivalence_igw_w_proof}
We divide the proof of \cref{lem:equivalence_igw_w} into a number of auxiliary lemmas whose proofs can be found in \cref{proof:lem:W2boundIGW,proof:lem:calI_nonempty} 
\begin{lemma}
\label{lem:W2boundIGW}
    Fix $\mu\in\cP_{2}(\RR^{d_x}),\nu\in\cP_{2}(\RR^{d_y})$  and let $\pi^{\star}$ solve $\mathsf{IGW}(\mu,\nu)$. Then, if $\mathbf P\mathbf \Lambda_{\pi^{\star}}\mathbf Q^{\intercal}$ is the singular value decomposition of $\int xy^{\intercal}d\pi^{\star}(x,y)$  for $\bP\in\mathbb R^{d_x\times d_x}$ and $\bQ\in\mathbb R^{d_y\times d_y}$,
    \[
      \begin{aligned}   \mathsf{W}_2( \tilde \mu,\tilde \nu)^2 
   &\leq \frac{2} {\iota}%
   \left(\mathsf{IGW}(\mu,\nu)^2-\sum_{k\in\mathcal K}\left(\int x_k^2d\tilde \mu(x)\right)^2- \sum_{j\in\mathcal J}\left(\int y_j^2d\tilde \nu(y)\right)^2 \right)
  \\
  &
   + \sum_{j\in\mathcal J} \int y_j^2d\tilde \nu(y) +\sum_{k\in\mathcal K} \int x_k^2d\tilde \mu(x),
      \end{aligned}  
    \]
    where $\tilde \mu =(\bdelta \bm P^{\intercal})_{\sharp}\mu$, and $\tilde \nu =(\bm Q^{\intercal})_{\sharp}\nu$
 for $
    \bdelta = \begin{bmatrix}
                \mathbf I_{d_x}&\mathbf 0 
        \end{bmatrix}^{\intercal}\in \St(d_x,d_y)\subset \mathbb R^{d_y\times d_x},
$ 
\[
\begin{gathered}
    \mathcal I \coloneqq \left\{i\in\{1,\dots, d_x\}:\int x_i^2d \tilde \mu(x)\neq 0,\int y_i^2d \tilde \nu(y) \neq 0\right\},
    \\ 
    \mathcal J\coloneqq \left\{i\in\{1,\dots, d_y\}\cap \mathcal I^c :\int y_i^2d \tilde \nu(y) \neq 0\right\},
        \mathcal K\coloneqq \left\{k\in\{1,\dots, d_y\}\cap \mathcal I^c :\int x_k^2d \tilde \mu(x) \neq 0\right\},
\end{gathered}
\]
and $\iota\coloneqq \min_{i\in\mathcal I}\left\{ \int x_i^2d\tilde \mu(x)+\int x_i^2d\tilde \nu(y) \right\}$ provided that $\mathcal I\neq\emptyset$. 
\end{lemma}

\begin{lemma}    \label{lem:calI_nonempty} In the setting of \cref{lem:W2boundIGW}, if $\mathcal I= \emptyset$, then $\mu$ or $\nu$ is a point mass.  
\end{lemma}

\begin{proof}[Proof of \cref{lem:equivalence_igw_w}] \cref{lem:equivalence_igw_w} is a direct consequence of \cref{lem:W2boundIGW,lem:calI_nonempty}.  Indeed, as $\mu$ and $\nu$ are not point masses, $\mathcal I$ is nonempty by \cref{lem:calI_nonempty}, so that  \cref{lem:W2boundIGW} is applicable. Under the  convention that the minimum over an empty set is $+\infty$, the sum over an empty set is $0$, and $+\infty \cdot 0=0$, we set  $\varsigma\coloneqq \min_{j\in\mathcal J}\left\{ \int y_j^2d\tilde \nu(y)\right\}$, $\zeta \coloneqq \min_{k\in\mathcal K}\left\{ \int x_i^2d\tilde\mu(x)\right\}$, and obtain that 
\[
\sum_{j\in\mathcal J} \left(\int y_j^2d\tilde\nu(y)\right)^2\geq \varsigma\sum_{j\in\mathcal J} \int y_j^2d\tilde\nu(y),\; \sum_{k\in\mathcal K} \left(\int x_i^2d\tilde\mu(x)\right)^2\geq \zeta\sum_{k\in\mathcal K}\int x_i^2d\tilde\mu(x),
\]
where we have again used the notation $\tilde \mu =(\bdelta \bm P^{\intercal})_{\sharp}\mu$, and $\tilde \nu =(\bm Q^{\intercal})_{\sharp}\nu$.
Applying these bounds in   
 \cref{lem:W2boundIGW} yields 
 \[
      \begin{aligned}   \mathsf{W}_2(\tilde \mu,\tilde\nu)^2 
   &\leq \frac{2} {\iota}%
   \left(\mathsf{IGW}(\mu,\nu)^2-\sum_{k\in\mathcal K}\left(\int x_k^2d\tilde \mu(x)\right)^2- \sum_{j\in\mathcal J}\left(\int y_j^2d\tilde \nu(y)\right)^2 \right)
  \\
  &
   +\varsigma^{-1} \sum_{j\in\mathcal J} \left(\int y_j^2d\tilde \nu(y)\right)^2 +\zeta^{-1}\sum_{k\in\mathcal K} \left(\int x_k^2d\tilde \mu(x)\right)^2,
      \end{aligned}  
    \]
    where $1/\infty$ is understood as $0$.
Using \eqref{eq:non_negative_integral} and \eqref{eq:IGW_lower_bound}, 
we have that  
$
\mathsf{IGW}(\mu,\nu)^2- \sum_{j\in\mathcal J}M_2(\tilde \nu)_j^2-\sum_{k\in\mathcal K}M_2(\tilde\mu)^2_k\geq 0$ so that, setting $\kappa' =\frac 12\min\{\iota,2\varsigma,2\zeta\}$, which is finite since $\mathcal I$ is nonempty,
$
      \begin{aligned}   \mathsf{W}_2(\tilde \mu,\tilde\nu)^2
   &\leq (\kappa')^{-1}%
   \mathsf{IGW}(\mu,\nu)^2.
      \end{aligned}  
    $
    The claimed result then follows by noting that $\kappa' \geq \frac 12\min\{\iota,\varsigma,\zeta\} = \kappa$ where $\kappa$ is the constant from \cref{lem:equivalence_igw_w}.
\end{proof}

\subsection{Proof of \texorpdfstring{\cref{prop:IGW_weak_conv}}{Proposition 2.3}}
\label{proof:prop:IGW_weak_conv}
If $\mathsf{IGW}(\mu_n,\mu)\to 0$, we have from \eqref{eq:IGW_lower_bound} that\footnote{We highlight that this step in the proof does not require the assumption that $\mu_n$ and $\mu$ not be point masses.}
 \[
    \begin{aligned}
       \mathsf{IGW}(\mu_n,\mu)^2 
        &\geq \sum_{i\in\mathcal I} \left(\!\left(\int \!x_i^2d(\bdelta\bP^{\intercal}_n)_{\sharp}\mu_n(x)\right)^2\!\!+\left(\int \!y_i^2d(\bQ_n^{\intercal})_{\sharp}\mu(y)\right)^2 
        \!\!-2 \left( \int \!x_iy_id\tilde \pi_n(x,y)\right)^2 \right) \\
        &+ \sum_{k\in\mathcal K}\left(\int x_k^2d(\bdelta\bP_n^{\intercal})_{\sharp}\mu_n(x)\right)^2+ \sum_{j\in\mathcal J}\left(\int y_j^2d(\bQ^{\intercal}_n)_{\sharp}\mu(y)\right)^2 
       \\
       &= \sum_{i\in\mathcal I} \left(\!\left(\int\! x_i^2d(\bdelta\bP^{\intercal}_n)_{\sharp}\mu_n(x)\right)^2\!\!+\left(\int \!y_i^2d(\bQ_n^{\intercal})_{\sharp}\mu(y)\right)^2 
        \!\!-2 \left( \int \!x_iy_id\tilde \pi_n(x,y)\right)^2 \right) \\
        &+ \sum_{k\in\mathcal K}\left(\!\left(\int \!x_k^2d(\bdelta\bP^{\intercal}_n)_{\sharp}\mu_n(x)\right)^2\!\!+\left(\int \!y_k^2d(\bQ_n^{\intercal})_{\sharp}\mu(y)\right)^2 
        \!\!-2 \left( \int \!x_ky_kd\tilde \pi_n(x,y)\right)^2 \right)
        \\
        &+ \sum_{j\in\mathcal J}\left(\!\left(\int \!x_j^2d(\bdelta\bP^{\intercal}_n)_{\sharp}\mu_n(x)\right)^2\!\!+\left(\int \!y_j^2d(\bQ_n^{\intercal})_{\sharp}\mu(y)\right)^2 
        \!\!-2 \left( \int\! x_jy_jd\tilde \pi_n(x,y)\right)^2 \right)
        \\
        &=\sum_{i=1}^{d}  \left(\left(\!\int \!x_i^2d(\bdelta\bP^{\intercal}_n)_{\sharp}\mu_n(x)\right)^2\!\!+\left(\int \!y_i^2d(\bQ_n^{\intercal})_{\sharp}\mu(y)\right)^2 
        \!\!-2 \left( \int \!x_iy_id\tilde \pi_n(x,y)\right)^2 \right)
        \\
        &\geq \sum_{i=1}^{d}\left(\int x_i^2d(\bdelta\bP^{\intercal}_n)_{\sharp}\mu_n(x)-\int y_i^2d(\bQ_n^{\intercal})_{\sharp}\mu(y)\right)^2 
        \geq 0,
    \end{aligned}
\]
where the first equality is due to the fact that $\int y_k^2d(\mathbf Q_n^{\intercal})_{\sharp}\mu(y)=0$ and $\int x_ky_k d\tilde\pi_n(x,y)=0$ for every $k\in\mathcal K$ and a similar conclusion applies to $\mathcal J$ (recall the proof of \cref{lem:W2boundIGW}). The before last inequality follows by  the Cauchy-Schwarz inequality; 
\[
\left(\int x_iy_i d\tilde\pi_n(x,y)\right)^2\leq {
 \int x_i^2d\tilde \pi_n(x,y) \int y_i^2d\tilde \pi_n(x,y)} = \int x_i^2 d (\bdelta\bP^{\intercal}_n)_{\sharp}\mu_n(x) \int y_i^2 d(\bQ^{\intercal}_n)_{\sharp}\mu(y).     
\] It follows that, for every $i\in\{1,\dots,d\}$ ,
$
\int x_i^2d(\bdelta\bP^{\intercal}_n)_{\sharp}\mu_n(x)-\int y_i^2d(\bQ_n^{\intercal})_{\sharp}\mu(y)\to 0$ so that
\[
    \sum_{i=1}^{d} \int x_i^2d(\bdelta\bP^{\intercal}_n)_{\sharp}\mu_n(x)-\sum_{i=1}^{d}\int y_i^2d(\bQ_n^{\intercal})_{\sharp}\mu(y) = M_2(\mu_n)-M_2(\mu)\to 0.
\]
Furthermore, our previous application of the Cauchy-Schwarz  inequality is asymptotically tight, that is, for each $i\in\{1,\dots, d\}$, 
\[
0\leq  \int x_i^2d(\bdelta\bP^{\intercal}_n)_{\sharp}\mu_n(x)\int y_i^2d(\bQ_n^{\intercal})_{\sharp}\mu(y)-\left(\int x_iy_id\tilde \pi_n(x,y)\right)^2\to 0,
\]
whereby 
\[
\begin{aligned}
&\left(\int x_iy_id\tilde \pi_n(x,y)\right)^2 - \left(\int y_i^2d(\bQ_n^{\intercal})_{\sharp}\mu(y)\right)^2\\
&=\left(\int x_iy_id\tilde \pi_n(x,y)\right)^2 - \int y_i^2d(\bdelta\bP_n^{\intercal})_{\sharp}\mu(y)\int y_i^2d(\bQ_n^{\intercal})_{\sharp}\mu(y)
\\
&+\left(\int y_i^2d(\bdelta\bP_n^{\intercal})_{\sharp}\mu(y)-\int y_i^2d(\bQ_n^{\intercal})_{\sharp}\mu(y)\right)\int y_i^2d(\bQ_n^{\intercal})_{\sharp}\mu(y)
\to 0,
\end{aligned}
\]
i.e, there exists a sequence $(\alpha_n)_{n\in\mathbb N}$ such that $\alpha_n\in\{-1,1\}^d$ and $R_{n,i}\coloneqq (\alpha_{n})_i\int x_iy_id\tilde \pi_n(x,y) -\int y_i^2d(\bQ_n^{\intercal})_{\sharp}\mu(y)\to 0$ for each $i\in \{1,\dots, d\}$. Conclude that   
\[
\begin{aligned}
0 &\leq\mathsf W_2\left((\diag(\alpha_n))_{\sharp} (\mathbf \Delta\mathbf{P}^{\intercal}_n)_{\sharp}\mu_n,(\mathbf Q^{\intercal}_n)_{\sharp}\mu\right)^2\\&\leq M_2\left((\mathbf \Delta\mathbf{P}^{\intercal}_n)_{\sharp}\mu_n\right)+M_2((\mathbf Q^{\intercal}_n)_{\sharp}\mu)-2\sum_{i=1}^{d}(\alpha_n)_i\int x_iy_id\tilde \pi_n(x,y)
\to 2M_2(\mu)-2M_2(\mu)=0.
\end{aligned}
\]
 It follows from Theorem 6.9 in \cite{villani2008optimal} that  $(\diag(\alpha_n))_{\sharp}(\mathbf Q_n\mathbf \Delta\mathbf{P}^{\intercal}_n)_{\sharp}\mu_n\stackrel{w}{\to}\mu$ so that the claimed result follows by noting that $\bdelta_n\coloneqq (\diag(\alpha_n) )_{\sharp}\bQ_n\bdelta\bP_n^{\intercal}\in\St(d,d)$. 

   On the other hand, if there exists $\bdelta_n\in \mathrm{St}(d,d)$ such that $(\bdelta_n)_\sharp\mu_n\stackrel{w}{\to} \mu$ and $M_2\big(\mu_n)\to M_2(\mu)$, it holds that 
    \[
    \mathsf{IGW}(\mu_n,\mu) = \mathsf{IGW}((\bdelta_n)_{\sharp}\mu_n,\mu)\leq \left(2M_2(\mu_n)+2M_2(\mu)\right)^{\frac 12}\mathsf W_2((\bdelta_n)_{\sharp}\mu_n,\mu)\to 0, 
    \]
    as follows from invariance of $\mathsf{IGW}$ to unitary isomorphisms and \eqref{eq:IGW_vs_W2}. 

    Now, assume that $\left(\mathbf \Delta_n\right)_{\sharp}\mu_n\stackrel{w}\to \mu$ for some  $\mathbf \Delta_n\in\St(d,d)$ and that the second moments converge.  Note that the Stiefel manifold $\mathrm{St}(d,d)$ is compact in the sense that it forms a closed and bounded subset of $\mathbb R^{d\times d}$. Consequently, there exists a subsequence $n'$ along which $\bdelta_{n'}\to \bdelta\in \St(d,d)$ as matrices. Now, for any bounded Lipschitz function $f:\mathbb R^d\to\mathbb R$ with Lipschitz constant $L$, 
    \[
    \begin{aligned}
        \left| \int fd(\bdelta_{\sharp}\mu_{n'}-\mu)\right|&\leq \left| \int f(\bdelta (\cdot))        -f(\bdelta_{n'} (\cdot))d\mu_{n'}\right|+ \left| \int fd\left((\bdelta_{n'})_{\sharp}\mu_{n'}-\mu\right)\right|
        \\
        &\leq L\|\bdelta-\bdelta_{n'}\|_{\mathrm F}M_1(\mu_{n'})+ \left| \int fd\left((\bdelta_{n'})_{\sharp}\mu_{n'}-\mu\right)\right|.
    \end{aligned}
    \]
    As $\left(\bdelta_{n}\right)_{\sharp}\mu_n\stackrel{w}{\to} \mu$, the leftmost term converges to $0$ by the Portmanteau theorem. As for the first term, we have that $\|x\|\leq 1+\|x\|^2$ for all $x\in\mathbb R^{d}$ so that Definition 6.8 in \cite{villani2008optimal} asserts that $M_1(\mu_{n})\to M_1(\mu)$ as $n\to \infty$ (recall that $\left(\bdelta_{n}\right)_{\sharp}\mu_n\stackrel{w}{\to} \mu$ and convergence of second moments implies convergence in $2$-Wasserstein distance). As $\|\bdelta-\bdelta_{n'}\|_{\mathrm{F}}\to 0$, it follows that $\int f d(\bdelta_{\sharp}\mu_{n'})\to\int f d\mu$ as $n\to \infty$. Since $f$ was an arbitrary bounded Lipschitz function, we conclude that $\bdelta_{\sharp}\mu_{n'}\stackrel{w}{\to} \mu$ by the Portmanteau theorem.
\qed

We remark here that \cref{prop:topology} admits a simpler form in the case that $\mu$ and $\nu$ are distributions on the real line. We record the result here, as it will be used in the proof of \cref{prop:topology} ahead. 

\begin{corollary}
\label{cor:1dCase}
    If $\mu,\nu\in\mathcal P_2(\mathbb R)$, 
    \begin{equation}
    \label{eq:1dComparison}
        \frac 12 ({M_2(\mu)+M_2(\nu)})\min\left\{\mathsf W_2(\mu,\nu)^2,\mathsf W_2((-\id)_{\sharp}\mu,\nu)^2\right\}\leq \mathsf{IGW}(\mu,\nu)^2.
    \end{equation}
    If $(\mu_n)_{n\in\mathbb N}\subset \mathcal P_2(\mathbb R)$ is such that $\mathsf{IGW}(\mu_n,\mu)\to 0$, then  $\left(s_n\id\right)_{\sharp}\mu_n\stackrel{w}{\to} \mu$ and $M_2(\mu_n)\to M_2(\mu)$ where \[
        s_n=\begin{cases}
            1,&\text{if }\mathsf W_2(\mu_n,\mu)\leq \mathsf W_2((-\id)_{\sharp}\mu_n,\mu),
            \\
            -1,&\text{if }\mathsf W_2(\mu_n,\mu)> \mathsf W_2((-\id)_{\sharp}\mu_n,\mu).
        \end{cases}
    \]
    If $(a_n\id)_{\sharp}\mu_n\to \mu$ for some $a_n\in\{-1,1\}$ and $M_2(\mu_n)\to M_2(\mu)$ then $\mathsf{IGW}(\mu_n,\mu)\to 0$. 
\end{corollary}

The proof of \cref{cor:1dCase} is included in \cref{proof:cor:1dCase}.

\subsection{Proof of \texorpdfstring{\cref{thm:UnivariateIGW}}{Theorem 3.1}} 
\label{proof:thm:UnivariateIGW}
We have from Appendix~E in \cite{rioux2024entropic} that, for any $\rho\in\cP_2(\RR^{d_x})$ and $\eta\in\cP_2(\RR^{d_y})$,
\[
\begin{aligned}
\IGW(\rho,\eta)^2
&=
\| \mathbf R_{\rho}\|_{\mathrm {F}}^2 +\| \mathbf R_{\eta}\|_{\mathrm {F}}^2
+\inf_{\pi\in\Pi(\rho,\eta)}-2 \left\|\int xy^{\intercal} \dd\pi(x,y)\right\|_{\F}^2.
\end{aligned}
\]
Instantiating this result with the univariate distributions $\mu,\nu$, the infimum can be written as
\begin{equation}\label{eq:IGWcouplingeqn}
\inf_{\pi\in\Pi(\mu,\nu)}-2 \left(\int xy\,\dd\pi(x,y)\right)^2
=
-2\max\left\{\left( \int xy\,\dd\pi^{\min}(x,y)\right)^2,\left(\int xy\,\dd\pi^{\max}(x,y)\right)^2\right\},
\end{equation}
where $\pi^{\min}\in \argmin_{\pi\in\Pi(\mu,\nu)} \int xy\,\dd\pi(x,y)$ and
$\pi^{\max}\in \argmax_{\pi\in\Pi(\mu,\nu)} \int xy\,\dd\pi(x,y)$.
It is easy to see that
\[
\begin{aligned}
\argmax_{\pi\in\Pi(\mu,\nu)} \int xy\,\dd\pi(x,y)
&=
\argmin_{\pi\in\Pi(\mu,\nu)} \int -xy\,\dd\pi(x,y),\\
\argmin_{\pi\in\Pi(\mu,\nu)} \int xy\,\dd\pi(x,y)
&=
\left\{(-\id,\id)_{\sharp}\check\pi:\check\pi\in\argmin_{\pi\in\Pi((-\id)_{\sharp}\mu,\nu)} \int- xy\,\dd\pi(x,y) \right\}.
\end{aligned}
\]
With the objectives in this form, we note that, for any $\pi\in\Pi(\mu,\nu)$,
\[
\int -2xy\,\dd\pi(x,y)
=
\int (x-y)^2\,\dd\pi(x,y) - \int x^2\,\dd\mu(x)-\int y^2\,\dd\nu(y),
\]
so that $\pi^{\max}$ corresponds to an optimal coupling for the quadratic OT  problem between $\mu$ and $\nu$ and, likewise,
$\pi^{\min}$ is related to a solution of the quadratic OT problem between $(-\id)_{\sharp}\mu$ and $\nu$.

It is well known (cf.\,e.g.,\ \cite[Theorem~2.9]{santambrogio2015optimal}) that the unique optimal coupling for the quadratic OT problem between univariate measures $\alpha,\beta$ is
$\pi^{\star}= \big(F_{\alpha}^{-1},F_{\beta}^{-1}\big)_{\sharp}(\mathrm{Unif}([0,1]))$.
Applying this fact with $(\alpha,\beta)=(\mu,\nu)$ and $(\alpha,\beta)=((-\id)_{\sharp}\mu,\nu)$ proves the general claim.
The statement for uniform distributions on the same number of points follows immediately by inserting the expressions for the relevant CDFs into the derived formula. In this case $\pi^\star$ and $\pi_\star$ are induced by the identity and anti-identity permutations, respectively as claimed.
\qed

\subsection{Proof of \texorpdfstring{\cref{lem:min_attained_slicing}}{Proposition 4.2}}
\label{proof:lem:min_attained_slicing}
As shown in \cref{lemma:IGW_lip} ahead, for each fixed $\theta \in\mathbb S^{d_y-1}$, the function $\mathbf \Delta\in \St(d_x,d_y)\mapsto \mathsf{IGW}((\theta^{\intercal}\bdelta)_{\sharp}\mu,(\theta^{\intercal})_{\sharp}\nu)^2$ is Lipschitz continuous relative to $\|\cdot\|_{\mathrm{F}}$ with constant at most $L'_{\mu,\nu}=4M_2(\mu)^2+4M_2(\mu)M_2(\nu)$. It follows that, for any $\mathbf \Delta,\mathbf \Delta'\in\St(d_x,d_y)$, 
\[
\begin{aligned}
    \left|\int_{\mathbb S^{d_y-1}}\mathsf{IGW}((\theta^{\intercal}\bdelta)_{\sharp}\mu,(\theta^{\intercal})_{\sharp}\nu)^2-\mathsf{IGW}((\theta^{\intercal}\bdelta')_{\sharp}\mu,(\theta^{\intercal})_{\sharp}\nu)^2 d\sigma_{d_y}(\theta) \right|&
    \\
    &\hspace{-12em}\leq L'_{\mu,\nu}\int_{\mathbb S^{d_y-1}}\|\mathbf \Delta-\mathbf \Delta'\|_{\mathrm{F}} d\sigma_{d_y}(\theta) = L'_{\mu,\nu} \|\mathbf \Delta-\mathbf \Delta'\|_{\mathrm{F}}, 
    \end{aligned}
\]
so that $\mathbf \Delta\in\St(d_x,d_y)\mapsto \int_{\mathbb S^{d_y-1}}\mathsf{IGW}((\theta^{\intercal}\bdelta)_{\sharp}\mu,(\theta^{\intercal})_{\sharp}\nu)^2 d\sigma_{d_y}(\theta)$ is Lipschitz continuous with the same constant. As $\St(d_x,d_y)$ is compact relative to $\|\cdot\|_{\mathrm{F}}$, each continuous function over $\St(d_x,d_y)$ achieves its minimum, proving the claim.
\qed

\subsection{\texorpdfstring{Proof of \cref{prop:metric}}{Proof of metric properties}}
\label{proof:prop:metric}
The proof of this result relies upon the following technical lemma, whose proof is provided in \cref{proof:lem:continuity_orientation} ahead.
\begin{lemma}
\label{lem:continuity_orientation}
    For any $\mu\in\mathcal P_2(\mathbb R^{d_x}),\nu\in\mathcal P_2(\mathbb R^{d_y})$, and $\mathbf \Delta\in\St(d_x,d_y)$, the functions $a\in\mathbb R^{d_y}\mapsto \mathsf W_2({(a^{\intercal}\bdelta)_{\sharp}\mu},{(a^{\intercal})_{\sharp}\nu})$ and $a\in\mathbb R^{d_y}\mapsto \mathsf W_2((-\id)_{\sharp}{(a^{\intercal}\bdelta)_{\sharp}\mu},{(a^{\intercal})_{\sharp}\nu})$ are continuous. 
\end{lemma}

We first prove the pseudometric property. To see that $\mathsf{IGW}(\mu,\nu)\geq 0$,  for any $\mu,\nu\in\cP_2(\mathbb R^d)$, \cref{lem:min_attained_slicing} implies that 
$
\aIGW(\mu,\nu)^2=\int
\mathsf{IGW}((\theta^{\intercal}\bdelta)_{\sharp}\mu,(\theta^{\intercal})_{\sharp}\nu)^2d\sigma_{d}(\theta)$
for some choice of $\bdelta\in\St(d,d)$. Since $
\mathsf{IGW}((\theta^{\intercal}\bdelta)_{\sharp}\mu,(\theta^{\intercal})_{\sharp}\nu)^2\geq 0$ for every $\theta\in\mathbb S^{d-1}$, it follows that $\aIGW(\mu,\nu)\geq 0$. If $\mu=\nu$, $0\leq \aIGW(\mu,\nu)^2\leq \int
\mathsf{IGW}((\theta^{\intercal})_{\sharp}\mu,(\theta^{\intercal})_{\sharp}\mu)^2d\sigma_{d}(\theta)=0$ so that $\aIGW(\mu,\nu)=0$. 
To show symmetry, notice that, for any $\bdelta\in\St(d,d)$, after the change of variables $\vartheta = \bdelta^{\intercal}\theta$,
$$\int
\mathsf{IGW}((\theta^{\intercal}\mathbf \bdelta)_{\sharp}\mu,(\theta^{\intercal})_{\sharp}\nu)^2d\sigma_{d}(\theta) = \int
\mathsf{IGW}((\vartheta^{\intercal})_{\sharp}\mu,((\bdelta\vartheta)^{\intercal})_{\sharp}\nu)^2d\sigma_{d}(\vartheta).$$ Take  the infimum over $\bdelta$ on both sides, we get 
\begin{align*}
    \aIGW(\mu,\nu)^2 &= \inf_{\Delta\in\St(d,d)} \int
\mathsf{IGW}((\vartheta^{\intercal})_{\sharp}\mu,((\bdelta\vartheta)^{\intercal})_{\sharp}\nu)^2d\sigma_{d}(\vartheta)\\
&= \inf_{\Delta\in\St(d,d)} \int
\mathsf{IGW}(((\bdelta\vartheta)^{\intercal})_{\sharp}\nu,(\vartheta^{\intercal})_{\sharp}\mu)^2d\sigma_{d}(\vartheta) = \aIGW(\nu,\mu)^2.
\end{align*}
The last equality follows by noting that $\St(d,d)$ is the set of all orthonormal matrices whereby $\{\bdelta:\bdelta\in  \St(d,d)\} = \{\bdelta^{\intercal}:\bdelta\in  \St(d,d)\}$. 

As for the triangle inequality, if $\mu,\nu,\rho\in\mathcal P_2(\mathbb R^d)$, Minkowski's inequality in $L^2(\mathbb S^{d-1},\sigma_{d})$ and the triangle inequality for  IGW yield that 
\[
\begin{aligned}
\aIGW(\mu,\nu)&=    \min_{\bdelta\in\St(d,d)}\left(\int 
\mathsf{IGW}((\theta^{\intercal}\bdelta)_{\sharp}\mu,(\theta^{\intercal})_{\sharp}\nu)^2d\sigma_{d}(\theta)\right)^{1/2}
\\
&\leq \min_{\bdelta\in\St(d,d)}\left(\int 
\left(\mathsf{IGW}((\theta^{\intercal}\bdelta)_{\sharp}\mu,(\theta^{\intercal}\bm\Gamma)_{\sharp}\rho)+\mathsf{IGW}((\theta^{\intercal}\bm\Gamma)_{\sharp}\rho,(\theta^{\intercal})_{\sharp}\nu)\right)^2d\sigma_{d}(\theta)\right)^{1/2}
\\
&\leq \min_{\bdelta\in\St(d,d)}\left(\int 
\mathsf{IGW}((\theta^{\intercal}\bdelta)_{\sharp}\mu,(\theta^{\intercal}\bm\Gamma)_{\sharp}\rho)^2d\sigma_{d}(\theta)\right)^{1/2}
\\
&\hspace{7em}+\left(\int\mathsf{IGW}((\theta^{\intercal}\bm\Gamma)_{\sharp}\rho,(\theta^{\intercal})_{\sharp}\nu)^2d\sigma_{d}(\theta)\right)^{1/2},
\end{aligned}
\]
where $\bgamma\in \St(d,d)$ is arbitrary. The first term satisfies 
\[
\min_{\bdelta\in\St(d,d)}\int 
\mathsf{IGW}((\theta^{\intercal}\bdelta)_{\sharp}\mu,(\theta^{\intercal}\bm\Gamma)_{\sharp}\rho)^2d\sigma_{d}(\theta)=\aIGW(\mu,\bm\Gamma_{\sharp}\rho)^2=\aIGW(\mu,\rho)^2,
\]
where we recall that $\aIGW$ is invariant to orthogonal transformations.  
We then tighten the upper bound in the previous display by  minimizing over the choice of $\bm \Gamma$; 
\[
\begin{aligned}
\aIGW(\mu,\nu)
&\leq \aIGW(\mu,\rho)+\min_{\bgamma\in\St(d,d)}\left(\int\mathsf{IGW}((\theta^{\intercal}\bm\Gamma)_{\sharp}\rho,(\theta^{\intercal})_{\sharp}\nu)^2d\sigma_{d}(\theta)\right)^{1/2}
\\
&=\aIGW(\mu,\rho)+\aIGW(\rho,\nu),
\end{aligned}
\]
as desired. With that, we move to establish Items (i) and (ii).

\textbf{Item (ii): Equivalence under analyticity condition.} We show that when $\mu,\nu$ have analytic characteristic functions, $\aIGW(\mu,\nu)=0$ if and only if $\IGW(\mu,\nu)=0$.
  Suppose first that $\aIGW(\mu,\nu)=0$. By \cref{lem:min_attained_slicing}, there exists $\bdelta\in\St(d,d)$ for which $\mathsf{IGW}((\theta^{\intercal}\bdelta)_{\sharp}\mu,(\theta^{\intercal})_{\sharp}\nu)=0$  for every $\theta\in \mathcal O$ where $\mathcal O\subset \mathbb S^{d-1}$ has full $\sigma_{d}$ measure. Thus, for all such $\theta\in\mathcal O$ there exists $\bm D _{\theta}\in\St(1,1)=\{-1,1\}$ for which $\mu_{\theta}\coloneqq(\bm D_{\theta})_{\sharp}(\theta^{\intercal}\bdelta)_{\sharp}\mu=(\theta^{\intercal})_{\sharp}\nu\eqqcolon \nu_{\theta}$.

  Suppose first that there exists some $\theta \in\mathcal O$ at which $(-\id)_{\sharp}(\theta^{\intercal}\bdelta)_{\sharp}\mu\neq(\theta^{\intercal})_{\sharp}\nu$. By \cref{lem:continuity_orientation}, there exists a neighborhood $N_\theta\subset \mathbb R^{d}$ of $\theta$ on which $\mathsf{W}_2((-\id)_{\sharp}((\cdot)^{\intercal}\bdelta)_{\sharp}\mu,((\cdot)^{\intercal})_{\sharp}\nu)>0$ so that this inequality holds, in particular, at every $\vartheta \in N_{\theta}\cap \mathcal O$ which has positive measure relative to $\sigma_{d}$. 

  \begin{minipage}{0.53\textwidth}
  Conclude that $\mathsf W_2((\vartheta^{\intercal}\bdelta)_{\sharp}\mu,(\vartheta^{\intercal})_{\sharp}\nu)=0$  for every  $\vartheta \in N_{\theta}\cap \mathcal O$ and, by continuity,  on the entire set $\mathbb S^{d-1}\cap N_{\theta}$ so that, for every $\vartheta \in \mathbb S^{d-1}\cap N_{\theta}$, $
  \Phi_{(\vartheta^{\intercal}\bdelta)_{\sharp}\mu}(t)=\Phi_{(\vartheta^{\intercal})_{\sharp}\nu}(t)$ for each $t\in \mathbb R$, i.e., 
  $
    \int e^{i ts}d(\vartheta^{\intercal}\bdelta)_{\sharp}\mu(s)= \int e^{i (t\vartheta)^{\intercal}x}d\bdelta_{\sharp}\mu(x) = \int e^{i (t\vartheta)^{\intercal}y}d\nu(y).$
In sum,  $\Phi_{\bdelta_{\sharp}\mu}(t)=\Phi_{\nu}(t)$ 
  for every $t\in\mathcal T\coloneqq\left\{s\vartheta:s\in\mathbb R\text{ and }\vartheta\in \mathbb S^{d-1} \cap N_{\theta}\right\}$. It is easy to see that $\mathcal T$ contains an open ball centered at $\theta$ for instance, see \cref{fig:metric-proof-cone} for an illustration. However, since $\Phi_{\bdelta_{\sharp}\mu}$ and  $\Phi_{\nu}$ are analytic and coincide on some open set, Corollary 2.3.8 in \cite{krantz2001function} guarantees that $\Phi_{\bdelta_{\sharp}\mu}=\Phi_{\nu}$ on $\mathbb R^d$ and so $\bdelta_{\sharp}\mu=\nu$. 
  
  By the same token, if there exists $\theta\in\mathcal O$ at which $(\theta^{\intercal}\bdelta)_{\sharp}\mu\neq (\theta^{\intercal})_{\sharp}\nu$, then $(-\bdelta)_{\sharp}\mu= \nu$. Finally, if $(-\id)_{\sharp}(\theta^{\intercal}\bdelta)_{\sharp}\mu= (\theta^{\intercal})_{\sharp}\nu$ and $(\theta^{\intercal}\bdelta)_{\sharp}\mu= (\theta^{\intercal})_{\sharp}\nu$ for every $\theta\in\mathcal O$, the squared sliced $2$-Wasserstein distances, 

  \[
    \begin{gathered} 
    \int\mathsf W_2((-\id)_{\sharp}(\theta^{\intercal}\bdelta)_{\sharp}\mu, (\theta^{\intercal})_{\sharp}\nu)^2d\sigma_{d}(\theta)\\\text{ and }\\\int\mathsf W_2((\theta^{\intercal}\bdelta)_{\sharp}\mu, (\theta^{\intercal})_{\sharp}\nu)^2d\sigma_{d}(\theta)  
\end{gathered}
  \]
  are both $0$ so that Proposition 2.2 in \cite{bayraktar2021strong} implies that $(-\bdelta)_{\sharp}\mu= \nu$ and $(\bdelta)_{\sharp}\mu= \nu$. 
    \end{minipage}
\hfill 
\begin{minipage}[H]{0.46\textwidth}
\vspace{-5em}
\centering
\begin{tikzpicture}[scale=0.79, every node/.style={font=\small}]
    \coordinate (O) at (0,0);
    \def\R{2.15}
    \def\ThetaAng{50}
    \def\CapA{0.72}
    \def\CapB{0.37}
    \def\CapRot{-48}
    \colorlet{CapRed}{red!50!black}
    \coordinate (Theta) at ({(\R-0.58)*cos(\ThetaAng)},{(\R-0.58)*sin(\ThetaAng)});
    \coordinate (CapC) at (Theta);

    \coordinate (CapL) at ($(CapC)+({-\CapA*cos(\CapRot)},{-\CapA*sin(\CapRot)})$);
    \coordinate (CapR) at ($(CapC)+({\CapA*cos(\CapRot)},{\CapA*sin(\CapRot)})$);
    \coordinate (ExtL) at ($(O)!1.82!(CapL)$);
    \coordinate (ExtR) at ($(O)!1.82!(CapR)$);
    \coordinate (TopC) at ($(ExtL)!0.5!(ExtR)+(0.06,0.04)$);

    \fill[gray!35, fill opacity=0.09] (O) circle (\R);
    \draw[thick] (O) circle (\R);
    \draw[black!35, dashed] (O) ellipse [x radius=\R, y radius=0.62];
    \draw[black!20] (-0.26,-2.02) .. controls (-0.95,-0.45) and (-0.82,0.92) .. (-0.10,2.02);
    \draw[black!10, dashed] (0.10,-2.02) .. controls (0.84,-0.90) and (0.96,0.45) .. (0.28,2.02);

    \fill[purple!24, opacity=0.28] (O) -- (ExtL) -- (ExtR) -- cycle;
    \fill[CapRed, opacity=0.16] (O) -- (ExtL) -- (ExtR) -- cycle;
    \draw[CapRed, line width=0.85pt] (O) -- (ExtL);
    \draw[CapRed, line width=0.85pt] (O) -- (ExtR);
    \fill let \p1 = ($(ExtR)-(ExtL)$),
               \n1 = {0.98838*veclen(\x1,\y1)/2},
               \n2 = {\n1*\CapB/\CapA} in
        [CapRed!10, rotate around={\CapRot:(TopC)}] (TopC) ellipse [x radius=\n1, y radius=\n2];
    \draw let \p1 = ($(ExtR)-(ExtL)$),
               \n1 = {0.98838*veclen(\x1,\y1)/2},
               \n2 = {\n1*\CapB/\CapA} in
        [CapRed, draw opacity=0.5, line width=0.85pt, rotate around={\CapRot:(TopC)}] (TopC) ellipse [x radius=\n1, y radius=\n2];

    \begin{scope}
        \clip (O) circle (\R);
        \fill[CapRed, opacity=0.30, rotate around={\CapRot:(CapC)}] (CapC) ellipse [x radius=\CapA, y radius=\CapB];
        \draw[CapRed, line width=1.7pt, rotate around={\CapRot:(CapC)}] (CapC) ellipse [x radius=\CapA, y radius=\CapB];
    \end{scope}

    \fill[CapRed] (Theta) circle (2.5pt);
    \node[CapRed, font=\large, anchor=west] at ($(Theta)+(-0.01,-0.22)$) {$\theta$};
    \fill[black] (O) circle (1.7pt);
\end{tikzpicture}
{\captionsetup{width=.9\linewidth}
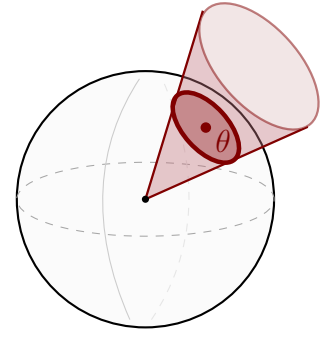
\captionof{figure}{Illustration of the proof of Item (ii) in Proposition~\ref{prop:metric}. For any $\theta\in\unitsph$ where the identity mapping is strictly optimal, there exists a spherical cap around $\theta$ on which the same mapping remains optimal. This optimality extends to the entire cone extending from the origin through that cap. As the characteristic functions of the projected measures coincide along all rays in that cone, they also coincide on the cone itself. Analyticity then forces the characteristic functions to coincide everywhere, hence $\IGW(\mu,\nu)=0$.}
}
\label{fig:metric-proof-cone} \vspace{-3em}
\end{minipage}

  For the other direction, if $\mathsf{IGW}(\mu,\nu)=0$, there exists $\bdelta\in\St(d,d)$ for which $\bdelta_{\sharp}\mu=\nu$, so 
  \[
  \begin{aligned}
  \aIGW(\mu,\nu)^2 &\leq \int \IGW((\theta^{\intercal}\bdelta)_{\sharp}\mu,(\theta^{\intercal})_{\sharp}\nu)^2d\sigma_{d}(\theta)\\
  &\leq  \int 2\left(M_2((\theta^{\intercal}\bdelta)_{\sharp}\mu)+M_2((\theta^{\intercal})_{\sharp}\nu)\right)\mathsf W_2((\theta^{\intercal}\bdelta)_{\sharp}\mu,(\theta^{\intercal})_{\sharp}\nu)^2d\sigma_{d}(\theta), 
  \end{aligned}
  \]
  where the upper bound is due to \eqref{eq:W2_upperbound}.
  Since $\bdelta_{\sharp}\mu=\nu$, the right hand side is $0$ by  Proposition 2.2 in \cite{bayraktar2021strong}  whereby $\aIGW(\mu,\nu)=0$ proving the claim.

\textbf{Item (i): Necessity of analyticity.}
We show by a direct construction that there exist  measures $\rho_1,\rho_2$ with $\aIGW(\rho_1,\rho_2)=0$ but $\IGW(\rho_1,\rho_2) > 0$.

Let $\kappa\in \cP_2(\RR^d)$ ($d>2$) be the base measure with density 
$g(x)=\frac{K}{(1+\|x\|^2)^{m}}$, where $K$ is the  normalizing constant. We take $m>\frac{d+2}{2}$ so that $\kappa$ has bounded second moment. Now, let
\[
\chi(x)=\begin{cases}\exp(-\frac{1}{1-\|x\|^2}),&\text{if }\|x\|<1,
\\
0,&\text{otherwise}.
\end{cases}
\]
Notice that $\chi(\frac{x-z}{r})$ is a smooth compactly supported function on $B(z,r)\coloneqq \left\{x\in\mathbb R^d:\|x-z\|\leq r\right\}$.
Let $\eta_1(x)=\chi\left(\frac{x-e_1}{1/2}\right)-\chi\left(\frac{x+e_1}{1/2}\right)$ and $\eta_2(x)=\chi\left(\frac{x-e_2}{1/4}\right)-\chi\left(\frac{x+e_2}{1/4}\right)$, where $e_1$ and $e_2$ are the standard unit vectors. Note that  $\eta_1$ is supported in $C_1\coloneqq \{x\in\mathbb R^d:|x_1|\geq|x_2|\}$ and $\eta_2$ is supported in $C_2\coloneqq \{x\in\mathbb R^d:|x_2|\geq|x_1|\}$ so that $\eta_1$ and $\eta_2$ have disjoint support and are both odd functions. Furthermore, $\eta_1$ and $\eta_2$ are Schwartz functions, that is, for every multi-index $\alpha,\beta$, 
$
\sup_{x\in\mathbb R^d}|x^{\beta}\partial^{\alpha}\eta_i(x)|<\infty \text{ for $i=1,2$}. 
$
As the Fourier transform defines an isomorphism on the set of all Schwartz functions (cf. e.g., Theorem 7.1.5 in \cite{hormander1983analysis}), $\eta_1$ and $\eta_2$ have inverse Fourier transforms, $\check \eta_1,\check\eta_2$, which are also Schwartz functions and, for $j=1,2$, we have that
\[
\begin{aligned}
\check{\eta_j}(x)=(2\pi)^{-d}\int_{\RR^d} e^{it^{\intercal}x}\eta_j(t)dt&=(2\pi)^{-d}\int_{\RR^d} \left(\cos(t^{\intercal}x)+i\sin(t^{\intercal}x)\right)\eta_j(t)dt
\\
&=(2\pi)^{-d}i\int_{\RR^d} \sin(t^{\intercal}x)\eta_j(t)dt,
\end{aligned}
\]
where the final equality follows from the fact that $\eta_j$ is odd. Conclude that $h_j\coloneqq i\check \eta_j$ is a real-valued odd Schwartz function so that  $h_1$ and $h_2$ integrate to $0$ and satisfy $|h_j(x)|\leq M(1+\|x\|^2)^{-m}$ for some choice of $M>0$.

Defining $f_1,f_2$ via
\begin{equation}
\label{eq:fDefinitions}
f_1(x) \coloneqq g(x)+\frac{K}{2M}h_1(x)+\frac{K}{2M}h_2(x)\geq 0,\quad f_2(x) \coloneqq g(x)+\frac{K}{2M}h_1(x)-\frac{K}{2M}h_2(x)\geq 0, 
\end{equation}
we see that $f_1,f_2$ integrate to $1$ so that they correspond to probability density functions for some distributions which we denote by $\rho_1$ and $\rho_2$. Moreover, their characteristic functions are
\begin{equation}
\label{eq:charFunctions}
\Phi_{\rho_1}(t) = \hat g(t)+\frac{iK}{2M}\left(\eta_1(t)+\eta_2(t)\right),\quad \Phi_{\rho_2}(t) = \hat g(t)+\frac{iK}{2M}\left(\eta_1(t)-\eta_2(t)\right).
\end{equation}
Notice for any $t\in \mathrm{Int}(C_1)$, the interior of $C_1$,
$
\Phi_{\rho_1}(t) = \hat g(t)+\frac{iK}{2M}\eta_1(t)=\Phi_{\rho_2}(t)$ whereas if  $t\in \mathrm{Int}(C_2)$, 
$\Phi_{\rho_1}(t)=\hat{g}(t)+\frac{iK}{2M}\eta_2(t)$ which coincides with the complex conjugate, $\overline{\Phi_{\rho_2}}(t)$, of $\Phi_{\rho_2}(t)$ since $\hat g$ is real-valued as $g$ is even and $\eta_1,\eta_2$ are also real-valued. We conclude that, if $\theta \in\mathbb S^{d-1}$ is such that $\theta \in \mathrm{Int}(C_1)$, then $\left\{s\theta:s\in\mathbb R\right\}\subset \mathrm{Int}(C_1)$ and $\Phi_{\rho_1}(s\theta)=\Phi_{\rho_2}(s\theta)$ for all $s\in\mathbb R$, that is, $(\theta^{\intercal})_{\sharp}\rho_1=(\theta^{\intercal})_{\sharp}\rho_2$. If $\theta \in \mathrm{Int}(C_2)$, we obtain similarly that $(\theta^{\intercal})_{\sharp}\rho_1=(\theta^{\intercal})_{\sharp}(-\id)_{\sharp}\rho_2=(-\id)_{\sharp}(\theta^{\intercal})_{\sharp}\rho_2$. This means that $\mathsf{IGW}((\theta^{\intercal})_{\sharp}\rho_1,(\theta^{\intercal})_{\sharp}\rho_2)=0$ for $\sigma_{d}$-almost every $\theta$ and so $\aIGW(\rho_1,\rho_2)=0$, noting that $C_1\cup C_2=\mathbb R^d$. 

We next show  that $\IGW(\rho_1,\rho_2)\neq 0$, that is, there does not exist some $\bm \Delta \in \St(d,d)$ with $(\bdelta^{\intercal})_{\sharp}\rho_1=\rho_2$.  Assume to the contrary that such a $\bdelta$ exists so that $\Phi_{(\bm\Delta^{\intercal})_{\sharp}\rho_1}(t)=\Phi_{\rho_1}(\bm\Delta t)=\Phi_{\rho_2}(t)$ for every $t\in\mathbb R^d$. Observe that 
\[
\begin{aligned}
\Phi_{\rho_1}(\bdelta t) 
&
= \int_{\RR^d} e^{i(\bdelta t)^{\intercal}x}\left(g(x)+\frac{K}{2M}\left(h_1(x)+h_2(x)\right)\right)dx
\\
&
=\int_{\RR^d} e^{i(\bdelta t)^{\intercal}\bm\Delta x}g(\bm \Delta x)dx+\frac{iK}{2M}\left(\eta_1(\bm \Delta t)+\eta_2(\bm \Delta t)\right)
\\
&
=\hat g(t)+\frac{iK}{2M}\left(\eta_1(\bm \Delta t)+\eta_2(\bm \Delta t)\right)
\end{aligned}
\]
where we have used the fact that $\bm \Delta^{\intercal}\bm \Delta=\id$ and that, by construction, $g$ is invariant to orthogonal transformations. It follows that $\Phi_{(\bdelta^{\intercal})_{\sharp}\rho_1}(t)=\Phi_{\rho_2}(t)$ if and only if 
\begin{equation}\label{eqn:counterexample_etas}
    \eta_1(\bm \Delta t)+\eta_2(\bm\Delta t)=\eta_1(t)-\eta_2(t) \quad \text{for all } t\in \mathbb R^{d}.
\end{equation}
We now illustrate that  no such $\bdelta$ exists. For $t\in B(e_1,1/2)\subseteq C_1$, we have $\eta_1(\bm \Delta t)+\eta_2(\bm\Delta t)=\eta_1(t)$ and, since $\eta_1$ and $\eta_2$ have disjoint supports,  either
$\eta_1(\bm \Delta t)=\eta_1(t) \text{ or } \eta_2(\bm\Delta t)=\eta_1(t).$ The latter case cannot occur since  
$\eta_1$ and $\eta_2$ are supported on balls  of different radius, so that no orthogonal transformation can make them coincide.  In the former case, $\bdelta=\mathrm{Id}$ since $\eta_1$ is an odd function with no symmetries. However, this implies that $\eta_1(t)+\eta_2(t)= \eta_1(t)-\eta_2(t)$ for each $t\in\mathbb R^{d}$, that is, $2\eta_2(t)=0$, yielding a clear contradiction.  
It follows that $\IGW(\rho_1,\rho_2) > 0$, proving the claim. 
\qed

\subsection{\texorpdfstring{Proof of \cref{prop:topology}}{Proof of topology result}}
\label{proof:prop:topology}
   Suppose first that $\IGW(\mu_n,\mu)\to 0$ as $n\to \infty$. By \cref{prop:IGW_weak_conv}, there exists $\bm \Delta_n\in\St(d,d)$ for which $(\bdelta_n^{\intercal})_{\sharp}\mu_n\stackrel{w}{\to}\mu$ and $M_2(\mu_n)\to M_2(\mu)$. By \cref{lem:continuity_orientation}, for any $\theta \in\mathbb S^{d-1}$,  $(\theta^{\intercal}\bdelta_n)_{\sharp}\mu_n\stackrel{w}{\to}(\theta^{\intercal})_{\sharp}\mu$ and $M_2((\theta^{\intercal})_{\sharp}\mu_n)\to M_2((\theta^{\intercal})_{\sharp}\mu)$ so that $\IGW((\theta^{\intercal}\bdelta_n)_{\sharp}\mu_n,(\theta^{\intercal})_{\sharp}\mu)\to 0$. Observe now that   
    \[
\begin{aligned}
0\leq \mathsf{IGW}((\theta^{\intercal}\bdelta_n)_{\sharp}\mu_n,(\theta^{\intercal})_{\sharp}\mu)^2
&\leq (\theta^{\intercal}\bdelta_n \mathbf R_{\mu_n} \bdelta_n^{\intercal} \theta)^2 +  (\theta^{\intercal} \mathbf R_{\mu} \theta)^2\leq M_2(\mu_n)^4+M_2(\mu)^4,
\end{aligned}
\] 
where we have used the Cauchy-Schwarz inequality to obtain that $|\theta^{\intercal}\bdelta_n \int zz^{\intercal}d\mu_n(z) \bdelta_n^{\intercal} \theta|\leq \int \|z\|^2d\mu_n(z) \|\bdelta_n^{\intercal} \theta\|^2=M_2(\mu_n)^2$ since $\St(d,d)$ is the set of $d\times d$ orthogonal matrices.   
Since $M_2(\mu_n)\to M_2(\mu)$, $\mathsf{IGW}((\theta^{\intercal}\bdelta_n)_{\sharp}\mu_n,(\theta^{\intercal})_{\sharp}\mu)^2$ is dominated by a constant so that we may apply the dominated convergence theorem to obtain that  
\[
0\leq \aIGW(\mu_n,\mu)^2\leq \int \IGW((\theta^{\intercal}\bdelta_n)_{\sharp}\mu_n,(\theta^{\intercal})_{\sharp}\mu)^2d\sigma_{d}(\theta)\to 0,  
\]
proving that $\aIGW(\mu_n,\mu)\to 0$. 

On the other hand, if $\aIGW(\mu_n,\mu)\to 0$ there exists a sequence $(\bdelta_n)_{n\in\mathbb N}\subset \St(d,d)$ along which $\int \IGW((\theta^{\intercal}\bdelta_n)_{\sharp}\mu_n,(\theta^{\intercal})_{\sharp}\mu)^2d\sigma_{d}(\theta)\to 0$. Consequently, there exists a set $A$ of full $\sigma_d$ measure for which, after extracting subsequences, $\IGW((\theta^{\intercal}\bdelta_n)_{\sharp}\mu_n,(\theta^{\intercal})_{\sharp}\mu)^2\to 0$ for every $\theta \in A$. (To simplify notation, we relabel the subsequence and continue to denote it by the same index)

Now, fix some $\theta\in A$. From \cref{cor:1dCase},
$\mathsf W_2\left((s_n^{\theta}\id)_{\sharp}(\theta^{\intercal}\mathbf \Delta_n)_{\sharp}\mu_n, (\theta^{\intercal})_{\sharp}\mu\right)\to 0$ where  
  \[s_n^{\theta}=\begin{cases}
            1,&\text{if }\mathsf W_2((\theta^{\intercal}\mathbf \Delta_n)_{\sharp}\mu_n, (\theta^{\intercal})_{\sharp}\mu)\leq \mathsf W_2((-\id)_{\sharp}(\theta^{\intercal}\mathbf \Delta_n)_{\sharp}\mu_n, (\theta^{\intercal})_{\sharp}\mu),
            \\
            -1,&\text{if }\mathsf W_2((\theta^{\intercal}\mathbf \Delta_n)_{\sharp}\mu_n, (\theta^{\intercal})_{\sharp}\mu)> \mathsf W_2((-\id)_{\sharp}(\theta^{\intercal}\mathbf \Delta_n)_{\sharp}\mu_n, (\theta^{\intercal})_{\sharp}\mu).
\end{cases}
\]

Assume that $\theta$ is such that $\mathsf W_2((\theta^{\intercal})_{\sharp}\mu,(-\id)_{\sharp}(\theta^{\intercal})_{\sharp}\mu)\neq 0$ (we leave the case where it nullifies to the end of the proof). We show that there exist some neighborhood of $\theta$ on which  $s_n^{\vartheta} = s_n^{\theta}$ for all sufficiently large $n$. We know that 
\[
\begin{aligned}
&\liminf_{n\to\infty}\mathsf W_2\left((s_n^{\theta}\id)_{\sharp}(\theta^{\intercal}\mathbf \Delta_n)_{\sharp}\mu_n, (-\id)_{\sharp}(\theta^{\intercal})_{\sharp}\mu\right)
\\
&\geq \mathsf W_2((\theta^{\intercal})_{\sharp}\mu,(-\id)_{\sharp}(\theta^{\intercal})_{\sharp}\mu)-\lim_{n\to\infty}\mathsf W_2\left((s_n^{\theta}\id)_{\sharp}(\theta^{\intercal}\mathbf \Delta_n)_{\sharp}\mu_n, (\theta^{\intercal})_{\sharp}\mu\right)
\\
&=\mathsf W_2\left((\theta^{\intercal})_{\sharp}\mu,(-\id)_{\sharp} (\theta^{\intercal})_{\sharp}\mu\right).
\end{aligned}
\]
Hence, for all $n$ sufficiently large, we have that 
\[
\begin{gathered}
\mathsf W_2\left((s_n^{\theta}\id)_{\sharp}(\theta^{\intercal}\mathbf \Delta_n)_{\sharp}\mu_n, (-\id)_{\sharp} (\theta^{\intercal})_{\sharp}\mu\right)>\frac34\mathsf W_2\left((\theta^{\intercal})_{\sharp}\mu,(-\id)_{\sharp} (\theta^{\intercal})_{\sharp}\mu\right) 
\text{ and}\\\mathsf W_2\left((s_n^{\theta}\id)_{\sharp}(\theta^{\intercal}\mathbf \Delta_n)_{\sharp}\mu_n, (\theta^{\intercal})_{\sharp}\mu\right)<\frac14\mathsf W_2\left((\theta^{\intercal})_{\sharp}\mu,(-\id)_{\sharp} (\theta^{\intercal})_{\sharp}\mu\right).
\end{gathered}
\]

Moreover, for all such $n$, we have
\[
\begin{aligned}
&\mathsf W_2\left((s_n^{\theta}\id)_{\sharp}(\vartheta^{\intercal}\mathbf \Delta_n)_{\sharp}\mu_n, (\vartheta^{\intercal})_{\sharp}\mu\right)
\\
&\leq \mathsf W_2\left((s_n^{\theta}\id)_{\sharp}(\theta^{\intercal}\mathbf \Delta_n)_{\sharp}\mu_n, (\theta^{\intercal})_{\sharp}\mu\right)+\|\theta-\vartheta\|\left(M_2((s_n^{\theta}\id)_{\sharp}(\mathbf \Delta_n^{\intercal})_{\sharp}\mu_n)+M_2(\mu)\right)
\\
&=\mathsf W_2\left((s_n^{\theta}\id)_{\sharp}(\theta^{\intercal}\mathbf \Delta_n)_{\sharp}\mu_n, (\theta^{\intercal})_{\sharp}\mu\right)+\|\theta-\vartheta\|\left(M_2(\mu_n)+M_2(\mu)\right)
\\&
<\frac14\mathsf W_2\left((\theta^{\intercal})_{\sharp}\mu,(-\id)_{\sharp} (\theta^{\intercal})_{\sharp}\mu\right)+\|\theta-\vartheta\|\left(M_2(\mu_n)+M_2(\mu)\right),
\end{aligned}
\]
and, likewise, 
\[
\begin{aligned}
\mathsf W_2\left((s_n^{\theta}\id)_{\sharp}(\vartheta^{\intercal}\mathbf \Delta_n)_{\sharp}\mu_n, (-\id)_{\sharp}(\vartheta^{\intercal})_{\sharp}\mu\right)&
\\
&\hspace{-4em}>\frac34\mathsf W_2\left((\theta^{\intercal})_{\sharp}\mu,(-\id)_{\sharp} (\theta^{\intercal})_{\sharp}\mu\right)-\|\theta-\vartheta\|\left(M_2(\mu_n)+M_2(\mu)\right),
\end{aligned}
\]
as follows from Proposition 2.2 in \cite{bayraktar2021strong}. We show later in the proof that $M_2(\mu_n)$ is uniformly bounded so that, for all sufficiently large $n$, there is a neighborhood $N_{\theta}$ of $\theta$ on which  
\[
\mathsf W_2\left((s_n^{\theta}\id)_{\sharp}(\vartheta^{\intercal}\mathbf \Delta_n)_{\sharp}\mu_n, (-\id)_{\sharp} (\vartheta^{\intercal})_{\sharp}\mu\right)>\mathsf W_2\left((s_n^{\theta}\id)_{\sharp}(\vartheta^{\intercal}\mathbf \Delta_n)_{\sharp}\mu_n,  (\vartheta^{\intercal})_{\sharp}\mu\right)\text{ for every }\vartheta\in N_{\theta}.
\]
By applying \cref{lem:pushforwardMeasures}, we obtain that   
\[
\begin{aligned}
\mathsf W_2\left((s_n^{\theta}\id)_{\sharp}(\vartheta^{\intercal}\mathbf \Delta_n)_{\sharp}\mu_n, (-\id)_{\sharp} (\vartheta^{\intercal})_{\sharp}\mu\right)^2&= \inf_{\pi\in\Pi(\mu_n,\mu)} \int \|s_n^{\theta}\vartheta^{\intercal}\bdelta_n x - (-\vartheta^{\intercal} y)\|^2 d\pi(x,y) 
\\
&=\inf_{\pi\in\Pi(\mu_n,\mu)} \int \|(-s_n^{\theta}\vartheta^{\intercal}\bdelta_n x)-  \vartheta^{\intercal} y\|^2 d\pi(x,y)
\\
&=\mathsf W_2\left((-s_n^{\theta}\id)_{\sharp}(\vartheta^{\intercal}\mathbf \Delta_n)_{\sharp}\mu_n,  (\vartheta^{\intercal})_{\sharp}\mu\right)
\end{aligned}
\]
so that $s_n^{\vartheta}$ agrees with $s_n^{\theta}$ at all $\vartheta\in N_{\theta}$ for all such $n$ sufficiently large. Therefore,  \begin{equation}
\label{eq:conv_proof_weak}
\mathsf W_2((s_n^{\theta}\id)_{\sharp}(\vartheta^{\intercal}\bdelta_n)_{\sharp}\mu_n,(\vartheta^{\intercal})_{\sharp}\mu)\to 0\text{ for each }\vartheta\in N_{\theta}\cap A.
\end{equation}
It remains to show that $(s_n^{\theta}\id)_{\sharp}(\bdelta_n)_{\sharp}\mu_n$ admits a limit in $\mathsf W_2$ so that we can apply the reasoning from the proof of \cref{prop:metric}.

We first show that the operator norms of $\mathbf{R}_{n}\coloneqq \int zz^{\intercal} d(\bdelta_n)_{\sharp}\mu_n(z)$ are uniformly bounded. Suppose to the contrary that $\|\mathbf R_n\|_{\mathrm{op}}$ is unbounded and let $\lambda_{1}^{n},\dots,\lambda_d^{n}$ denote the eigenvalues of  $\mathbf R_{{n}}$ ordered from largest to smallest and  $v_{1}^{n},\dots,v_d^{n}$ be the corresponding (orthonormal) eigenvectors noting that $\mathbf R_{n}$ is positive semidefinite. As $\|\mathbf R_n\|_{\mathrm{op}}$ is unbounded, there exists a subsequence $n'$ along which $\|\mathbf R_{n'}\|_{\mathrm{op}}\to\infty$. As $\mathbb S^{d-1}$ is compact,  $v_{1}^{n''}\to v\in\mathbb S^{d-1}$ along a further subsequence $n''$. Now, let $S\coloneqq \left\{\theta \in \mathbb S^{d-1}:|\theta^{\intercal}v|^2>\delta\right\}$ for some $1>\delta>0$. As $S$ has positive surface measure, there exists $\theta \in S\cap A$ for which 
  \[
        \theta^{\intercal}\mathbf{R}_{n''}\theta=\sum_{i=1}^{d}\lambda_i^{n''}\left|\theta^{\intercal}v_{i}^{n''}\right|^2\geq \lambda_1^{n''} \left|\theta^{\intercal}v_{1}^{n''}\right|^2.
  \]
  Since $v_1^{n''}\to v$, $\left|\theta^{\intercal}v_{1}^{n''}\right|^2>0$ for all $n''$ sufficiently large and, as  
$\lambda_1^{n''}=\|\mathbf R_{n''}\|_{\mathrm{op}}$ it follows that $\theta^{\intercal}\mathbf R_{n''}\theta \to \infty$ as $n''\to\infty$.
However, since   $\mathsf{IGW}((\theta^{\intercal}\bdelta_n)_{\sharp}\mu_n,(\theta^{\intercal})_{\sharp}\mu)\to 0$ for all $\theta\in A$, $M_2((\theta^{\intercal}\bdelta_n)_{\sharp}\mu_n)=\theta^{\intercal}\mathbf R_{n}\theta\to \theta^{\intercal}\int zz^{\intercal}d\mu(z)\theta$ contradicting the previous line. As such, we must have that that  $\|\mathbf R_{n}\|_{\mathrm{op}}$ is uniformly bounded.

By Markov's inequality,  $\mathbb P_{(s_n^{\theta}\id)_{\sharp}(\bdelta_n)_{\sharp}\mu_{n}}(\|X\|>t)\leq \frac{M_2((s_n^{\theta}\id)_{\sharp}(\bdelta_n)_{\sharp}\mu_{n})}{t^2}\leq \frac{d\|\mathbf R_n\|_{\mathrm{op}}}{t^2}$  which can be made arbitrarily small by taking $t$ sufficiently large. Conclude that $((s_n^{\theta}\id)_{\sharp}(\bdelta_n)_{\sharp}\mu_n)_{n\in\mathbb N}$ is a tight family and hence for any subsequence $n'$ of $n$ there exists a further subsequence, which we denote by $n''$, along which $(s_{n''}^{\theta}\id)_{\sharp}(\bdelta_{n''})_{\sharp}\mu_{n''}\stackrel{w}{\to}\gamma$ for some probability measure $\gamma$ by Prokhorov's theorem. By assumption, $\gamma$ admits an analytic characteristic function and, for any $\vartheta \in N_{\theta}\cap A$, 
\[
    \Phi_{(s_{n\!'\!'}^{\theta}\!\id)_{\sharp}(\!\vartheta^{\intercal}\bdelta_{n\!'\!'}\!)_{\sharp}\mu_{n\!'\!'}}\!(t)\!=\!\Phi_{(s_{n\!'\!'}^{\theta}\!\id)_{\sharp}(\!\bdelta_{n\!'\!'}\!)_{\sharp}\mu_{n\!'\!'}}\!(t \vartheta)\to\Phi_{\gamma}(t \vartheta)\text{ and }\Phi_{(s_{n\!'\!'}^{\theta}\!\id)_{\sharp}(\!\vartheta^{\intercal}\bdelta_{n\!'\!'}\!)_{\sharp}\mu_{n\!'\!'}}(t)\to \Phi_{\mu}(t\vartheta), 
\]
where the second limit follows from \eqref{eq:conv_proof_weak}. Conclude that, for each $t\in\mathbb R$ and $\vartheta\in N_{\theta}\cap A$, $\Phi_{\gamma}(t\vartheta)=\Phi_{\mu}(t\vartheta)$. By continuity of the characteristic functions, this equality extends to $\vartheta \in N_{\theta}\cap \mathbb S^{d-1}$ so that $\gamma=\mu$ by applying the same reasoning as in the proof of \cref{prop:metric}. With this, $(s_{n''}^{\theta}\id)_{\sharp}(\bdelta_{n''})_{\sharp}\mu_{n''}\stackrel{w}{\to}\mu$. It remains, therefore, to prove that $M_2((\bdelta_{n''})_{\sharp}\mu_{n''})\to M_2(\mu)$ to conclude that $\mathsf{IGW}(\mu_{n''},\mu)\to 0$ (recall \cref{prop:IGW_weak_conv}). 

Since $(s_{n''}^{\theta}\id)_{\sharp}(\bdelta_{n''})_{\sharp}\mu_{n''}\stackrel{w}{\to}\mu$, we have that $(s_{n''}^{\theta}\id)_{\sharp}(\theta^{\intercal}\bdelta_{n''})_{\sharp}\mu_{n''}\stackrel{w}{\to}(\theta^{\intercal})_{\sharp}\mu$ for 
 every $\theta \in\mathbb S^{d-1}$. As    $\mathsf{IGW}((\theta^{\intercal}\bdelta_n)_{\sharp}\mu_n,(\theta^{\intercal})_{\sharp}\mu)\to 0$ for every $\theta \in A$, $M_2\left((s_{n}^{\theta}\id)_{\sharp}(\theta^{\intercal}\bdelta_n)_{\sharp}\mu_n\right)=M_2\left((\theta^{\intercal}\bdelta_n)_{\sharp}\mu_n\right)\to M_2\left((\theta^{\intercal})_{\sharp}\mu\right)$ so that \[
 \int\mathsf W_2((s_{n''}^{\theta}\id)_{\sharp}(\theta^{\intercal}\bdelta_{n''})_{\sharp}\mu_{n''},(\theta^{\intercal})_{\sharp}\mu)d\sigma_{d}(\theta)\to 0.
 \]
 Theorem 2.1 in \cite{bayraktar2021strong} asserts that $\mathsf W_2((s_{n''}^{\theta}\id)_{\sharp}(\bdelta_{n''})_{\sharp}\mu_{n''},\mu)\to 0$ so that $\mathsf{IGW}(\mu_{n''},\mu)\to 0$. As this limit is independent of the choice of subsequence $n''$ it follows that $\mathsf{IGW}(\mu_{n''},\mu)\to 0$. 

 Finally, assume that $\mathsf W_2((\theta^{\intercal})_{\sharp}\mu,(-\id)_{\sharp}(\theta^{\intercal})_{\sharp}\mu)= 0$ for every $\theta\in A$.  Fix $\theta \in A$ and let $s_{n}^{\theta}\in\{-1,1\}$ be as in the first part of the proof so that $\mathsf W_2\left((s_{n}^{\theta}\id)_{\sharp}(\theta^{\intercal}\bdelta_n)_{\sharp}\mu_n,(\theta^{\intercal})_{\sharp}\mu\right)\to 0$. As 
 \[
\mathsf W_2\!\left((s_{n}^{\theta}\id)_{\sharp}(\theta^{\intercal}\!\bdelta_n)_{\sharp}\mu_n,(\theta^{\intercal})_{\sharp}\mu\right)\!=\!\mathsf W_2\!\left((\theta^{\intercal}\!\bdelta_n)_{\sharp}\mu_n,(s_{n}^{\theta}\id)_{\sharp}(\theta^{\intercal})_{\sharp}\mu\right)\!=\!\mathsf W_2\!\left((\theta^{\intercal}\!\bdelta_n)_{\sharp}\mu_n,(\theta^{\intercal})_{\sharp}\mu\right),
 \]
 $\mathsf W_2\left((\theta^{\intercal}\bdelta_n)_{\sharp}\mu_n,(\theta^{\intercal})_{\sharp}\mu\right)\to 0$ for every $\theta \in A$. It follows, as in the previous case, that the sliced Wasserstein distance between $(\bdelta_n)_{\sharp}\mu_n$ and $\mu$ converges to $0$ and hence that their unsliced Wasserstein distance also converges to $0$, proving the claimed result.
\qed

\subsection{\texorpdfstring{Proof of \cref{prop:slicedComparison} and \cref{prop:emp_rates_aigw}}{Proof of sliced comparison and empirical rates}}
\label{proof:prop:slicedComparison}

\begin{proof}[Proof of \cref{prop:slicedComparison}]
Applying \eqref{eq:W2_upperbound}, for each $\theta\in\mathbb S^{d_y-1}$ and $\bdelta\in\St(d_x,d_y)$, 
\begin{equation}
\label{eq:slicewise}
\mathsf{IGW}( (\theta^{\intercal}\bdelta)_{\sharp}\mu,(\theta^{\intercal})_{\sharp}\nu)^2 \leq {2\left(M_2((\theta^{\intercal}\bdelta)_{\sharp}\mu)+M_2((\theta^{\intercal})_{\sharp}\nu)\right)} \mathsf W_2((\theta^{\intercal}\bdelta)_{\sharp}\mu,(\theta^{\intercal})_{\sharp}\nu))^2. 
\end{equation}
Setting $\bdelta= [\mathbf I_{d_x}\;\mathbf 0]^{\intercal}$, we have that
\[
M_2((\theta^{\intercal}\bdelta)_{\sharp}\mu) = \int (\theta^{\intercal}\bdelta z)^2 d\mu(z)= \int \left(\sum_{i=1}^{d_x} \theta_iz_i\right)^2 d\mu(z)\leq\|\theta\|^2\int\|z\|^2d\mu(z)=M_2(\mu).
\]
Similarly, it holds that $M_2((\theta)^{\intercal}_{\sharp} \nu)\leq M_2(\nu)$ so that integrating both sides of \eqref{eq:slicewise}, it holds that 
\[
\aIGW(\mu,\nu)^2 \leq \int \mathsf{IGW}( (\theta^{\intercal}\bdelta)_{\sharp}\mu,(\theta^{\intercal})_{\sharp}\nu)^2d\sigma_{d_y}(\theta) \leq {2\left(M_2(\mu)+M_2(\nu)\right)} \aW_2((\bdelta)_{\sharp}\mu,\nu)^2,
\]
which proves the claimed result.
\end{proof}

\begin{proof}[Proof of \cref{prop:emp_rates_aigw}] 
We have from \cref{prop:slicedComparison} that  
\[
 \aIGW(\hat \mu_n,\mu)\leq {\sqrt 2\sqrt{M_2(\hat \mu_n)+M_2(\mu)}} \aW_2(\hat \mu_n,\mu),   
\]
noting that $\hat \mu_n$ and $\mu$ are distributions on the same Euclidean space and hence the particular alignment, $\bdelta$, from \cref{prop:slicedComparison} is the identity matrix. Taking expectations on both sides of the display and applying the Cauchy-Schwarz inequality, it holds that 
\[
\begin{aligned}
 \mathbb E\Big[\aIGW(\hat \mu_n,\mu)\Big]&\leq {\sqrt 2 \mathbb E\Big[\sqrt{M_2(\hat \mu_n)+M_2(\mu)}} \aW_2(\hat \mu_n,\mu)\Big]
 \\
 &\leq {\sqrt 2 \sqrt{\mathbb E\Big[{M_2(\hat \mu_n)+M_2(\mu)}\Big]}}\sqrt{\mathbb E\Big[ \aW_2(\hat \mu_n,\mu)^2\Big]}.   
\end{aligned}
\]
Note that $\mathbb E\Big[M_2(\hat \mu_n)\Big]= \mathbb E\Big[\frac 1n\sum_{i=1}^{n}\|X_i\|^2\Big] = \mathbb E[\|X_1\|^2]=M_2(\mu)$ since the samples are i.i.d. from $\mu$. Conclude that $\mathbb E\Big[\aIGW(\hat \mu_n,\mu)\Big]\leq 2 \sqrt{M_2(\mu)}\sqrt{\mathbb E\Big[ \aW_2(\hat \mu_n,\mu)^2\Big]}$.  

It remains to bound $\mathbb E\Big[ \aW_2(\hat \mu_n,\mu)^2\Big]$ using bounds from the literature. For instance, if $\mu$ has finite $q$-th moment for $q>2$, Equation (S40) in \cite{nadjahi2020statistical} asserts that there exists a constant  $C_q>0$ depending only on $q$ for which
\[
\mathbb E\Big[ \aW_2(\hat \mu_n,\mu)^2\Big]\leq C_qM_q^{2/q}(\mu)\begin{cases}
    n^{-1/2},&\text{if } q>4\\
       n^{-1/2}\log n,&\text{if } q=4\\
       n^{-(q-2)/q},&\text{if } q\in(2,4),
\end{cases}
\]
which proves the one-sample result upon taking the square root. 
The two-sample case follows from the one-sample setting via the triangle inequality; 
\[
\begin{aligned}
\aIGW(\hat\mu_n,\hat\nu_n)-\aIGW(\mu,\nu)&\leq \aIGW(\hat\mu_n,\mu) +\aIGW(\hat\nu_n,\nu) + \aIGW(\mu,\nu)-\aIGW(\mu,\nu)
\\
\aIGW(\hat\mu_n,\hat\nu_n)-\aIGW(\mu,\nu)&\geq \aIGW(\hat\mu_n,\hat\nu_n) -\aIGW(\hat\nu_n,\nu) - \aIGW(\hat \mu_n,\nu)-\aIGW(\hat\mu_n,\hat\nu_n)
\end{aligned}
\]
so that $\big|\aIGW(\hat\mu_n,\hat\nu_n)-\aIGW(\mu,\nu)\big|\leq  
\aIGW(\hat\mu_n,\mu)+\aIGW(\hat\nu_n,\nu)$. 
\end{proof}

\subsection{\texorpdfstring{Proof of \cref{prop:MC_slicedIGW}}{Proof of Monte Carlo sliced IGW error}}
\label{proof:prop:MC_slicedIGW}
We start by noting the following Lipschitz continuity properties, see \cref{appen:IGW_lip_proof} for a proof.

\begin{lemma}[Lipschitz continuity of projected IGW]\label{lemma:IGW_lip}
Let $\mu\in\mathcal P(\mathbb R^{d_x}),\nu\in\mathcal P(\mathbb R^{d_y})$ have finite second moment. Then,
\begin{enumerate}[leftmargin=*]
    \item 
    For any fixed $\bdelta\in \St(d_x,d_y)$, the function $\theta \in\mathbb S^{d_y-1}\mapsto  \IGW\left((\theta^{\intercal}\bdelta)_{\sharp}\mu,(\theta^{\intercal})_{\sharp}\nu\right)^2$ is Lipschitz continuous with constant at most $L_{\mu,\nu}\coloneqq 4(M_2(\mu)+M_2(\nu))^2$.
    \item 
    For any fixed $\theta \in\mathbb S^{d_y-1}$, the function $\bdelta\in \St(d_x,d_y)\mapsto  \IGW\left((\theta^{\intercal}\bdelta)_{\sharp}\mu,(\theta^{\intercal})_{\sharp}\nu\right)^2$ is Lipschitz continuous with respect to $\|\cdot\|_{\mathrm{F}}$ with constant at most $L'_{\mu,\nu}\coloneqq 4M_2(\mu)^2 + 4M_2(\mu)M_2(\nu)$. 
\end{enumerate} 
\end{lemma}

\begin{proof}[Proof of \cref{prop:MC_slicedIGW}]
Fix $\bdelta\in\St(d_x,d_y)$ and write $f(\theta)\coloneqq g_\bdelta(\theta)$, so that by Lemma~\ref{lemma:IGW_lip} we have that 
\[
    |f(\theta)-f(\theta')|\le L_{\mu,\nu}\|\theta-\theta'\|,\text{ for each }\theta,\theta'\in\unitsphy.
\]
It follows that $f/L_{\mu,\nu}$ is $1$-Lipschitz so that Equation 3.41 in \cite{wainwright2019high} yields that 
\[
\mathbb P\left(|f(\Theta)-\mathbb E[f(\Theta)]|^2\geq t\right)= \mathbb P\left(\frac{|f(\Theta)-\mathbb E[f(\Theta)]|}{L_{\mu,\nu}}\geq \frac{\sqrt{t}}{L_{\mu,\nu}}\right)\leq 2\sqrt{2\pi} e^{-\frac{d_yt}{L^2_{\mu,\nu}}}.
\]
Integrating the above inequality with respect to $t$, it holds that 
\begin{equation}
\label{eq:sphere_poincare}
   \Var\left(f(\Theta)\right)\leq 2\sqrt{2\pi}\int_0^{\infty} e^{-\frac{d_yt}{L^2_{\mu,\nu}}}dt  =2\sqrt{2\pi}\frac{L^2_{\mu,\nu}}{d_y}.
\end{equation} 
Since $\Theta_1,\ldots,\Theta_m$ are i.i.d.\ with law given by $\Theta$,
\[
    \widehat{\aIGW}_{\bdelta,m}^2(\mu,\nu)-\aIGW_\bdelta^2(\mu,\nu)
    =
    \frac{1}{m}\sum_{j=1}^m \Big(f(\Theta_j)-\EE[f(\Theta)]\Big),
\] so that $\widehat{\aIGW}_{\bdelta,m}^2(\mu,\nu)-\aIGW_\bdelta^2(\mu,\nu)$ is mean-zero 
and, applying \eqref{eq:sphere_poincare}, we obtain that
\[
    \Var\Big(\widehat{\aIGW}_{\bdelta,m}^2(\mu,\nu)-\aIGW_\bdelta^2(\mu,\nu)\Big)
    =
    \frac{1}{m}\Var\big(f(\Theta)\big)
    \le
    \frac{2\sqrt{2\pi}L_{\mu,\nu}^2}{md_y}.
\]
the result for fixed $\bm \Delta$ thus follows by applying Jensen's inequality as below,
\[
    \EE\Big[\big|\widehat{\aIGW}_{\bdelta,m}^2(\mu,\nu)\!-\!\aIGW_\bdelta^2\!(\mu,\nu)\big|\Big]
    \le
    \sqrt{\EE\Big[\big(\widehat{\aIGW}_{\bdelta,m}^2(\mu,\nu)\!-\!\aIGW_\bdelta^2\!(\mu,\nu)\big)^2\Big]}
    \le
    \frac{\sqrt{2}(2\pi)^{1/4}L_{\mu,\nu}}{\sqrt{md_y}}.
\]

To prove the result with the infimum, we apply a Rademacher symmetrization argument. Let $\{\sigma_j\}_j$ be i.i.d. Rademacher random variables that take values $\pm1$ with equal probability which are independent of the $\Theta_j$'s. Then,
\begin{align*}
    \EE\Big[\Big|\widehat{\aIGW}_{m}(\mu,\nu)^2\!-\!\aIGW(\mu,\nu)^2\Big|\Big]
    & = \EE_\Theta\!\!\left[\left|\inf_{\bdelta_1\in\St(d_x,d_y\!)}
    \frac{1}{m}\sum_{j=1}^m g_{\bdelta_1}(\Theta_j)-\!\!\!\!\!\inf_{\bdelta_2\in\St(d_x,d_y)} \EE[g_{\bdelta_2}(\Theta_j)]\right|\right]\\
    &\leq \EE_\Theta\!\!\left[\sup_{\bm \Delta\in\St(d_x,d_y)}  \left|\frac{1}{m}\sum_{j=1}^m g_\bdelta(\Theta_j)-\EE[g_\bdelta(\Theta_j)]\right|\right]\\
    &\leq 2\EE_{\Theta,\sigma}\left[\sup_{\bm \Delta\in\St(d_x,d_y)}  \left|\frac{1}{m}\sum_{j=1}^m \sigma_j g_\bdelta(\Theta_j)\right|\right]\\
    &\leq 2\EE_{\Theta,\sigma}\left[\sup_{f\in\mathcal{F}}  \frac{1}{m}\sum_{j=1}^m \sigma_j f(\Theta_j)\right],
\end{align*}
where $\mathcal{F} = \{\pm g_\bdelta:\bdelta\in\St(d_x,d_y)\}\cup\{0\}$. For any vector $\vartheta \stackrel{d}{=} (\Theta_1,\dots,\Theta_m)$, define the metric $\|f-g\|_{p,\vartheta} = \left(\frac{1}{m}\sum_{j=1}^m |f(\Theta_j)-g(\Theta_j)|^p\right)^{1/p}$. 
For each fixed $\vartheta$, equip $\mathcal{F}$ with the metric $\|\cdot\|_{2,\vartheta}$. Let $X_f^\vartheta = \frac{1}{m}\sum_{j=1}^m \sigma_j f(\Theta_j)$ and notice that the random process $\{X_f^\vartheta\}_{f\in\mathcal F}$ has zero mean and sub-Gaussian increments. In particular, we have that 
$$
\EE[e^{\lambda (X_f^\vartheta-X_g^\vartheta)}] = \EE\left[e^{\frac{\lambda}{m}\sum_{j=1}^m \sigma_j (f(\Theta_j)-g(\Theta_j))}\right]\leq e^{\frac{\lambda^2}{m^2} \sum_{j=1}^m \frac{|f(\Theta_j)-g(\Theta_j)|^2}{2}} = e^{\frac{\lambda^2}{2m}\|f-g\|_{2,\vartheta}^2},
$$
as follows from the proof of Theorem~2.2.1 in \cite{vershynin2026hdp}.

Next, we consider an $\epsilon$-net for the metric space $(\mathcal{F},\|\cdot\|_{2,\vartheta})$; a set $\mathcal{F}_\epsilon \subseteq \mathcal{F}$ is called an $\epsilon$-net for $\mathcal F$ if, for every $f\in \mathcal{F}$, there exists $f_\epsilon \in \mathcal{F}_\epsilon$ such that $\|f-f_\epsilon\|_{2,\vartheta}\leq \epsilon$. The $\epsilon$-covering number is then
$$
\mathcal{N}(\mathcal{F},\|\cdot \|_{2,\vartheta},\epsilon) = \inf \{ \mathrm{card}\left(\mathcal{F}_\epsilon\right): \mathcal{F}_\epsilon\text{ is an }\epsilon \text{-net of }\mathcal{F} \}.
$$
Recall that $0\in\mathcal{F}$ and, applying Dudley’s entropy integral bound with the  pseudometric being $\rho(f,g) = \frac{1}{\sqrt{m}}\|f-g\|_{2,\vartheta}$ \cite[Theorem~5.22]{wainwright2019high}, we obtain that
\begin{align*}
    \EE_\sigma\left[\sup_{f\in \mathcal{F}} X_f^{\vartheta}\right] &\leq 2\EE_\sigma\left[\sup_{f,g\in \mathcal{F},\rho(f,g)\leq \epsilon} \left|X_f^{\vartheta}-X_g^{\vartheta}\right|\right] + 32\int_{\epsilon/4}^{\sup_{f,g \in \mathcal{F}}\rho(f,g)} \sqrt{\log \mathcal{N}(\mathcal{F},\rho,\delta)} \ d\delta\\
    &\leq 2\EE_\sigma\left[\sup_{f,g\in \mathcal{F},\|f-g\|_{2,\vartheta}\leq \epsilon\sqrt{m}} \frac{1}{m}\sum_{j=1}^m  |\sigma_j||f(\Theta_j)-g(\Theta_j)|\right]
    \\&\hspace{7em}+32\int_{\epsilon/4}^{\sup_{f,g \in \mathcal{F}}\frac{1}{\sqrt{m}}\|f-g\|_{2,\vartheta}} \sqrt{\log \mathcal{N}\left(\mathcal{F},\frac{1}{\sqrt{m}}\|\cdot\|_{2,\vartheta},\delta\right)} \ d\delta\\
    &\leq 2\epsilon \sqrt{m}+ \frac{32}{\sqrt{m}}\int_{\sqrt{m}\epsilon/4}^{\sup_{f,g \in \mathcal{F}}\|f-g\|_{2,\vartheta}} \sqrt{\log \mathcal{N}(\mathcal{F},\|\cdot\|_{2,\vartheta},\delta)} \ d\delta\\
    &\leq 2\epsilon \sqrt{m}+ 64m^{-1/2}\sup_{f\in\mathcal{F}}\|f\|_{2,\vartheta}\sqrt{\log \mathcal{N}(\mathcal{F},\|\cdot\|_{1,\vartheta},\sqrt{m}\epsilon/4)}.
\end{align*}
Since $\epsilon$ is arbitrary, we may scale $\epsilon$ by $\frac{4}{\sqrt m}$ to obtain that 
$$
\EE_\sigma\left[\sup_{f\in \mathcal{F}} X_f^{\vartheta}\right]\leq \inf_{\epsilon>0} \left\{64m^{-1/2} \sup_{f\in\mathcal{F}}\|f\|_{2,\vartheta} \sqrt{\log \mathcal{N}(\mathcal{F},\|\cdot \|_{1,\vartheta},\epsilon)} + 8\epsilon\right\}.
$$ 
From the expression for univariate IGW in Theorem~\ref{thm:UnivariateIGW}, we know that, for any $\theta$, 
$$
g_\bdelta(\theta)= \IGW\big((\theta^{\intercal}\bdelta)_{\sharp}\mu,(\theta^{\intercal})_{\sharp}\nu\big)^2 \leq M_2((\theta^{\intercal}\bdelta)_{\sharp}\mu)^2+M_2((\theta^{\intercal})_{\sharp}\nu)^2\leq M_2(\mu)^2+M_2(\nu)^2.
$$ 
Therefore,
$$
\sup_{f\in\mathcal{F}}\|f\|_{2,\vartheta} = \sup_{\bm \Delta\in\St(d_x,d_y)} \|g_\bdelta\|_{2,\vartheta}= \sup_{\bm \Delta\in\St(d_x,d_y)}\left(
\frac{1}{m}\sum_{j=1}^m g_\bdelta(\Theta_j)^2
\right)^{1/2}\leq M_2(\mu)^2+M_2(\nu)^2.
$$
Combined, we have that
\begin{equation}\label{eqn:rademacher}
    \begin{aligned} \EE_{\Theta,\sigma}\!\!\left[\sup_{f\in\mathcal{F}}  \frac{1}{m}\sum_{j=1}^m \sigma_j f(\Theta_j)\right]
\!&\leq \inf_{\epsilon>0} \left\{8\epsilon+\EE_\Theta\!\left[\frac{64 (M_2(\mu)^2+M_2(\nu)^2)}{\sqrt m}\!\sqrt{\log \mathcal{N}(\mathcal{F},\|\cdot \|_{1,\vartheta},\epsilon)}\right]\right\}.
\end{aligned}
\end{equation}

We now bound the covering number $\mathcal{N}(\mathcal{F},\|\cdot \|_{1,\vartheta},\epsilon)$. By Lemma~\ref{lemma:IGW_lip} the function $\bdelta \mapsto g_\bdelta(\theta)$ is $L'_{\mu,\nu}$-Lipschitz uniformly in $\theta$ %
so that, for each $\bdelta_1,\bdelta_2\in\St(d_x,d_y)$,
$$
\|g_{\bdelta_1}-g_{\bdelta_2}\|_{1,\vartheta}= 
\frac{1}{m}\sum_{j=1}^m |g_{\bdelta_1}(\Theta_j)-g_{\bdelta_2}(\Theta_j)|
\leq L'_{\mu,\nu}\|\bdelta_1-\bdelta_2\|_{\mathrm{F}}.
$$
Thus, $\mathcal{N}(\mathcal{F},\|\cdot \|_{1,\vartheta},\epsilon)\leq 2\mathcal{N}(\St(d_x,d_y),\|\cdot \|_{\mathrm{F}},\epsilon/L'_{\mu,\nu})+1$ and it suffices to bound the covering number of the Stiefel manifold. To bound $\mathcal{N}(\St(d_x,d_y),\|\cdot \|_{\mathrm{F}},\epsilon/L'_{\mu,\nu})$, we construct a $\epsilon/L'_{\mu,\nu}$-net column by column. We begin with the first column, which lies in the unit sphere $\unitsphy$. Let $\delta_1>0$ and choose a $\delta_1$-net of $\unitsphy$. By Lemma 5.2
 in \cite{vershynin2010introduction}, $\mathcal{N}(\unitsphy, \|\cdot\|_2,\delta_1)\leq 3^{d_y} \delta_1^{-d_y}$ for each $\delta_1<1$. For each vector in this $\delta_1$-net, we pick a $\delta_2$-net for the set vectors orthogonal to it, which is isometric to $\mathbb{S}^{d_y-2}$. This can result in at most $\delta_1+\delta_2$ error in the second vector by the triangle inequality. Following the same procedure, the error in the $i^{th}$ column is at most $\sum_{j=1}^{i}\delta_j$. To obtain a $\epsilon/L'_{\mu,\nu}$-net under $\|\cdot\|_{\mathrm{F}}$, we need 
\begin{equation}\label{eqn:cover_constraint}
    \sum_{i=1}^{d_x} \left(\sum_{j=1}^{i}\delta_j\right)^2 \leq \left(\epsilon/L'_{\mu,\nu}\right)^2,
\end{equation}
so that the covering number will be bounded by
\begin{equation}\label{eqn:cover_count}
    \mathcal{N}(\St(d_x,d_y),\|\cdot \|_{\mathrm{F}},\epsilon/L'_{\mu,\nu}) \leq \prod_{i=1}^{d_x} (\delta_i/3)^{-(d_y-i+1)}.
\end{equation}
Setting $\delta_i=\delta M^{i-1}$ for some $0<M<1$, we bound the left hand side of \eqref{eqn:cover_constraint} as
$$
\sum_{i=1}^{d_x} \left(\sum_{j=1}^{i}\delta_j\right)^2 = \sum_{i=1}^{d_x} \delta^2 \left(\sum_{j=1}^{i}M^{j-1}\right)^2 = \sum_{i=1}^{d_x} \delta^2 \left(\frac{1-M^{i}}{1-M}\right)^2\leq \delta^2 d_x \left(\frac{1}{1-M}\right)^2
$$
So that the choice $\delta_i=\delta M^{i-1}$ is valid for \eqref{eqn:cover_constraint} for the choice  
$\delta = d_x^{-1/2}(1-M)\epsilon/L'_{\mu,\nu}$, where $M\in(0,1)$ is arbitrary. With this choice of $\delta_i$'s, \eqref{eqn:cover_count} can be bounded as
\begin{align*}
    \mathcal{N}\left(\St(d_x,d_y),\|\cdot \|_{\mathrm{F}},\epsilon/L'_{\mu,\nu}\right) &\leq \prod_{i=1}^{d_x} \left(d_x^{-1/2} \frac{(1-M)M^{i-1}\epsilon }{3L'_{\mu,\nu}}\right)^{-(d_y-i+1)}\\
    &\leq C_{d_x,d_y} \left(\epsilon/L'_{\mu,\nu}\right)^{-(d_xd_y - \frac{d_x(d_x-1)}{2})}.
\end{align*}
Therefore, plugging this expression into \eqref{eqn:rademacher}, we have for any $\epsilon>0$,
\begin{align*}
&\EE_{\Theta,\sigma}\left[\sup_{\bm \Delta\in\St(d_x,d_y)}  \left|\frac{1}{m}\sum_{j=1}^m \sigma_j g_\bdelta(\Theta_j)\right|\right]\\ &\leq 8\epsilon + \frac{64(M_2(\mu)^2+M_2(\nu)^2)}{\sqrt{m}}\sqrt{\left(d_xd_y - \frac{d_x(d_x-1)}{2}\right)\log \left(L'_{\mu,\nu}/\epsilon\right)+\log 4C_{d_x,d_y} }.
\end{align*}
Picking $\epsilon\sim m^{-1/2}$, it holds that
$$
\EE\Big[\Big|\widehat{\aIGW}_{m}(\mu,\nu)^2-\aIGW(\mu,\nu)^2\Big|\Big] \lesssim
\frac{(M_2(\mu)^2+M_2(\nu)^2}{\sqrt{m}}\sqrt{\left(d_xd_y \!-\! \frac{d_x(d_x-1)}{2}\right)\!\log\! \left(L'_{\mu,\nu} m ^{1/2}\right)}.
$$
This proves the claim.
\end{proof}

\subsection{\texorpdfstring{Proof of \cref{prop:subgradient}}{Proof of subgradient formula}}
\label{proof:prop:subgradient} 
   We begin with the proof of the first point. Note that $G^{\theta}$ is the optimal value function of a concave minimization problem over the polyhedron $\Pi(\mu,\nu)$, recalling that $\mu$ and $\nu$ are finitely supported so that the marginal constraints simplify to $2N$ equality constraints. Letting $\mathcal V = \{\pi_1,\dots,\pi_{|\mathcal V|}\}$ denote the set of vertices of $\Pi(\mu,\nu)$, concavity of the objective implies that 
    \[ 
    \inf_{\pi\in \Pi(\mu,\nu)}-2\left(\int \theta^{\intercal}\bdelta xy^{\intercal} \theta d\pi(x,y)\right)^2= \min_{i=1}^{|\mathcal V|}-2\left(\int \theta^{\intercal}\bdelta xy^{\intercal} \theta d\pi_i(x,y)\right)^2\coloneqq \min_{i=1}^{|\mathcal V|}\ell_{i,\theta}(\bdelta),
    \]
    where $\ell_{i,\theta}(\bdelta) = -2\left(\left\langle\left(\int xy^{\intercal} \theta \theta^{\intercal}  d\pi_i(x,y)\right)^{\intercal},\bdelta\right\rangle_{\mathrm{F}}\right)^2$, it follows from the chain rule that  
    \[
        D_{[\bdelta]}\ell_{i,\theta}(\Xi) = -4 \left\langle\left(\int xy^{\intercal} \theta \theta^{\intercal}  d\pi_i(x,y)\right)^{\intercal},\bdelta\right\rangle_{\mathrm{F}}\left\langle \left(\int xy^{\intercal} \theta \theta^{\intercal}  d\pi_i(x,y)\right)^{\intercal}, \Xi\right\rangle_{\mathrm F}, 
    \]
    for each $i\in \{1,\dots,|\mathcal V|\}$ and each $\Xi\in\mathbb R^{d_y\times d_x}$. Local Lipschitz continuity and the formula for the subdifferential of $G^{\theta}$ then follows from Proposition 2.3.12  and the surrounding discussion in \cite{clarke1990optimization} (see also Proposition 2.3.6 in the same reference). %

    The second point is a consequence of the sum rule for the Clarke subdifferential (see Corollary 2 and 3 on p.39-40 as well as Proposition 2.3.12 of \cite{clarke1990optimization}) noting that the function $\bdelta \mapsto M_2^2((\theta^{\intercal}\bdelta)_{\sharp}\mu)$ is Fr{\'e}chet differentiable at $\bdelta\in\mathbb R^{d_y\times d_x}$ with derivative
    \[
        \Xi\!\in\!\mathbb R^{d_y\!\times\! d_x}\mapsto 2\theta^{\intercal}\!\bdelta\bm{M}_{\mu}\bdelta^{\!\intercal}\theta\left(\theta^{\intercal} \Xi\bm{M}_{\mu}\bdelta^{\!\intercal}\theta\!+\!\theta^{\intercal} \bdelta\bm{M}_{\mu}\Xi^{\intercal}\theta\right) \!=\! 4 \theta^{\intercal}\bdelta\bm{M}_{\mu}\bdelta^{\!\intercal}\theta\left\langle(\bm{M}_{\mu}\bdelta^{\!\intercal}\theta\theta^{\intercal})^{\intercal}, \Xi\right\rangle_{\mathrm F}.
    \]
\qed

\section{Conclusion and Future Directions} 
This work has demonstrated that the IGW problem between one-dimensional distributions admits a closed form solution, enabling the development of a slicing paradigm for IGW. The metric and topological properties of the average-sliced IGW were derived and shown to coincide with those of the standard IGW under structural conditions and, surprisingly, may differ if those conditions are not met. In addition, the rate of convergence for empirical sliced IGW was provided along as an analysis of the MC error induced when approximating the integral over the sphere by a finite number of slices. A subgradient method was provided to account for the minimization over the Stiefel manifold which was shown to provably converge to a Rimenannian critical point. Synthetic numerical experiments were provided to empirically validate the theoretical results in addition to experiments on  comparing learned embeddings from LLMs and clustering heterogeneous data. 

A number of important open questions persist in the study of sliced GW problems. Chief among them is to establish other costs for which the univariate GW problem can be solved exactly. Another important direction is to develop more efficient algorithms for handling the optimization over the Stiefel manifold. In effect, the proposed approach based on the constraint dissolving technique was considered as it can be shown to converge to a Riemannian critical point whereas methods based on direct Riemannian optimization techniques do not appear to be provably convergent for the objective at hand, but are expected to be more computationally efficient. 
Furthermore, as noted previously, if $\aIGW(\mu,\nu)=0$,   $\IGW(\mu,\nu)$ is not necessarily zero unless additional conditions are met (e.g., $\mu,\nu$ have analytic characteristic functions or are supported on lattices). It is thus of interest to characterize more classes of distributions under which this implication does hold.   

{

\bibliographystyle{amsplain}
\bibliography{ref}

}

\newpage 
\appendix

\section{Proofs of Auxiliary Results} 
We begin this section with a useful technical lemma; the remainder of the section compiles the proofs of the lemmas used to prove the main results.
\begin{lemma}
\label{lem:pushforwardMeasures}
   Given measurable functions $T:\mathbb R^{d_0}\to \mathbb R^{d_0'}$ and $S:\mathbb R^{d_1}\to \mathbb R^{d_1'}$, 
   \[
        \Pi(T_{\sharp}\mu,S_{\sharp}\nu) = \left\{ (T,S)_{\sharp}\pi:\pi\in\Pi(\mu,\nu)\right\}.
   \]
\end{lemma}

\begin{proof}
    It is clear that, for any $\pi'\in\Pi(\mu,\nu)$, $(T,S)_{\sharp}\pi'\in\Pi({T}_{\sharp}\mu,S_{\sharp}\nu)$. On the other hand, if we fix some $\pi\in\Pi(T_{\sharp}\mu,S_{\sharp}\nu)$, we may disintegrate $\mu$ and $\nu$ as $\left( \mu_{x'}  \right)_{x'\in\mathbb R^{d_0'}}\subset \mathcal P(\mathbb R^{d_0})$ and $\left( \nu_{y'}  \right)_{y'\in\mathbb R^{d_1
    '}}\subset \mathcal P(\mathbb R^d_1)$ for which 
\[
    \begin{aligned}
        \mu(A) &=   \int_{\mathbb R^{d_0'}}\left( \int_{T^{-1}(x')}\mathbbm 1_{A}(x)d\mu_{x'}(x)  \right) d\left(T_{\sharp}\mu  \right)(x'),
    \\
    \nu(B) &=  \int_{\mathbb R^{d_1'}}\left( \int_{S^{-1}(y')}\mathbbm 1_{B}(y)d\nu_{y'}(y)  \right) d\left(S_{\sharp}\nu  \right)(y'
    ),
    \end{aligned}
\]
for any Borel sets $A\subset \mathbb R^{d_0},B\subset \mathbb R^{d_1}$ and
\[
\begin{aligned}
    \mu_{x'}\left( T^{-1}(x')  \right)=1\quad \text{for }T_{\sharp}\mu\text{-a.e. } x',\quad 
    \nu_{y'}\left( S^{-1}(y')  \right)=1\quad \text{for }S_{\sharp}\nu\text{-a.e. } y',
\end{aligned}
\]
see Theorem 5.3.1 in \cite{ambrosio2008gradient}.
Define $\pi'\in\mathcal P(\mathbb R^{d_0}\times \mathbb R^{d_1})$ via 
\[
    \pi'(A\times B) = \int_{\mathbb R^{d_0'}\times\mathbb R^{d_1'}}\left( \int_{T^{-1}(x')}\int_{ S^{-1}(y')} \mathbbm 1_{A\times B}(x,y)d\nu_{y'}(y)d\mu_{x'}(x)  \right)d\pi(x',y'), 
\]
for each pair of Borel sets $A\subset\RR^{d_0},B\subset \RR^{d_1}$.
Then, 
\[
    \begin{aligned}
    &\pi'(A\times \mathbb R^{d_1}) = \int_{\mathbb R^{d_0'}\times\mathbb R^{d_1'}}\int_{T^{-1}(x')}\mathbbm 1_{A}(x) d\mu_{x'}(x)\nu_{y'}(S^{-1}(y'))d\pi(x',y')=\mu(A),
    \end{aligned}
\]
so that $\pi'$ has first marginal given by $\mu$; a similar argument shows that its second marginal is $\nu$. Finally, for any Borel sets $C\subset \mathbb R^{d_0'},D\in\mathbb R^{d_1'}$, 
\[
    \begin{aligned}
    (T,S)_{\sharp}\pi'(C\times D) &=\int_{\mathbb R^{d_0'}\times\mathbb R^{d_1'}}\mu_{x'}\left(T^{-1}(x')\cap T^{-1}(C)  \right)\nu_{y'}\left(S^{-1}(y')\cap S^{-1}(D)  \right) d\pi(x',y')
    \\
    &= \int_{\mathbb R^{d_0'}\times\mathbb R^{d_1'}}\mathbbm 1_{C\times D}(x',y') d\pi(x',y') = \pi(C\times D),
    \end{aligned}
\]
proving the desired claim.
\end{proof}

\subsection{\texorpdfstring{Proof of \cref{lem:W2boundIGW}}{Proof of Wasserstein bound for IGW}}
\label{proof:lem:W2boundIGW}
Let $\pi^{\star}$ be an arbitrary solution of $\mathsf{IGW}(\mu,\nu)$ and consider the SVD of the matrix $\int xy^{\intercal}d\pi^{\star}(x,y) = \bm P\blambda_{\pi^{\star}}\bm Q^{\intercal}$ where $\bm P\in\mathbb R^{d_x\times d_x}$ and $\bm Q\in\mathbb R^{d_y\times d_y}$ are orthogonal and $\blambda_{\pi^{\star}}\in\mathbb R^{d_x\times d_y}$ contains the singular values ordered from largest to smallest. Now, let  
$
    \bdelta = \begin{bmatrix}
                \mathbf I_{d_x}&\mathbf 0 
        \end{bmatrix}^{\intercal}\in \St(d_x,d_y)\subset \mathbb R^{d_y\times d_x}.
$
Since $\mathsf{IGW}$ is invariant to orthogonal transformations, $\mathsf{IGW}( (\bdelta\bP^{\intercal})_{\sharp}\mu,(\bQ^{\intercal})_{\sharp}\nu) =\mathsf{IGW}( \mu,\nu)$ and, 
by virtue of \cref{lem:pushforwardMeasures},  $\tilde \pi \coloneqq (\bdelta\mathbf P^{\intercal},\mathbf Q^{\intercal})_{\sharp} \pi^{\star} \in\Pi(\bdelta\bP^{\intercal})_{\sharp}\mu,(\bQ)^{\intercal}_{\sharp}\nu)$ with   
\[
\begin{aligned}
    \mathsf{IGW}( (\bdelta\bP^{\intercal})_{\sharp}\mu,(\bQ)^{\intercal}_{\sharp}\nu)^2&\leq \iint |\langle x,x'\rangle-\langle y,y'\rangle|^2 d \tilde \pi\otimes \tilde \pi (x,y,x',y')
    \\
    &= \iint |\langle \bdelta \bP^{\intercal}x,\bdelta \bP^{\intercal}x'\rangle-\langle \bQ^{\intercal}y,\bQ^{\intercal}y'\rangle|^2 d \pi^{\star}\otimes \pi^{\star} (x,y,x',y')
    \\
    &=\iint |\langle x,x'\rangle-\langle y,y'\rangle|^2 d \pi^{\star}\otimes \pi^{\star} (x,y,x',y')=\mathsf{IGW}(\mu,\nu)^2,
    \end{aligned}
\]
where we have used the fact that $\bdelta \bP^{\intercal}$ and $ \bQ^{\intercal}$ are orthogonal matrices. It follows that $\tilde \pi$ solves $\mathsf{IGW}( (\bdelta\bP^{\intercal})_{\sharp}\mu,(\bQ)^{\intercal}_{\sharp}\nu)$.

Observe that 
\[
\begin{aligned}
   \mathsf{W}_2((\bdelta \bm P^{\intercal})_{\sharp}\mu,(\bm Q^{\intercal})_{\sharp}\nu)^2 &\leq \int \| x  - y\|^2d\tilde \pi(x,y) = \int \|x\|^2 -2 \langle x,y\rangle  +\|y\|^2d\tilde \pi(x,y)
   \\
   &= M_2((\bdelta \bm P^{\intercal})_{\sharp}\mu) + M_2((\bm Q^{\intercal})_{\sharp}\nu) -2\Tr\left( \int xy^{\intercal}d\tilde \pi(x,y) \right). 
   \end{aligned}
\]
On the other hand,  setting $\tilde \mu=(\bdelta \bm P^{\intercal})_{\sharp}\mu$ and $\tilde \nu =(\bm Q^{\intercal})_{\sharp}\nu$
\[
\begin{aligned}
\mathsf{IGW}(\mu,\nu)^2=
   \mathsf{IGW}(\tilde \mu,\tilde \nu)^2 &=\left\|\int xx^{\intercal} d\tilde\mu(x)\right\|^2_{\mathrm F}+ \left\|\int yy^{\intercal} d\tilde\nu(y)\right\|^2_{\mathrm F} -2 \left\|\int xy^{\intercal}d\tilde \pi(x,y)\right\|^2_{\mathrm F}. 
   \end{aligned}
\]
Note that 
\[
\int xy^{\intercal}d\tilde \pi(x,y) = \bdelta \bP^{\intercal}\int xy^{\intercal}d\pi^{\star}(x,y) \bQ =\bdelta \blambda_{\pi^{\star}}  
\]
and observe that $\bdelta \blambda_{\pi^{\star}}=\diag\begin{pmatrix} (\blambda_{\pi^{\star}})_{11} & (\blambda_{\pi^{\star}})_{22}&\cdots &(\blambda_{\pi^{\star}})_{d_xd_x}&0&\cdots&0 \end{pmatrix}\in\mathbb R^{d_y\times d_y}$ so that $\int xy^{\intercal}d\tilde \pi(x,y)$ is a diagonal matrix. Conclude that    
\[
\left\|\int xy^{\intercal}d\tilde \pi(x,y)\right\|^2_{\mathrm F} =\sum_{i=1}^{d_x} (\blambda_{\pi^{\star}})_{ii}^2 \text{ and } \Tr\left(\int xy^{\intercal}d\tilde \pi(x,y)\right) = \sum_{i=1}^{d_x} ( \blambda_{\pi^{\star}})_{ii}
\]
and from the above, $(\blambda_{\pi^{\star}})_{ii} = \int x_iy_id\tilde\pi(x,y)$, where we observe that
\begin{equation}
\label{eq:non_negative_integral}
    0\leq \int (x_i-y_i)^2d\tilde \pi(x,y) = \int x_i^2 d\tilde\mu(x)+\int y_i^2 d\tilde\nu(y)- 2\int x_iy_id\tilde \pi(x,y).  
\end{equation}
Conclude that $0\leq 2\left(\blambda_{\pi^{\star}}\right)_{ii} = 2\int x_iy_id\tilde\pi(x,y)\leq M_2(\tilde\mu)_i +M_2(\tilde \nu)_i$, where we have set  $M_2(\tilde\mu)_i = \int x_i^2 d\tilde \mu(x)$ for $i=1,\dots,d_x$ and similarly for $\tilde \nu$ so that, for each $i=1,\dots, d_x$,   
\begin{equation}
\label{eq:cross_corr_lower_bound}
-2\left(\int x_iy_id\tilde \pi(x,y)\right)^2 \geq - \left(\int x_iy_id\tilde \pi(x,y)\right) \left(M_2(\tilde \mu)_i+M_2(\tilde \nu)_i\right).
\end{equation}
 
 By the Cauchy-Schwarz inequality,  we further have that 
 \begin{equation}
 \label{eq:CSmathcalI}
 0\leq \int x_iy_id\tilde\pi(x,y)\leq \sqrt{M_2(\tilde \mu)_iM_2(\tilde \nu)_i}
 \end{equation}
 so that  $\int x_iy_id\tilde\pi(x,y)=0$ at each $i\in \mathcal I^c$ where $\mathcal I\coloneqq \{i=1,\dots, d_x: M_2(\tilde \mu)_i\text{ and }M_2(\tilde\nu)_i\neq 0\}$ which we assume to be nonempty.    
 Setting 
$\mathcal J\coloneqq \{j=1,\dots,d_y:M_2(\tilde \nu)_j\neq 0\}\cap \mathcal I^c$ and $\mathcal K\coloneqq \{k=1,\dots,d_y:M_2(\tilde \mu)_k\neq 0\}\cap \mathcal I^c$,
\begin{equation}
\label{eq:IGW_lower_bound}
    \begin{aligned}
       \mathsf{IGW}(\mu,\nu)^2 &= \sum_{i,j=1}^{d_y}\left(\int x_i x_jd\tilde\mu(x)\right)^2 + \sum_{i,j=1}^{d_y}\left(\int y_i y_jd\tilde \nu(y)\right)^2 -2 \sum_{i=1}^{d_x}\left( \int x_iy_id\tilde \pi(x,y)\right)^2
        \\
        &\geq \sum_{i\in\mathcal I} \left(M_2(\tilde\mu)_i^2+M_2(\tilde\nu)_i^2-2 \left( \int x_iy_id\tilde \pi(x,y)\right)^2 \right) + \sum_{j\in\mathcal J}M_2(\tilde\nu)^2_j+ \sum_{k\in\mathcal K}M_2(\tilde\mu)^2_k
        \\
        &\geq  \frac{1}{2}\sum_{i\in\mathcal I} \left(M_2(\tilde\mu)_i+M_2(\tilde\nu)_i\right)\left(M_2(\tilde\mu)_i+M_2(\tilde\nu)_i-2\int x_iy_id\tilde \pi(x,y)\right) 
        \\
        &+ \sum_{j\in\mathcal J}M_2(\tilde\nu)^2_j+\sum_{k\in\mathcal K}M_2(\tilde\mu)^2_k,  
    \end{aligned}
\end{equation}
where we have used the standard inequality $(a+b)^2\leq 2(a^2+b^2)$ and \eqref{eq:cross_corr_lower_bound} in the final lower bound. By \eqref{eq:non_negative_integral} we further have that   
\[
\begin{aligned}
    \sum_{i\in\mathcal I} \left(M_2(\tilde \mu)_i+M_2(\tilde \nu)_i\right)\underbrace{\left(M_2(\tilde \mu)_i+M_2(\tilde \nu)_i-2\int x_iy_id\tilde \pi(x,y)\right)}_{\geq 0}&
    \\
    &\hspace{-17em}\geq \sum_{i\in\mathcal I} \min_{i\in\mathcal I}\left\{M_2(\tilde\mu)_i+M_2(\tilde \nu)_i\right\}\!\left(\!M_2(\tilde\mu)_i+M_2(\tilde\nu)_i-2\!\!\int\! x_iy_id\tilde \pi(x,y)\!\right)
    \\
&\hspace{-17em}= \iota \sum_{i\in\mathcal I} \left(M_2(\tilde\mu)_i+M_2(\tilde\nu)_i-2\int x_iy_id\tilde \pi(x,y)\right),
\end{aligned}
\]
for $\iota \coloneqq \min_{i\in\mathcal I}\left\{ M_2(\tilde \mu)_i+M_2(\tilde \nu)_i\right\}$. Applying this bound yields
\[
\begin{aligned}
   &\mathsf{W}_2((\bdelta \bm P^{\intercal})_{\sharp}\mu,(\bm Q^{\intercal})_{\sharp}\nu)^2
   \\
   &\hspace{4em}\leq \sum_{i=1}^{d_y} M_2(\tilde \mu)_i  +\sum_{j=1}^{d_y} M_2(\tilde\nu)_j-2 \sum_{i=1}^{d_x} \int x_iy_id\tilde \pi(x,y)
   \\
   &\hspace{4em}=\sum_{i\in\mathcal I} \left(M_2(\tilde \mu)_i +M_2(\tilde \nu)_i-2 \int x_iy_id\tilde \pi(x,y)\right) +\sum_{j\in\mathcal J} M_2(\tilde \nu)_j +\sum_{k\in\mathcal K} M_2(\tilde \mu)_k
   \\
   &\hspace{4em}\leq \frac{2} {\iota}%
   \left(\mathsf{IGW}(\mu,\nu)^2- \sum_{j\in\mathcal J}M_2(\tilde \nu)_j^2-\sum_{k\in\mathcal K}M_2(\tilde\mu)^2_k\right)
   + \sum_{j\in\mathcal J} M_2(\tilde \nu)_j +\sum_{k\in\mathcal K} M_2(\tilde \mu)_k,
   \end{aligned}
\]
proving the claim.

\subsection{\texorpdfstring{Proof of \cref{lem:calI_nonempty}}{Proof of nonempty index set lemma}}
\label{proof:lem:calI_nonempty}
We assume the notation of the proof of \cref{lem:equivalence_igw_w} above, letting $\pi^{\star}$ solve $\mathsf{IGW}(\mu,\nu)$ and $\tilde \pi$ solve $\mathsf{IGW}(\tilde \mu,\tilde \nu)$. From \eqref{eq:CSmathcalI}, if $\mathcal I=\emptyset$, then $\int x_iy_id\tilde \pi(x,y)=0$ for every $i\in\{1,\dots,d_x\}$ so that all singular values of $\int xy^{\intercal}d\pi^{\star}(x,y)$ are $0$, i.e., $\int xy^{\intercal}d\pi^{\star}(x,y)$ is the $0$ matrix. From the decomposition 
 \[
 \begin{aligned}
 \mathsf{IGW}(\mu,\nu)^2=
    \left\|\int xx^{\intercal} d\mu(x)\right\|^2_{\mathrm F} + \left\|\int yy^{\intercal} d\nu(y)\right\|^2_{\mathrm F}+\inf_{\pi\in\Pi(\mu,\nu)} -2 \left\|\int xy^{\intercal}d \pi(x,y)\right\|^2_{\mathrm F}, 
   \end{aligned}
\] 
we see that every coupling $\pi\in\Pi(\mu,\nu)$ must, therefore, satisfy $\int xy^{\intercal}d\pi(x,y)=0$.

Applying \cref{lem:pushforwardMeasures} it follows that, for any $a\in\mathbb R^{d_0},b\in\RR^{d_1}$, every  coupling, $\pi_{a,b}$, between $(a^{\intercal})_{\sharp}\mu$ and $(b^{\intercal})_{\sharp}\nu$ satisfies $\int xy d\pi_{a,b}(x,y)=0$.  
Assuming to the contrary that $\mu$ and $\nu$ are not point masses, there exist $a,b$ for which $\mu_a = (a^{\intercal})_{\sharp}\mu$ and $\nu_b=(b^{\intercal})_{\sharp}\nu$ are not point masses. Hence, there exist $x_1<x_2$ and $y_1<y_2$ for which $\mu_a((-\infty,x_1]),\mu_a([x_2,\infty)),\nu_b((-\infty,y_1]),\nu_b([y_2,\infty))>0$. Fix some $\epsilon>0$ and consider the measure
\[
\begin{aligned}
\pi' = \mu_a\otimes \nu_b+\epsilon\left(\mu_a\vert_{(-\infty,x_1]}\otimes \nu_b\vert_{[y_2,\infty)}+\mu_a\vert_{[x_2,\infty)}\otimes \nu_b\vert_{(-\infty,y_1]}\right.&\\&\hspace{-10em}\left.-\mu_a\vert_{(-\infty,x_1]}\otimes \nu_b\vert_{(-\infty,y_1]}-\mu_a\vert_{[x_2,\infty)}\otimes \nu_b\vert_{[y_2,\infty)}\right),
\end{aligned}
\]
where $\mu_a\vert_{(-\infty,x_1])}$ is the restriction of $\mu_a$ to $(-\infty,x_1]$ normalized by $\mu_a((-\infty,x_1])$ and so on.
It is easy to see that $\pi'$ is a coupling of $(\mu_a,\nu_b)$ and, for any $A\subset (-\infty,x_1],B\subset (-\infty,y_1]$, 
\[
\begin{aligned}
    \pi'(A\times B)&= \mu_a(A)\nu_b(B)-\epsilon \mu_a\vert_{(-\infty,x_1]}\left(A\right)\nu_b\vert_{(-\infty,y_1]}\left(B\right)
    \\
    &=\mu_a(A)\nu_b(B)\left(1-\frac{\epsilon}{\mu_a((-\infty,x_1])\nu_b((-\infty,y_1])}\right)
\end{aligned}
\]
so that we set $0<\epsilon< \mu_a((-\infty,x_1])\nu_b((-\infty,y_1])$ to ensure that $\pi'(A\times B)\geq 0$. Similarly, if $C\subset [x_2,\infty)$ and $D\subset [y_2,\infty)$, 
\[
\begin{aligned}
    \pi'(C\times D)
    &=\mu_a(C)\nu_b(D)\left(1-\frac{\epsilon}{\mu_a([x_2,\infty))\nu_b([y_2,\infty))}\right)\geq 0
\end{aligned}
\]
if we set $0<\epsilon< \mu_a([x_2,\infty))\nu_b([y_2,\infty))$. We conclude that $\pi'$ is a nonnegative measure with marginals $\mu_a$ and $\nu_b$ once $\epsilon$ is sufficiently small and hence is an element of $\Pi(\mu_a,\nu_b)$. Finally, 
\[
\begin{aligned}
    \int\! xy d\pi'(x,y)\!&=\!\EE_{\mu_a}[X]\mathbb E_{\nu_b}[Y]+\epsilon \EE_{\mu_a\vert_{[x_2,\infty)}}[X]\mathbb E_{\nu_b\vert_{(-\infty,y_1]}}[Y]+\epsilon \EE_{\mu_a\vert_{(-\infty,x_1]}}[X]\mathbb E_{\nu_b\vert_{[y_2,\infty)}}[Y] 
    \\
    &-\epsilon\mathbb E_{\mu_a\vert_{(-\infty,x_1]}}[X]\mathbb E_{\nu_b\vert_{(-\infty,y_1]}}[Y]-\epsilon \EE_{\mu_a\vert_{[x_2,\infty)}}[X]\mathbb E_{\nu_b\vert_{[y_2,\infty)}}[Y]
    \\
    &=\!\EE_{\mu_a}\![X]\mathbb E_{\nu_b}\![Y]\!+\!\epsilon \!\left(\!\EE_{\mu_a\vert_{[x_2,\infty)}}\![X]\!-\!\EE_{\mu_a\vert_{(-\infty,x_1]}}\![X]\!\right)\!\!\left(\!\mathbb E_{\nu_b\vert_{(-\infty,y_1]}}\![Y]\!-\!\mathbb E_{\nu_b\vert_{[y_2,\infty)}}\![Y]\!\right)\!. 
\end{aligned}
\]
Evidently, $\EE_{\mu_a\vert_{[x_2,\infty)}}[X]> \EE_{\mu_a\vert_{(-\infty,x_1]}}[X]$ and $\mathbb E_{\nu_b\vert_{(-\infty,y_1]}}[Y]< \mathbb E_{\nu_b\vert_{[y_2,\infty)}}[Y]$ so that 
$\int xyd\pi'(x,y)<\int xy d\mu_a\otimes \nu_b(x,y)$, contradicting the fact that all couplings must have 0 cross-correlation.

\subsection{\texorpdfstring{Proof of \cref{cor:1dCase}}{Proof of one-dimensional case}}
\label{proof:cor:1dCase}
  If neither $\mu$ nor $\nu$ is a point mass, \eqref{eq:1dComparison} follows from \cref{lem:equivalence_igw_w}. Suppose that $\mu$ is a point mass and $\nu$ is arbitrary. In this case, $\mu\otimes \nu$ is the unique coupling of $\mu$ and $\nu$ so that 
   \[
   \begin{aligned}
    \mathsf W_2(\mu,\nu)^2 &= M_2(\mu)+M_2(\nu)-2\mathbb E_{\mu}[X]\mathbb E_{\nu}[Y]
    \\
    \mathsf W_2(-(\id)_{\sharp}\mu,\nu)^2 &= M_2(\mu)+M_2(\nu)+2\mathbb E_{\mu}[X]\mathbb E_{\nu}[Y]
    \\
    \mathsf{IGW}(\mu,\nu)^2 &= M_2(\mu)^2+M_2(\nu)^2-2\left(\mathbb E_{\mu}[X]\mathbb E_{\nu}[Y]\right)^2.
    \end{aligned}
   \]
   If $\mathbb E_{\mu}[X]\mathbb E_{\nu}[Y]\geq 0$, $\mathsf W_2(\mu,\nu)\leq \mathsf W_2((-\id)_{\sharp}\mu,\nu)$ and  
   \[
        0\leq \int (x-y)^2d\mu\otimes \nu(x,y)= M_2(\mu)+M_2(\nu)-2\mathbb E_{\mu}[X]\mathbb E_{\nu}[Y] 
   \]
   so that $0\leq 2\mathbb E_{\mu}[X]\mathbb E_{\nu}[Y]\leq M_2(\mu)+M_2(\nu)$ whereby 
   \[
   \begin{aligned}
        \mathsf{IGW}(\mu,\nu)^2 &\geq M_2(\mu)^2+M_2(\nu)^2-\mathbb E_{\mu}[X]\mathbb E_{\nu}[Y](M_2(\mu)+M_2(\nu))
        \\
        &=M_2(\mu)^2+M_2(\nu)^2+\frac 12 \left(\mathsf W_2(\mu,\nu)^2-M_2(\mu)-M_2(\nu)\right)(M_2(\mu)+M_2(\nu))
        \\
        &= \underbrace{M_2(\mu)^2+M_2(\nu)^2-\frac{1}{2}\left(M_2(\mu)+M_2(\nu)\right)^2}_{=\frac{1}{2}\left(M_2(\mu)-M_2(\nu)\right)^2\geq 0}+\frac{1}{2}\mathsf W_2(\mu,\nu)^2(M_2(\mu)+M_2(\nu))
    \end{aligned} 
   \]
   proving the claim. If $\mathbb E_{\mu}[X]\mathbb E_{\nu}[Y]\leq 0$, then we work with the expression for $\mathsf W_2((-\id)_{\sharp}\mu,\nu)$ and the argument follows analogously. 

   With \eqref{eq:1dComparison} in hand, if $\mathsf{IGW}(\mu_n,\mu)\to 0$, then 
   \[
        \frac{1}{2}(M_2(\mu_n)+M_2(\mu))\min \left\{\mathsf W_2(\mu_n,\mu)^2,\mathsf W_2((-\id)_{\sharp}\mu_n,\mu)^2\right\}\to 0 
   \]
   so that either $M_2(\mu_n)+M_2(\mu)\to 0$ or $\min \left\{\mathsf W_2(\mu_n,\mu)^2,\mathsf W_2((-\id)_{\sharp}\mu_n,\mu)^2\right\}\to 0$. In the former case, this mean that $\mu_n$ converges weakly to a point mass at $0$ and that $\mu$ is a point mass at $0$. In the later case, we see that $\mathsf W_2\left((s_n\id)_{\sharp}\mu_n,\mu\right)\to 0$ proving the claim. The final assertion is a direct consequence of \cref{prop:IGW_weak_conv}.
\qed

\subsection{\texorpdfstring{Proof of \cref{lem:continuity_orientation}}{Proof of orientation-continuity lemma}}
\label{proof:lem:continuity_orientation}
  Let $a \in \mathbb R^{d_y}$ be arbitrary and $(a_n)_{n\in\mathbb N}\subset \mathbb R^{d_y} $ be a sequence converging to $a$. It is easy to see that the functions $g_n:x\in\mathbb R^{d_x}\mapsto a_n^{\intercal}\bdelta x$ converge uniformly on any compact set to $g:x\in\mathbb R^{d_x}\mapsto a^{\intercal}\bdelta x$ so that Lemma 5.2.1 in \cite{ambrosio2008gradient} yields that $((a_n)^{\intercal}\bdelta)_{\sharp}\mu\stackrel{w}{\to}(a^{\intercal}\bdelta)_{\sharp}\mu$. Similarly, we obtain that $(-\id)_{\sharp}((a_n)^{\intercal}\bdelta)_{\sharp}\mu\stackrel{w}{\to}(-\id)_{\sharp}(a^{\intercal}\bdelta)_{\sharp}\mu$ and that $((a_n)^{\intercal})_{\sharp}\nu\stackrel{w}{\to}(a^{\intercal})_{\sharp}\nu$. Furthermore, 
    \[
        M_2(((a_n)^{\intercal}\bdelta)_{\sharp}\mu)=a_n^{\intercal}\bdelta\int xx^{\intercal}d\mu(x){\bdelta^{\intercal}a_n}\to a^{\intercal}\bdelta\int xx^{\intercal}d\mu(x){\bdelta^{\intercal}a}=M_2((a^{\intercal}\bdelta)_{\sharp}\mu),
    \]
    and the second moments of the other distributions converge similarly. Conclude that 
    \[
    \mathsf W_2(((a_n)^{\intercal}\bdelta)_{\sharp}\mu,((a_n)^{\intercal})_{\sharp}\nu)\to\mathsf W_2((a^{\intercal}\bdelta)_{\sharp}\mu,(a^{\intercal})_{\sharp}\nu)
    \]%
 and, by the same logic,   
 \[
 \mathsf W_2((-\id)_{\sharp}((a_n)^{\intercal}\bdelta)_{\sharp}\mu,((a_n)^{\intercal})_{\sharp}\nu)\to\mathsf W_2((-\id)_{\sharp}(a^{\intercal}\bdelta)_{\sharp}\mu,(a^{\intercal})_{\sharp}\nu).
 \]
    As the choice of $a$ and sequence was arbitrary, the result follows.
\qed

\subsection{\texorpdfstring{Proof of \cref{lemma:IGW_lip}}{Proof of IGW Lipschitz lemma}} 
\label{appen:IGW_lip_proof}
Recall that $\mathsf{IGW}(\mu,\nu)^2=\sF_1(\mu,\nu)+\sF_2(\mu,\nu)$, where 
\begin{align*}
    \sF_1(\mu,\nu)&=\int |\llangle x,x'\rrangle|^2d\mu\otimes \mu(x,x')+\int |\llangle y,y'\rrangle|^2d\nu\otimes \nu(y,y'),
    \\
    \sF_2(\mu,\nu)&=\inf_{\pi\in\Pi(\mu,\nu)} -2\int \llangle x,x'\rrangle\llangle y,y'\rrangle d\pi\otimes \pi(x,y,x',y').
\end{align*}

We begin with Lipschitz continuity in $\theta$. 
For any $\theta,\vartheta\in \mathbb S^{d_y-1}$, we have that 
\begin{align*}
&|\IGW\big((\theta^{\intercal}\bdelta)_{\sharp}\mu,(\theta^{\intercal})_{\sharp}\nu\big)^2 - \IGW\big((\vartheta^{\intercal}\bdelta)_{\sharp}\mu,(\vartheta^{\intercal})_{\sharp}\nu\big)^2\\
&\leq \left| \sF_1((\theta^{\intercal}\bdelta)_{\sharp}\mu,(\theta^{\intercal})_{\sharp}\nu) - \sF_1((\vartheta^{\intercal}\bdelta)_{\sharp}\mu,(\vartheta^{\intercal})_{\sharp}\nu)\right| 
\\
&+ \left| \sF_2((\theta^{\intercal}\bdelta)_{\sharp}\mu,(\theta^{\intercal})_{\sharp}\nu) - \sF_2((\vartheta^{\intercal}\bdelta)_{\sharp}\mu,(\vartheta^{\intercal})_{\sharp}\nu)\right|.
\end{align*}
We bound the two terms separately starting with the second,
\[
    \begin{aligned}
    &\left| \sF_2((\theta^{\intercal}\bdelta)_{\sharp}\mu,(\theta^{\intercal})_{\sharp}\nu) - \sF_2((\vartheta^{\intercal}\bdelta)_{\sharp}\mu,(\vartheta^{\intercal})_{\sharp}\nu)\right|\\
    &=\left|\inf_{\pi\in\Pi(\mu,\nu)}-2 \left(\int \theta^{\intercal}\bdelta xy^{\intercal}\theta d\pi(x,y)\right)^2 - \inf_{\pi\in\Pi(\mu,\nu)}-2 \left(\int \vartheta^{\intercal}\bdelta xy^{\intercal}\vartheta d\pi(x,y)\right)^2\right| 
    \\
    &\leq 2\sup_{\pi\in\Pi(\mu,\nu)}\left|\left(\int \theta^{\intercal}\bdelta xy^{\intercal}\theta d\pi(x,y)\right)^2 - \left(\int \vartheta^{\intercal}\bdelta xy^{\intercal}\vartheta d\pi(x,y)\right)^2\right| 
    \\
    &= 2\sup_{\pi\in\Pi(\mu,\nu)}\left|\int \theta^{\intercal}\bdelta xy^{\intercal}\theta d\pi(x,y) - \int \vartheta^{\intercal}\bdelta xy^{\intercal}\vartheta d\pi(x,y)\right|\times\\
    &\hspace{12em}\left|\int \theta^{\intercal}\bdelta xy^{\intercal}\theta d\pi(x,y) + \int \vartheta^{\intercal}\bdelta xy^{\intercal}\vartheta d\pi(x,y)\right|.
\end{aligned}
\]
Observe that 
\[
    \begin{aligned}
        \theta^{\intercal}\bdelta xy^{\intercal}\theta- \vartheta^{\intercal}\bdelta xy^{\intercal}\vartheta&= (\theta-\vartheta)^{\intercal}\bdelta xy^{\intercal}\theta +\vartheta^{\intercal}\bdelta xy ^{\intercal}(\theta-\vartheta),  
    \end{aligned}
\]
so that 
\[
    \begin{aligned}
    &\left|\int \theta^{\intercal}\bdelta xy^{\intercal}\theta -  \vartheta^{\intercal}\bdelta xy^{\intercal}\vartheta d\pi(x,y)\right|
    \\
    &\leq \int \left|(\theta-\vartheta)^{\intercal}\bdelta xy^{\intercal}\theta  \right| d\pi(x,y) + \int \left|\vartheta^{\intercal}\bdelta xy ^{\intercal}(\theta-\vartheta) \right| d\pi(x,y)\\
    &\leq \int \|\theta-\vartheta\| \|x\| \|y\|   d\pi(x,y) + \int \|x\|\|y\| \|\theta-\vartheta\|  d\pi(x,y)\\
    &\leq 2\|\theta-\vartheta\| \sqrt{M_2(\mu)M_2(\nu)} ,
    \end{aligned}
\]
Applying a similar approach, we have that 
\[
    \left|\int \theta^{\intercal}\bdelta xy^{\intercal}\theta d\pi(x,y) + \int \vartheta^{\intercal}\bdelta xy^{\intercal}\vartheta d\pi(x,y)\right| \leq 2\sqrt{M_2(\mu)M_2(\nu)}.
\]
Overall, 
\[
    \begin{aligned}
    &\left| \sF_2((\theta^{\intercal}\bdelta)_{\sharp}\mu,(\theta^{\intercal})_{\sharp}\nu) - \sF_2((\vartheta^{\intercal}\bdelta)_{\sharp}\mu,(\vartheta^{\intercal})_{\sharp}\nu)\right|
\leq 8M_2(\mu)M_2(\nu)\|\theta-\vartheta\|. 
\end{aligned}
\]
We now bound the remaining term using a similar approach,
\begin{align*}
    &\left| \sF_1((\theta^{\intercal}\bdelta)_{\sharp}\mu,(\theta^{\intercal})_{\sharp}\nu) - \sF_1((\vartheta^{\intercal}\bdelta)_{\sharp}\mu,(\vartheta^{\intercal})_{\sharp}\nu)\right| \\&\hspace{8em}= \left|\int \left(\theta^{\intercal}\bdelta x'x^{\intercal}\bdelta^{\intercal}\theta  \right)^2d\mu\otimes \mu(x,x')-\int \left(\vartheta^{\intercal}\bdelta x'x^{\intercal}\bdelta^{\intercal}\vartheta  \right)^2d\mu\otimes \mu(x,x')\right|\\
    &\hspace{8em}+\left|\int \left(\theta^{\intercal}y'y^{\intercal}\theta  \right)^2d\nu\otimes \nu(y,y')-\int \left(\vartheta^{\intercal}y'y^{\intercal}\vartheta  \right)^2d\nu\otimes \nu(y,y')\right|.
\end{align*}
Observe that 
\[
\begin{aligned}
&\left|\left(\theta^{\intercal}\bdelta x'x^{\intercal}\bdelta^{\intercal}\theta  \right)^2- \left(\vartheta^{\intercal}\bdelta x'x^{\intercal}\bdelta^{\intercal}\vartheta  \right)^2\right|
\\
&=\left|\left(\theta^{\intercal}\bdelta x'x^{\intercal}\bdelta^{\intercal}\theta  \right)^2- \left(\vartheta^{\intercal}\bdelta x'x^{\intercal}\bdelta^{\intercal}\vartheta  \right)^2\right|
\\
&=\left|(\theta-\vartheta)^{\intercal}\bdelta x'x^{\intercal}\bdelta^{\intercal}\theta  + \vartheta^{\intercal}\bdelta x'x^{\intercal}\bdelta^{\intercal}(\theta-\vartheta)  \right|\left|\theta^{\intercal}\bdelta x'x^{\intercal}\bdelta^{\intercal}\theta  + \vartheta^{\intercal}\bdelta x'x^{\intercal}\bdelta^{\intercal}\vartheta  \right|,
\end{aligned}
\]
so that $\left|\left(\theta^{\intercal}\bdelta x'x^{\intercal}\bdelta^{\intercal}\theta  \right)^2- \left(\vartheta^{\intercal}\bdelta x'x^{\intercal}\bdelta^{\intercal}\vartheta  \right)^2\right|\leq \|x\|^2\|x'\|^2\left(\|\theta\|^2+\|\vartheta\|^2\right)(\|\theta\|+\|\vartheta\|)\|\theta-\vartheta\|$, and 
 $\left|\int \left(\theta^{\intercal}\bdelta x'x^{\intercal}\bdelta^{\intercal}\theta  \right)^2d\mu\otimes \mu(x,x')-\int \left(\vartheta^{\intercal}\bdelta x'x^{\intercal}\bdelta^{\intercal}\vartheta  \right)^2d\mu\otimes \mu(x,x')\right|\leq 4 \|\theta-\vartheta\|M_2(\mu)^2$. Similarly,   
$
\left|\int \left(\theta^{\intercal}y'y^{\intercal}\theta  \right)^2d\nu\otimes \nu(y,y')-\int \left(\vartheta^{\intercal}y'y^{\intercal}\vartheta  \right)^2d\nu\otimes \nu(y,y')\right| \leq 4 \|\theta-\vartheta\|M_2(\nu)^2,
$
so that,
\[
\begin{aligned}
\left| \sF_1((\theta^{\intercal}\bdelta)_{\sharp}\mu,(\theta^{\intercal})_{\sharp}\nu) - \sF_1((\vartheta^{\intercal}\bdelta)_{\sharp}\mu,(\vartheta^{\intercal})_{\sharp}\nu)\right| \leq 4 \|\theta-\vartheta\|(M_2(\mu)^2+M_2(\nu)^2).
\end{aligned}
\] 
By compiling these results, we obtain the desired Lipschitz constant.

For the second part of the statement, we have similarly that, for any $\bdelta_1,\bdelta_2\in\mathbb R^{d_y\times d_x}$, 
\begin{align*}
&|\IGW\big((\theta^{\intercal}\bdelta_1)_{\sharp}\mu,(\theta^{\intercal})_{\sharp}\nu\big)^2 - \IGW\big((\theta^{\intercal}\bdelta_2)_{\sharp}\mu,(\theta^{\intercal})_{\sharp}\nu\big)^2\\
&\leq \left| \sF_1((\theta^{\intercal}\bdelta_1)_{\sharp}\mu,(\theta^{\intercal})_{\sharp}\nu) - \sF_1((\theta^{\intercal}\bdelta_2)_{\sharp}\mu,(\theta^{\intercal})_{\sharp}\nu)\right| 
\\
&+ \left| \sF_2((\theta^{\intercal}\bdelta_1)_{\sharp}\mu,(\theta^{\intercal})_{\sharp}\nu) - \sF_2((\theta^{\intercal}\bdelta_2)_{\sharp}\mu,(\theta^{\intercal})_{\sharp}\nu)\right|.
\end{align*}

Following a similar approach, we can bound
\[
    \begin{aligned}
    &\left| \sF_2((\theta^{\intercal}\bdelta_1)_{\sharp}\mu,(\theta^{\intercal})_{\sharp}\nu) - \sF_2((\theta^{\intercal}\bdelta_2)_{\sharp}\mu,(\theta^{\intercal})_{\sharp}\nu)\right| \\
    &\leq 2\sup_{\pi\in\Pi(\mu,\nu)}\left|\left(\int \theta^{\intercal}\bdelta_1 xy^{\intercal}\theta d\pi(x,y)\right)^2 - \left(\int \theta^{\intercal}\bdelta_2 xy^{\intercal}\theta d\pi(x,y)\right)^2\right| 
    \\
    &= 2\sup_{\pi\in\Pi(\mu,\nu)}\left|\int \theta^{\intercal}(\bdelta_1-\bdelta_2) xy^{\intercal}\theta d\pi(x,y)\right|\left|\int \theta^{\intercal}(\bdelta_1+\bdelta_2) xy^{\intercal}\theta d\pi(x,y)\right|\\
    &\leq 2\|\bdelta_1-\bdelta_2\|_{\mathrm{op}}\|\bdelta_1+\bdelta_2\|_{\mathrm{op}} M_2(\mu)M_2(\nu),
\end{aligned}
\]
and, 
\[
    \begin{aligned}
    &\left| \sF_1((\theta^{\intercal}\bdelta_1)_{\sharp}\mu,(\theta^{\intercal})_{\sharp}\nu) - \sF_1((\theta^{\intercal}\bdelta_2)_{\sharp}\mu,(\theta^{\intercal})_{\sharp}\nu)\right| \\
    &=\left|\int \left(\theta^{\intercal}\bdelta_1 x'x^{\intercal}\bdelta_1^{\intercal}\theta  \right)^2-\left(\theta^{\intercal}\bdelta_2 x'x^{\intercal}\bdelta_2^{\intercal}\theta  \right)^2d\mu\otimes \mu(x,x')\right|\\
    &= \left|\int \left(\theta^{\intercal}\bdelta_1 x'x^{\intercal}\bdelta_1^{\intercal}\theta  +\theta^{\intercal}\bdelta_2 x'x^{\intercal}\bdelta_2^{\intercal}\theta  \right)\left(\theta^{\intercal}\bdelta_1 x'x^{\intercal}\bdelta_1^{\intercal}\theta  -\theta^{\intercal}\bdelta_2 x'x^{\intercal}\bdelta_2^{\intercal}\theta  \right)d\mu\otimes \mu(x,x')\right|,
\end{aligned}
\]
where we observe that 
\[
\begin{aligned}
&\left|\left(\theta^{\intercal}\bdelta_1 x'x^{\intercal}\bdelta_1^{\intercal}\theta  +\theta^{\intercal}\bdelta_2 x'x^{\intercal}\bdelta_2^{\intercal}\theta  \right)\left(\theta^{\intercal}\bdelta_1 x'x^{\intercal}\bdelta_1^{\intercal}\theta  -\theta^{\intercal}\bdelta_2 x'x^{\intercal}\bdelta_2^{\intercal}\theta  \right)\right|
\\
&\leq \left(\|\bdelta_1\|^2_{\mathrm{op}}+\|\bdelta_2\|^2_{\mathrm{op}}\right)\|x\|^2\|x'\|^2\|\bdelta_1+\bdelta_2\|_{\mathrm{op}}\|\bdelta_1-\bdelta_2\|_{\mathrm{op}},
\end{aligned}
\]
so that \[
\left| \sF_1((\theta^{\intercal}\bdelta_1)_{\sharp}\mu,(\theta^{\intercal})_{\sharp}\nu) - \sF_1((\theta^{\intercal}\bdelta_2)_{\sharp}\mu,(\theta^{\intercal})_{\sharp}\nu)\right|\leq 2\left(\|\bdelta_1\|^2_{\mathrm{op}}+\|\bdelta_2\|^2_{\mathrm{op}}\right)M_2(\mu)^2\|\bdelta_1-\bdelta_2\|_{\mathrm{op}}.
\]
When $\bdelta_1,\bdelta_2\in\St(d_x,d_y)$, we readily obtain that  
\begin{align*}
&|\IGW\big((\theta^{\intercal}\bdelta_1)_{\sharp}\mu,(\theta^{\intercal})_{\sharp}\nu\big)^2 - \IGW\big((\theta^{\intercal}\bdelta_2)_{\sharp}\mu,(\theta^{\intercal})_{\sharp}\nu\big)^2\\
    &\leq (4M_2(\mu)^2 + 4M_2(\mu)M_2(\nu) )\|\bdelta_1-\bdelta_2\|_{\mathrm{op}} \leq (4M_2(\mu)^2 + 4M_2(\mu)M_2(\nu) )\|\bdelta_1-\bdelta_2\|_{\mathrm{F}}.
\end{align*}
    
\qed

\newpage

\section{\texorpdfstring{Formula for $\aIGW$ between Gaussians}{Formula for sliced IGW between Gaussians}} 
\label{app:aIGWGaussians}
To derive the claimed formula, we first observe that, for any $\bdelta\in\St(d_x,d_y)$ and $\theta \in \mathbb S^{d_y-1}$,  
 \[
\IGW\left((\theta^{\intercal}\bdelta)_{\sharp}\mu,(\theta^{\intercal})_{\sharp}\nu\right)^2 = (\theta^{\intercal}\bdelta \bsigma_\mu \bdelta^{\intercal}\theta)^2 + (\theta^{\intercal}\bsigma_\nu\theta)^2 - 2\theta^{\intercal}\bdelta \bsigma_\mu \bdelta^{\intercal}\theta \theta^{\intercal}\bsigma_\nu\theta,
 \]
 as follows from Corollary 3.2 in \cite{dandapanthula2025optimal}. To integrate this expression over $\mathbb S^{d_y-1}$, we 
underscore that each term is of the form $\theta^{\intercal}\mathbf A \theta \theta^{\intercal}\mathbf B\theta$ for some PSD matrices $\mathbf A,\mathbf B\in\mathbb R^{d_y\times d_y}$. As 
$
    \int_{\mathbb S^{d_y-1}} \theta^{\intercal}\mathbf A \theta \theta^{\intercal}\mathbf B\theta \ d\sigma_{d_y}(\theta) = \sum_{i,j,k,l=1}^{d_y} A_{ij} B_{kl} \int_{\mathbb S^{d_y-1}}\theta_i \theta_j\theta_k \theta_l \ d\sigma_{d_y}(\theta),
$
we apply the following formula for integrating polynomials over the sphere \cite[p.447]{folland2001integrate}, 
$$
\int_{\mathbb S^{d_y-1}}\!\!\!\!\!\!\theta_i \theta_j\theta_k \theta_l \ d\sigma_{\!d_y} \!\!=\! \begin{cases}
    \frac{2\Gamma(1/2)^{d_y-1}\Gamma(5/2)}{\Gamma(2+d_y/2)|\mathbb S^{d_y-1}|} \!=\! \frac{3}{d_y(d_y+2)} & \!\!\text{if } i\!=\!j\!=\!k\!=\!l\\
    \frac{2\Gamma(1/2)^{d_y-2}\Gamma(3/2)^2}{\Gamma(2+d_y/2)|\mathbb S^{d_y-1}|}\!=\! \frac{1}{d_y(d_y+2)} & \!\!\text{if } i\!=\!j\!\neq\! k\!=\!l \text{ up to permuting } (i,j,k,l),
    \\0&\!\!\text{otherwise,}
\end{cases}
$$
whereby,
\begin{align*}
    \int_{\mathbb S^{d_y-1}} \theta^{\intercal}\mathbf A \theta \theta^{\intercal}\mathbf B\theta \ d\sigma_{d_y}(\theta)\hspace{-5em} \\
    &= \frac{3}{d_y(d_y+2)}\sum_{i=1}^{d_y} A_{ii}B_{ii}  + \frac{1}{d_y(d_y+2)}\left(\sum_{i=1}^{d_y}\sum_{\substack{j=1\\j\neq i}}^{d_y}A_{ii}B_{jj} + A_{ij}B_{ij}+A_{ij}B_{ji} \right)\\
    &=\frac{1}{d_y(d_y+2)}\left(\sum_{i,j=1}^{d_y}A_{ii}B_{jj} + A_{ij}B_{ij}+A_{ij}B_{ji} \right)\\
    &=\frac{1}{d_y(d_y+2)}\big(\Tr(\mathbf A)\Tr(\mathbf B)+\Tr(\mathbf A^{\intercal}\mathbf B)+\Tr(\mathbf A\mathbf B) \big).
\end{align*}
Using the above formula and accounting for the minimization over $\St(d_x,d_y)$, we have that 
\begin{align*}
    \aIGW(\mu,\nu)^2  
    &=\!\!\inf_{\bdelta\in \St(d_x,d_y)}\frac{1}{d_y(d_y\!+\!2)}\left(\Tr(\bdelta \bsigma_\mu \bdelta^{\!\intercal})^2\!+\!2\Tr((\bdelta \bsigma_\mu \bdelta^{\!\intercal})^2) \!+\! \Tr(\bsigma_\nu)^2 \!+\! 2\Tr(\bsigma_\nu^2)\right)\\
    &-\frac{2}{d_y(d_y+2)}\left(\Tr(\bdelta \bsigma_\mu \bdelta^{\intercal})\Tr(\bsigma_\nu)+2\Tr(\bdelta \bsigma_\mu \bdelta^{\intercal}\bsigma_\nu) \right)\\
    &=\frac{1}{d_y(d_y+2)}\left(\Tr( \bsigma_\mu )^2+2\Tr(\bsigma_\mu ^2) + \Tr(\bsigma_\nu)^2 + 2\Tr(\bsigma_\nu^2)-2\Tr(\bsigma_\mu)\Tr(\bsigma_\nu)\right)\\
    &-\frac{4}{d_y(d_y+2)}\sup_{\bdelta\in \St(d_x,d_y)}\Tr(\bdelta \bsigma_\mu \bdelta^{\intercal}\bsigma_\nu), 
\end{align*}
where we have used the cyclic property of the trace and the fact that $\bdelta^{\intercal}\bdelta = \bI_{d_x}$. Now, by Von Neumann's trace inequality \cite{neumann1937matrix}, 
$
\Tr(\bdelta \bsigma_\mu \bdelta^{\intercal}\bsigma_\nu) \leq \sum_{i=1}^{d_x}\lambda_i(\bdelta \bsigma_\mu \bdelta^{\intercal})\lambda_i(\bsigma_\nu)=\sum_{i=1}^{d_x}\lambda_i(\bsigma_\mu )\lambda_i(\bsigma_\nu)
$. $\bsigma_{\mu}$ and $\bdelta \bsigma_{\mu} \bdelta^{\intercal}$. Diagonalizing both covariance matrices as $\bsigma_{\mu}=\bm U\bm D_{\mu}\bm U^{\intercal}$ and $\bsigma_{\nu}=\bm V\bm D_{\nu}\bm V^{\intercal}$, we see that the trace inequality is saturated for 
 $\bdelta =\bV \bJ\bU^{\intercal}$ where $\bJ = [\bI_{d_x}\  \mathbf{0}]^\intercal$. Inserting this formula in the above display proves the assertion.

\newpage

\section{Proof of Equivalence for Distributions on Lattices}
\label{sec:rem:lattice}

Let $\nu$ be supported in a full rank lattice $\Lambda\coloneqq \{\mathbf{B}x:x\in\mathbb{Z}^d\}$ where $\mathbf{B}\in \mathbb{R}^{d\times d}$ is invertible and that $\aIGW(\mu,\nu)=0$. We remove the full-rank assumption at the end of the proof. Following the proof of \cref{prop:metric}, there exists $\bdelta\in\St(d,d)$ and a set $C\subseteq \mathbb{R}^d$ satisfying the property that if $x\in C$, then $tx \in C$ for each $t\in\mathbb R$ for which $\Phi_{(\bdelta)_{\sharp}\mu}=\Phi_\nu \text{ on }C,$ $\Phi_{(\bdelta)_{\sharp}\mu}=\overline{\Phi_\nu} \text{ on } C^c$. 
One can check $C^c\cup \{0\}$ also  satisfies the property that if $x\in C^c\cup \{0\}$, then $tx \in C^c\cup \{0\}$ for each $t\in\mathbb R$.

If either $C = \emptyset$ or $C^c= \emptyset$, the conclusion is immediate. 
Suppose that $C\neq \emptyset$ and $C^c\neq \emptyset$ and fix $a\in C$ for which the imaginary part of $\Phi_{\nu}(a)$ is  not zero. If such an $a$ does not exist, $\Phi_{\nu}(x)$ is real for each $x\in C$ and  $\Phi_\nu(x)=\overline{\Phi_\nu(x)}$ for each $x\in C$ whereby $\Phi_{(\bdelta)_{\sharp}\mu}=\overline{\Phi_{\nu}}$ and the proof is complete. The existence of such $a$ also implies $C$ contains a cone with an open base (i.e. there exists $U$ open in $\mathbb S^{d-1}$ such that
$
\{tx:t\in \mathbb{R}, x \in U \} \subseteq C
$), as in proof of \cref{prop:metric}. By similar argument, $C^c \cup \{0\}$ also contains such a cone.

We now prove that there exists a $b\in C^c$ with $\Phi_{\nu}(b)=1$. To this end, let $\Lambda^*$ be the dual lattice to $\Lambda$, i.e., the set of all $y\in\mathbb R^d$ for which $\{y\in\mathrm{span}(\mathbf B\mathbb Z^d):y^{\intercal}x\in\mathbb Z,\text{ for all } x\in \Lambda\}$. Since $\Lambda$ is full rank, $\Lambda^{*}=\mathbf B^{-\intercal}\mathbb Z^d$ is also a nonempty full-rank lattice and, for each $b\in2\pi\Lambda^{*}$,
\[
    \Phi_{\nu}(b) = \int_{\Lambda} e^{i b^{\intercal} x} d\nu (x)=\nu (\Lambda)=1,
\] 
since $b^{\intercal} x \in 2\pi\mathbb Z$ for each $x\in\Lambda$. It remains to show that $\Lambda^*\cap C^c\neq \emptyset$. To this end,
recall $C^c\cup\{0\}$ contains a cone $K$ and $-K$. By Minkowski's theorem (see Theorem 2.3.4 in \cite{hans2024lecture}), $(K\cup (-K))\backslash\{0\}\subset C^c$ contains a point $\lambda\in\Lambda^*$. It follows that there exists $z\in\mathbb Z^n$ with $\lambda=\mathbf B^{-\intercal}z$ and so $k\lambda \in \Lambda^*$ for any $k\in\mathbb Z$ so that $2\pi k\lambda \in 2\pi \Lambda^*\cap C^c$ for each $k\in\mathbb Z\backslash\{0\}$. Since $\lambda \in (K\cup (-K))\backslash\{0\}$ lies in the interior of the cone, there exists an open ball of radius $r>0$ centered at $2\pi\lambda$, denoted $B(2\pi\lambda,r)$, which is contained in $C^c$. Thanks to the scaling property, for any positive integer $k$, $2\pi k\lambda\in C^c$, and $B(2\pi k\lambda,kr)\subset C^c$ as if $\|l-2\pi k\lambda\|\leq kr$, then $\|l/k-2\pi\lambda\|\leq r$ so that $l/k \in B(2\pi\lambda,r)\subset C^c$. Letting $k$ be a positive integer greater than $\|a\|$ and setting $b=2\pi k\lambda$ it holds that $a\in C,b\in C^c,a+b\in C^c$, $\Phi_{\nu}(b)=1$, and $\Phi_{\nu}(a)$ has non-zero imaginary part.  

With this, we consider the Gram matrix of $\Phi_{(\bdelta)_{\sharp}\mu}$ at the points $0,a,a+b$, given by 
\[\begin{aligned}
    \begin{pmatrix}
        \Phi_{(\bdelta)_{\sharp}\mu}(0) & \Phi_{(\bdelta)_{\sharp}\mu}(-a) & \Phi_{(\bdelta)_{\sharp}\mu}(-(a+b))
        \\ 
        \Phi_{(\bdelta)_{\sharp}\mu}(a) & \Phi_{(\bdelta)_{\sharp}\mu}(0) & \Phi_{(\bdelta)_{\sharp}\mu}(-b)
        \\
        \Phi_{(\bdelta)_{\sharp}\mu}(a+b) & \Phi_{(\bdelta)_{\sharp}\mu}(b) & \Phi_{(\bdelta)_{\sharp}\mu}(0)
    \end{pmatrix}&
    \\
    &\hspace{-8em}= \begin{pmatrix}
        1 & \Phi_{\nu}(-a) & \overline{\Phi_{\nu}(-(a+b))}
        \\ 
        \Phi_{\nu}(a) & 1 & 1
        \\
        \overline{\Phi_{\nu}(a+b)} & 1 & 1
    \end{pmatrix},
\end{aligned}
\]
where we note that 
\[
    \Phi_{\nu}(a+b) = \int_{\Lambda} e^{ia^{\intercal}x}e^{ib^{\intercal}x}d\nu(x) = \int_{\Lambda} e^{ia^{\intercal}x}d\nu(x)=\Phi_{\nu}(a)
\]
since $b^{\intercal} x \in 2\pi\mathbb Z$ for each $x\in \Lambda$ by construction. Similarly, $\Phi_{\nu}(-(a+b))=\Phi_{\nu}(-a)$. It follows that the determinant of the above matrix is 
\[
\begin{aligned}
   & 1+\Phi_{\nu}(-a)\overline{\Phi_{\nu}(a)}+\Phi_{\nu}(a)\overline{\Phi_{\nu}(-a)}-\overline{\Phi_{\nu}(-a)}\overline{\Phi_{\nu}(a)}-{\Phi_{\nu}(-a)}{\Phi_{\nu}(a)}-1
   \\
   & =\Phi_{\nu}(-a)^2+\Phi_{\nu}(a)^2-2{\Phi_{\nu}(-a)}{\Phi_{\nu}(a)}=\left(\Phi_{\nu}(-a)-{\Phi_{\nu}(a)}\right)^2,
\end{aligned}
\]
where we have used the fact that $\Phi_{\nu}(-t) = \overline{\Phi_{\nu}(t)}$ for any $t\in\mathbb R^d$. 
Since $\Phi_{\nu}(a)=c+di$ with $d\neq 0$ by construction, $\left(\Phi_{\nu}(-a)-{\Phi_{\nu}(a)}\right)^2=(-2di)^2=-4d^2<0$. However, Bochner's theorem (cf. e.g., Theorem 7.13.1 \cite{bogachev2007measure}) asserts that the characteristic function of a measure must be positive definite, i.e., each of its Gram matrices are positive semidefinite. As the above Gram matrix is not positive semidefinite, $\Phi_{(\bdelta)_{\sharp}\mu}$ is not a characteristic function unless $\Phi_{(\bdelta)_{\sharp}\mu}=\Phi_{\nu}$ or $\Phi_{(\bdelta)_{\sharp}\mu}=\overline{\Phi_{\nu}}$ on the entire space.  

Now we return to the case when the lattice $\Lambda$ does not have full rank. In this case, let the rank be $d'<d$, and basis $\mathbf{B} = [\mathbf{b}_1, \ldots, \mathbf{b}_{d'}]$ with linearly independent columns. Since $\nu$ is supported on $\Lambda= \{\mathbf{B}x:x\in\mathbb{Z}^{d'}\}$, 
\[
\Phi_{\nu}(x)
= \int_{\Lambda} e^{i \langle x, y\rangle}\, d\nu(y)
= \int_{\Lambda} e^{i \langle \mathrm{proj}_{\mathrm{span}(\mathbf B)} x,\, y\rangle}\, d\nu(y)
= \Phi_{\nu}\bigl(\mathrm{proj}_{\mathrm{span}(\mathbf B)} x\bigr).
\]
Therefore, the characteristic function $\Phi_{\nu}$ is completely determined by its values on $\mathrm{span}(\mathbf B)$. So we can restrict the analysis to the subspace $\mathrm{span}(\mathbf B)$. Notice that for the general case, the dual lattice can be written as
$\Lambda^{*}=\mathbf B(\mathbf B^{\intercal}\mathbf B)^{-1}\mathbb Z^{d'}$, which is also a full rank lattice when restricted to $\mathrm{span}(\mathbf B)$.
The same argument as the full rank case then applies, with 
$C$ and $C^c$ being replaced by $C\in \mathrm{span}(\mathbf B)$ and $\mathrm{span}(\mathbf B)\setminus C$,
and with the notion of interior and openness interpreted relative to $\mathrm{span}(\mathbf B)$.

\newpage

\section{\texorpdfstring{Convergence of \cref{alg:sgm}}{Convergence of the subgradient method}}
\label{sec:alg}

Given that the proof of \cref{thm:convergenceSGM} requires a number of auxiliary results, we provide additional exposition in this section.

With \cref{prop:subgradient} in hand, it is relatively straightforward to derive bounds on the (local) Lipschitz constant of $F$. It will be helpful, in the sequel, to establish estimates on this constant over the set of all points $\bdelta\in\mathbb R^{d_y\times d_x}$ which are close to the Stiefel manifold in the sense that $\|\bdelta^{\intercal}\bdelta-\bI_{d_x}\|_{\mathrm{F}}\leq R$ for some fixed $R>0$.  
\begin{corollary}[Subgradient bound]  
\label{cor:LipschitzConstant}  In the setting of \cref{prop:subgradient}, for any $\bm \Delta\in\mathbb R^{d_y\times d_x}$, \[\begin{aligned}\sup_{\bm S\in\partial F(\bm\Delta)}\|\bm S\|_{\mathrm{F}}&\leq 4M_2(\mu)M_2(\nu)\|\bm \Delta\|_{\mathrm{op}}+4M_2(\mu)^2\|\bm\Delta\|_{\mathrm{op}}^3\text{ and }\|\bm\Delta\|_{\mathrm{op}}\leq \sqrt{1+\|\bm \Delta^{\intercal}\bm \Delta-\bm{I}_{d_x}\|_{\mathrm F}}.
    \end{aligned} 
  \]
  If
   $\bm \Delta,\bar{\bm \Delta}\in\mathbb R^{d_y\times d_x}$ are such that $\max\{\|\bm \Delta^{\intercal}\bm \Delta-\mathbf{I}_{d_x}\|_{\mathrm{F}},\|\bar {\bm \Delta}^{\intercal}\bar{\bm \Delta}-\mathbf{I}_{d_x}\|_{\mathrm{F}}\}\leq R$, then
  \[
    \left|F(\bm \Delta)-F(\bar{\bm \Delta})\right|\leq L_R\|\bm\Delta-\bar{\bm \Delta}\|_{\mathrm{F}},
  \]
  for $L_R\coloneqq 4M_2(\mu)\sqrt{1+R}\left(M_2(\nu)+M_2(\mu)(1+R)\right).$ 
   \end{corollary}
 
\begin{proof}
   We first establish the inequality  $\|\bm\Delta\|_{\mathrm{op}}\leq \sqrt{1+\|\bm \Delta^{\intercal}\bm \Delta-\bm{I}_{d_x}\|_{\mathrm F}}$ as follows,
    \[
        \|\bm \Delta^{\intercal}\bm \Delta-\bm{I}_{d_x}\|_{\mathrm F}^2=\sum_{i=1}^{d_x}\left(\lambda_i(\bm \Delta^{\intercal}\bm\Delta)-1\right)^2\geq\left( \lambda_1(\bm\Delta^{\intercal}\bm \Delta)-1\right)^2=\left(\|\bm\Delta\|_{\mathrm{op}}^2-1\right)^2.
    \]
The bound on the norm of the subgradients is a direct consequence of Proposition 2.1.2 in \cite{clarke1990optimization} and  the proof of \cref{lemma:IGW_lip}, see \cref{appen:IGW_lip_proof}. The final claim also follows directly from the proof of \cref{lemma:IGW_lip}. 
\end{proof}

\begin{proposition}[Propositions 3.1, 3.2, and 3.5 %
in \cite{hu2024constraint}]
\label{thm:diffH}
Let $H: \mathbf{\Delta} \in\mathbb R^{d_y\times d_x}\mapsto \mathbb R$ 
be given as 
\[
\begin{aligned}
    H(\mathbf \Delta) =  F(A(\mathbf{\Delta})) + \frac{\beta}{4}\left\| \mathbf{\Delta}^{\intercal}\mathbf{\Delta}-\mathbf{I}_{d_x}\right\|_{\mathrm{F}}^2,\text{ where } 
    A(\mathbf{\Delta}) = \frac{1}{8}\mathbf{\Delta}\left(  15\mathbf{I}_{d_x} - 10\mathbf{\Delta}^{\intercal}\mathbf{\Delta}+3(\mathbf{\Delta}^{\intercal}\mathbf{\Delta})^2 \right),
\end{aligned}
\]
and $\beta >0$ is a constant. Then, for any $\bdelta\in\mathbb R^{d_y\times d_x}$,
\begin{enumerate}[leftmargin=*]
    \item $\|\bdelta^{\intercal}\bdelta-\bI_{d_x}\|^3_{\mathrm{F}}\geq \|A(\bdelta)^{\intercal}A(\bdelta)-\bI_{d_x}\|_{\mathrm{F}}$, 
    \item $
\partial H(\mathbf{\Delta}) = DA_{[\mathbf \Delta]}(\partial F(A(\bm \Delta))) + \beta\mathbf \Delta(\bm \Delta^{\intercal}\bm\Delta-\bm I_{d_x}),$
\item $
DA_{[\mathbf \Delta]}(\bm \Xi)
\mspace{-2mu}=\mspace{-2mu}\frac{1}{8}\bm \Xi\left(15 \bm I_{d_x}\mspace{-2mu}-\mspace{-2mu}10\bm\Delta^{\intercal}\bm \Delta\mspace{-2mu}+\mspace{-2mu} 3(\bm \Delta^{\intercal}\bm \Delta)^2\right)\mspace{-2mu}-\mspace{-2mu}\bm \Delta\Phi\left( \mathbf \Delta^{\intercal}\mathbf \Xi \right)\mspace{-1mu} + \mspace{-1mu}\frac 32 \mathbf \Delta\Phi\left(\Phi(\mathbf \Delta^{\intercal}\mathbf \Xi)\mspace{-2mu} \left(\mathbf \Delta^{\intercal}\mathbf \Delta\mspace{-2mu}-\mspace{-2mu}\mathbf I_{d_x}\right)\right)$ where, for a matrix $\bm M\in\mathbb R^{d_y\times d_x}, \Phi(\bm M)=\frac{1}2 (\mathbf M+\mathbf {M}^{\intercal})$.
\end{enumerate}
\end{proposition}

The main insight in \cite{hu2024constraint}, which we exposit below, is that applying unconstrained optimization procedures to $H$ can yield, in certain cases, a (Riemannian) Clarke critical point for $F$ and even local minimizers of $F$ over $\St(d_x,d_y)$.
\cref{thm:diffH} endows us with a formula for the subdifferential of $H$ at $\bm\Delta$ in terms of the subdifferential for $F$ at $A(\bdelta)$ so that subgradient methods can easily be implemented and shows that, if $\bdelta\in\mathbb R^{d_y\times d_x}$ $\|\bdelta^{\intercal}\bdelta-\bm{I}_{d_x}\|_{\mathrm{F}}< 1$, then $A(\bdelta)$ violates the Stiefel manifold constraint less than $\bdelta$ so that $A(\bdelta)$ can be thought of as a correction term that whereas $\frac{\beta}{4}\|\bdelta^{\intercal}\bdelta-\bm{I}_{d_x}\|_{\mathrm{F}}$ serves as a penalty term. The following result establishes that, under certain conditions, a critical point of $H$ is in fact a Riemannian critical for $F$ over the Stiefel manifold.     

\begin{proposition}[Theorems 3.8 and 3.12 and Remark 3.6 %
in \cite{hu2024constraint}]
\label{thm:stationaryPoints}
Suppose that ${\mathbf \Delta}\in\mathbb R^{d_y\times d_x}$ satisfies  $\left\|{\mathbf \Delta}^{\intercal}{\mathbf \Delta}-\mathbf{I}_{d_x}\right\|_{\mathrm{F}}\leq \frac{\beta}{2\beta +8 L_1}$ for $L_1$ as defined in \cref{cor:LipschitzConstant}, then 
\[
\inf_{\mathbf W\in\partial H(\bf {\Delta})}\|\bm W\|_{\mathrm F} \geq \frac \beta 4\left\|{\mathbf \Delta}^{\intercal}{\mathbf \Delta}-\mathbf{I}_{d_x}\right\|_{\mathrm{F}}\geq \frac{\beta}{4}\|A(\bm \Delta)^{\intercal}A(\bm \Delta)-\bm I_{d_x}\|_{\mathrm{F}}^{1/3}.
\]
Moreover, if $\mathbf \Delta\in \St(d_x,d_y)$, then $A(\bdelta)=\bdelta$ and $\partial H(\bm \Delta)=\partial_RF(A(\bm \Delta))=\partial_RF(\bm \Delta)$.
Finally, if $\bdelta^{\star}$  is a local minimizer of $H$ satisfying $\left\|{(\mathbf \Delta^{\star})}^{\intercal}{\mathbf \Delta^{\star}}-\mathbf{I}_{d_x}\right\|_{\mathrm{F}}\leq \frac{\beta}{2\beta+8L_1}$, then $\bdelta^{\star}$ is a local minimizer of $F$ over $\St(d_x,d_y)$.
\end{proposition}

\cref{thm:stationaryPoints} implies, in particular, that if $H$ admits a critical point, $\bdelta\in\mathbb R^{d_y\times d_x}$, which violates the Stiefel manifold condition by less than a prescribed tolerance, $\bdelta$ must actually be an element of the Stiefel manifold and is critical in the Riemannian sense. Local minimality also transfers over in a similar manner.    

As \cref{thm:stationaryPoints} requires a certain \emph{a priori} bound on how much $\mathbf \Delta$ violates the Stiefel manifold constraint, it is convenient to establish sufficient conditions for the iterates of a minimization procedure applied to $H$ to satisfy these conditions. Precisely, we consider iterates generated by \cref{alg:sgm} which can be shown to satisfy $\|\bm \Delta_k^{\intercal}\bm \Delta_k-\bm I_{d_x}\|\leq \frac 16$ for all $k\in\mathbb N$ provided that the algorithm is initialized at a point, $\bdelta_1\in\mathbb R^{d_x\times d_y}$, satisfying the same constraint.

\begin{proposition}[Stiefel violation]
\label{prop:boundedIterates}
   Fix $\alpha=4 M_2(\mu)\sqrt{\frac{217}{216}}\left(M_2(\nu)+\frac{217}{216}M_2(\mu)\right),\beta\geq 162\alpha,$ and let $\eta_k\leq \frac{1}{2\beta}$ for all $k\in\mathbb N$. Then, the iterates of \cref{alg:sgm} satisfy $\|\bdelta_k^{\intercal}\bdelta_k-\bI_{d_x}\|_{\mathrm{F}}\leq 1/6$. 
\end{proposition}
\begin{proof}
  The proof of \cref{prop:boundedIterates} follows directly from Lemma 4.12 in \cite{hu2024constraint} once we establish that, at any $\bm \Delta_k\in\mathbb R^{d_x\times d_y}$ with $\|\bdelta_k^{\intercal}\bdelta_k-\bI_{d_x}\|_{\mathrm F}\leq \frac 16$ it holds that $\|D A_{[\bdelta_k]}(\bm S_k)\|_{\mathrm F}\leq 2.6\alpha$. First, we have from \cref{cor:LipschitzConstant} that for any $\bm S_k\in\partial F(A(\bm \Delta_k))$, 
  \[
  \begin{aligned}
        \|\bm S_k\|_{\mathrm{F}}\mspace{-1mu}\leq\mspace{-1mu} 4M_2(\mu){\sqrt{1\mspace{-1mu}+\mspace{-1mu}\|A(\bm \Delta_k)^{\intercal}A(\bm \Delta_k)\mspace{-1mu}-\mspace{-1mu}\bI_{d_x}\|_{\mathrm{F}}}}\times&\\&\hspace{-4em}\left( M_2(\nu)\mspace{-1mu}+\mspace{-1mu}M_2(\mu)\left(1\mspace{-1mu}+\mspace{-1mu}\|A(\bm \Delta_k)^{\intercal}A(\bm \Delta_k)\mspace{-1mu}-\mspace{-1mu}\bI_{d_x}\|_{\mathrm{F}}\right)\right).
 \end{aligned} 
  \]
  By \cref{thm:diffH}, if $\|\bdelta_k^{\intercal}\bdelta_k-\bI_{d_x}\|_{\mathrm{F}}\leq \frac 16$,  $\|A(\bdelta_k)^{\intercal}A(\bdelta_k)-\bI_{d_x}\|_{\mathrm{F}}\leq\|\bdelta_k^{\intercal}\bdelta_k-\bI_{d_x}\|^3_{\mathrm{F}}\leq \frac 1{216}$ so that $\|\bm S_k\|_{\mathrm{F}}\leq \alpha$. Now, recall that 
  \[
  DA_{[\mathbf \Delta]}(\bm \Xi)
\mspace{-1mu}=\mspace{-1mu}\frac{1}{8}\bm \Xi\left(15 \bm I_{d_x}-10\bm\Delta^{\intercal}\bm \Delta\mspace{-1mu}+\mspace{-1mu} 3(\bm \Delta^{\intercal}\bm \Delta)^2\right)\mspace{-1mu}-\mspace{-1mu}\bm \Delta\Phi\left( \mathbf \Delta^{\intercal}\mathbf \Xi \right)\mspace{-1mu} +\mspace{-1mu} \frac 32 \mathbf \Delta\Phi\left(\Phi(\mathbf \Delta^{\intercal}\mathbf \Xi) \left(\mathbf \Delta^{\intercal}\mathbf \Delta\mspace{-1mu}-\mspace{-1mu}\mathbf I_{d_x}\right)\right)\] 
for $\Phi(\mathbf M)=\frac 12 (\mathbf M+\mathbf M^{\intercal})$. We now simplify the various expressions. First, 
\[
\begin{aligned}
&\frac{1}{8}\bm \Xi\left(15 \bm I_{d_x}-10\bm\Delta^{\intercal}\bm \Delta+ 3(\bm \Delta^{\intercal}\bm \Delta)^2\right)
\\
&=\frac{1}{8}\bm \Xi\left(8\mathbf I_{d_x}+7\left(\bm I_{d_x}-\bm\Delta^{\intercal}\bm \Delta\right)+ 3(\bm \Delta^{\intercal}\bm \Delta)\left(\bm \Delta^{\intercal}\bm \Delta-\mathbf I_{d_x}\right)\right)
\\
&=\bm \Xi+\frac78\bm \Xi\left(\bm I_{d_x}-\bm\Delta^{\intercal}\bm \Delta\right)+ \frac 38\bm \Xi \left(\bm \Delta^{\intercal}\bm \Delta-\mathbf I_{d_x}\right)^2+\frac38\bm \Xi\left(\bm \Delta^{\intercal}\bm \Delta-\mathbf I_{d_x}\right),
\end{aligned}
\]
whereby $\frac{1}{8}\bm S_k\left(15 \bm I_{d_x}-10\bm\Delta_k^{\intercal}\bm \Delta_k+ 3(\bm \Delta_k^{\intercal}\bm \Delta_k)^2\right)\leq \alpha +\frac{10}{48}\alpha+\frac{3}{288}\alpha=\frac{39}{32}\alpha$. 

For the next term, it holds that
\[
\|\bm\Delta^{\intercal}_k\bm S_k\|_{\mathrm{F}}\leq \|\bm\Delta^{\intercal}_k\|_{\mathrm{op}}\|\bm S_k\|_{\mathrm{F}}\leq \alpha\sqrt{1+\|\bm\Delta_k^{\intercal}\bm \Delta-\bI_{d_x}\|_{\mathrm{F}}}\leq \alpha \sqrt{\frac{7}{6}}
\]
so that $\|\bm\Delta_k\Phi(\bm\Delta^{\intercal}_k\bm S_k)\|_{\mathrm{F}}\leq \|\bm\Delta_k\|_{\mathrm{op}}\|\Phi(\bm\Delta^{\intercal}_k\bm S_k)\|_{\mathrm{F}}\leq \frac 76 \alpha$ noting that, for any matrix $\bm M\in\mathbb R^{d_y\times d_x}$, $\|\Phi(\bm M)\|_{\mathrm{F}}\leq \frac{1}{2}(\|\bm M^{\intercal}\|_{\mathrm{F}}+\|\bm M\|_{\mathrm{F}})=\|\bm M\|_{\mathrm{F}}$.  

Using this logic for the final term we have 
\[
\begin{aligned}
\frac 32 
\|\mathbf \Delta_k\Phi\left(\Phi(\mathbf \Delta_k^{\intercal}\mathbf S_k) \left(\mathbf \Delta_k^{\intercal}\mathbf \Delta_k-\mathbf I_{d_x}\right)\right)\|_{\mathrm{F}}
&\leq\frac 32 
\|\mathbf \Delta_k\|_{\mathrm{op}}\|\Phi(\mathbf \Delta_k^{\intercal}\mathbf S_k) \left(\mathbf \Delta_k^{\intercal}\mathbf \Delta_k-\mathbf I_{d_x}\right)\|_{\mathrm{F}}
\\
&\leq\frac 32 
\|\mathbf \Delta_k\|_{\mathrm{op}}\|\mathbf \Delta_k^{\intercal}\mathbf S_k\|_{\mathrm{F}}\|\mathbf \Delta_k^{\intercal}\mathbf \Delta_k-\mathbf I_{d_x}\|_{\mathrm{F}}\leq \frac 7{24}\alpha.
\end{aligned}
\]
Applying the triangle inequality, we obtain that 
$
    \|DA_{[\bm \Delta_k]}(\bm S_k)\|_{\mathrm{F}}\leq \left(\frac{39}{32}+\frac{7}{6}+ \frac{7}{24}\right)\alpha=\frac{743}{288}\alpha\leq 2.7 \alpha, 
$
so that we may apply Lemma 4.12 in \cite{hu2024constraint} which asserts that if $\beta\geq 60\cdot 2.7 \alpha$ and $\eta_k\leq \frac 1{2\beta}$, the subgradient method generates iterates $\bm \Delta_k$ satisfying $\|\bm \Delta_k^{\intercal}\bm \Delta_k-\bI_{d_x}\|_{\mathrm{F}}\leq 1/6$ provided that $\bm \Delta_0$ also satisfies this condition. 
\end{proof}

With \cref{prop:boundedIterates} in hand, it remains to show that the iterates generated by \cref{alg:sgm} converge in a suitable sense. To this end, we leverage Corollary 5.9 in \cite{davis2020stochastic} (see also Corollary 4.16 in \cite{hu2024constraint}) for which it suffices to prove that $H$ is a semialgebraic function, that is, the graph of $H$, $\mathrm{grp}(H)\coloneqq \left\{\left(\mathrm{vec}(\mathbf \Delta),y\right)\in\mathbb R^{d_xd_y}\times \mathbb R:H(\mathbf \Delta)=y \right\}$, can be written as $\{(\mathrm{vec}(\bdelta),y)\in\mathbb R^{d_xd_y}\times \mathbb R:p_{i}(\mathrm{vec}(\bdelta),y)\leq 0 \text{ for } i=1,\dots, \ell\}$ where $(p_i)_{i=1}^{\ell}$ is a finite collection of polynomials; any set of this form is said to be  semialgebraic. 

\begin{lemma}[Semialgebraicity of $H$]
\label{lem:semialgebraic}
   The function $\bm \Delta\in\mathbb R^{d_y\times d_x}\mapsto H(\bm \Delta)$ is semialgebraic. 
\end{lemma}
\begin{proof}
We first compile some useful properties of semialgebraic maps. The first two points can be found on p.29 in \cite{coste2002introduction} and the last two in Corollary 2.9 of the same reference,  
\begin{enumerate}[(1),leftmargin=*]
    \item any map with polynomial coordinates is semialgebraic, 
    \item the absolute value of a semialgebraic function is semialgebraic,
    \item  the composition of two semialgebraic mappings is semialgebraic,
    \item semialgebraic functions are closed under addition and scalar multiplication.
\end{enumerate}
As a direct corollary to items (2) and  (4), we obtain that the pointwise minimum of two semialgebraic functions $f,g:\mathbb R^{d_y\times d_x}\to \mathbb R$ is semialgebraic as 
$
\min\{f(x),g(x)\}=\frac{1}{2}\left(f(x)+g(x)-|f(x)-g(x)|\right)
$. Functions defined as the pointwise minimum of finitely many semialgebraic functions can be seen to be semialgebraic by iteratively applying the above result.

With these points in mind, we note that $A$ is a semialgebraic mapping and that
\[
\begin{aligned}
P:
\bdelta\in\mathbb R^{d_y\times d_x}&\mapsto \frac{\beta}{4}\|\bdelta^{\intercal}\bdelta-\bI_{d_x}\|_{\mathrm{F}}^2, 
\\
U_i:\bdelta\in\mathbb R^{d_y\times d_x}&\mapsto \left(\left(\theta^{(i)}\right)^{\intercal}\bdelta\int zz^{\intercal}d\mu(z)\bdelta^{\intercal}\theta^{(i)}\right)^2,\; (i=1,\dots,R)
\\
V_{i,\pi}:\bdelta\in\mathbb R^{d_y\times d_x}&\mapsto -2\left(\left(\theta^{(i)}\right)^{\intercal}\bdelta \int xy^{\intercal}d\pi(x,y)\theta^{(i)}\right)^2,\; (i=1,\dots,R, \pi\in\Pi(\mu,\nu))
\end{aligned}
\]
are all polynomials of $\mathrm{vec}(\bm\Delta)$ and are hence semialgebraic. Since
\[
F(\bdelta)= \frac{1}{R}\sum_{i=1}^{R}\left(U_i(\bdelta) +M_2(((\theta^{(i)})^{\intercal})_{\sharp}\nu)+\min_{j=1}^{|\mathcal V|}V_{i,\pi_j}(\bm \Delta)\right),
\]
where $\mathcal V$ is the set of vertices of $\Pi(\mu,\nu)$ (recall the proof of \cref{prop:subgradient}), $F$ is semialgebraic as follows from the previous deliberations. Now, $H(\bm\Delta)=F(A(\bdelta))+P(\bdelta)$, proving that $H$ is semialgebraic.
\end{proof}

With these preparations in hand, we readily obtain the convergence of \cref{alg:sgm}.

\begin{proof}[Proof of \cref{thm:convergenceSGM}]
Given the previous deliberations
\cref{thm:convergenceSGM} is a direct consequence of Corollary 5.9 in \cite{davis2020stochastic}. Notably, while their result is stated for locally Lipschitz Whitney stratifiable functions, \cref{lem:semialgebraic} proves that  $H$ is a semialgebraic function and hence satisfies the conditions of Corollary 5.9 in \cite{davis2020stochastic}; see the discussion following that result. 

It remains to show that Assumption C in the same reference holds. Given the choice of step sequence and the fact that the subgradients are not stochastic, it suffices to show that $\sup_{k\in\mathbb N}\|\bdelta_k\|_{\mathrm{F}}<\infty$. However, since $\eta_k\leq \frac{1}{2\beta}$, \cref{prop:boundedIterates} yields that $\|\bdelta_k\|_{\mathrm{op}}\leq \sqrt{1+\|\bdelta_k^{\intercal}\bdelta_k-\bI_{d_x}\|_{\mathrm{F}}}\leq \sqrt{\frac76}$, as all norms on $\mathbb R^{d_y\times d_x}$ are equivalent, this proviso is met.

Applying Corollary 5.9 in \cite{davis2020stochastic}, we obtain that the iterates of \cref{alg:sgm} admit cluster points which are critical points of $H$.
Finally, since $\beta \geq 2L_1$, \cref{thm:stationaryPoints} asserts that any critical point $\bar\bdelta$ of $H$ satisfying $\|\bar\bdelta^{\intercal}\bar\bdelta-\bI_{d_x}\|_{\mathrm{F}}\leq \frac{1}{6}\leq \frac{\beta}{2\beta+8L_1}$ lies on the Stiefel manifold and is a Riemannian critical point for $F$ over the Stiefel manifold. However, all iterates of \cref{alg:sgm} satisfy $\|\bdelta_k^{\intercal}\bdelta_k-\bI_{d_x}\|_{\mathrm{F}}\leq {\frac16}$, see \cref{prop:boundedIterates}.  
\end{proof}

\newpage
\section{Additional Details on Experiments}
\label{app:sec:experiments_details}

The code used to generate the experimental results is available in the supplement to the submission.

\subsection{Monte-Carlo Error and Empirical Convergence Validation}
\label{app:sec:validation}

We provide additional details for the validation experiments from \cref{sec:validation}. Recall that we fix $d_x=5$ and $d_y=10$, and consider centered Gaussian distributions $\mu=\mathcal N(0,\bsigma_\mu)$ and $\nu=\mathcal N(0,\bsigma_\nu)$. The covariance matrices are generated as
\[
    \bsigma_\mu=\mathbf S_\mu^{\intercal}\mathbf S_\mu,
    \qquad
    \bsigma_\nu=\mathbf S_\nu^{\intercal}\mathbf S_\nu,
\]
where the entries of $\mathbf S_\mu\in\mathbb R^{d_x\times d_x}$ and $\mathbf S_\nu\in\mathbb R^{d_y\times d_y}$ are sampled uniformly from $[0,1]$ and are given respectively by  
\[
\begin{pmatrix}
0.3236 & 0.2191 & 0.6560 & 0.5031 & 0.8479 \\
0.9789 & 0.9872 & 0.9049 & 0.8046 & 0.1269 \\
0.0338 & 0.0987 & 0.1037 & 0.5406 & 0.2284 \\
0.5851 & 0.4209 & 0.8828 & 0.3423 & 0.2496 \\
0.7108 & 0.2808 & 0.6270 & 0.4175 & 0.8096
\end{pmatrix}
\]
and
\[
\begin{pmatrix}
0.9161 & 0.2135 & 0.7142 & 0.8724 & 0.2980 & 0.6841 & 0.4130 & 0.5256 & 0.2159 & 0.3184 \\
0.7531 & 0.5333 & 0.7351 & 0.6972 & 0.1939 & 0.8199 & 0.8531 & 0.6292 & 0.6706 & 0.8047 \\
0.9606 & 0.8050 & 0.5413 & 0.1567 & 0.3258 & 0.7775 & 0.0108 & 0.0359 & 0.5499 & 0.3627 \\
0.6861 & 0.3847 & 0.6097 & 0.7144 & 0.2869 & 0.7862 & 0.0350 & 0.9600 & 0.5924 & 0.4226 \\
0.9395 & 0.2172 & 0.7135 & 0.4205 & 0.5785 & 0.1037 & 0.5993 & 0.9797 & 0.8764 & 0.0343 \\
0.9353 & 0.6064 & 0.1920 & 0.2301 & 0.8376 & 0.5580 & 0.3768 & 0.5624 & 0.8435 & 0.3506 \\
0.7907 & 0.2041 & 0.4826 & 0.1531 & 0.5318 & 0.4466 & 0.8603 & 0.7347 & 0.2604 & 0.3680 \\
0.2773 & 0.0817 & 0.3731 & 0.1265 & 0.8599 & 0.5523 & 0.4136 & 0.7552 & 0.1519 & 0.4664 \\
0.4340 & 0.9845 & 0.7846 & 0.2799 & 0.7939 & 0.6587 & 0.5673 & 0.2789 & 0.7801 & 0.0522 \\
0.2094 & 0.5340 & 0.9646 & 0.4347 & 0.9174 & 0.6073 & 0.9301 & 0.0570 & 0.4485 & 0.8742
\end{pmatrix}
\]
where the displayed values are truncated to four digits after the decimal point.

The value of $\aIGW(\mu,\nu)$ and the corresponding population optimizer over $\St(d_x,d_y)$ are provided in closed form in \cref{ex:aIGWGaussian}. Furthermore, the IGW cost for the projected distributions can be evaluated directly from the covariance matrices. Specifically, for each $\bdelta\in\mathbb R^{d_y\times d_x}$ and $\theta\in\unitsphy$,
\[
\IGW\left((\theta^{\intercal}\bdelta)_{\sharp}\mu,(\theta^{\intercal})_{\sharp}\nu\right)^2
=
(\theta^{\intercal}\bdelta \bsigma_\mu \bdelta^{\intercal}\theta)^2
+
(\theta^{\intercal}\bsigma_\nu\theta)^2
-
2\theta^{\intercal}\bdelta \bsigma_\mu \bdelta^{\intercal}\theta \, \theta^{\intercal}\bsigma_\nu\theta,
\]
by \cite[Corollary 3.2]{dandapanthula2025optimal}. Thus, the MC validation experiment does not require discretizing $\mu$ and $\nu$. In this experiment, we set $\beta=100$, a maximum number of iterations of $2500$, and take the stepsize sequence $\eta_k = \min\{.01/\max\{\|\mathbf G_k\|,1\},5000/(k+1) \}$ for the standard subgradient method \cref{alg:sgm}. While this choice of $\beta$ and stepsize do not fall within the regime required to apply \cref{alg:sgm}, they are chosen as they were seen to lead to a good tradeoff between performance and compute time. For the Riemannian subgradient method, \cref{alg:Riemanniansgm}, we set a maximum number of iterations of $500$, a gradient tolerance of $5\times 10^{-6}$, and a maximum number of backtracking steps of $12$. The average runtime for this experiment over the $25$ runs  as a function of the number of slices with the different choices of initialization and algorithm is provided in \cref{fig:MCtimes}.   

For the validation of the sample complexity, we compute the sliced IGW distance between empirical measures. The complexity of computing one subgradient of $F$ in this setting is as follows. 

For a fixed $\theta\in\mathbb S^{d_y-1}$ and $\bm \Delta \in \St(d_x,d_y)$ when $\mu$ and $\nu$ are both uniformly distributed on $N$ points $(X_i)_{i=1}^N$ and  $(Y_j)_{j=1}^N$, respectively . Following \cref{thm:UnivariateIGW}, the optimal coupling is obtained by performing the following routine under the assumption that the projected distributions are still uniformly supported on the same number of points. 
\begin{enumerate}[leftmargin=*]
    \item Compute the projected support points for $\mu$, $x^{\theta}_i=\theta^{\intercal}\bm \Delta X_i$ for $_{i=1}$ to $N$: $O(d_xd_y + Nd_x)$, 
    \item Compute the projected support points for $\nu$, $y^{\theta}_i=\theta^{\intercal}Y_i$ for ${i=1}$ to $N$: $O(Nd_y)$,
   \item Sort the projected support points $x_i^{\theta}$ and $y_i^{\theta}$ in ascending order: $O(N\log(N))$, 
   \item Compute $v_{\id}=\left(\sum_{i=1}^Nx_i^{\theta}y_i^{\theta} \right)^2$ and $v_{\check{\id}}=\left(\sum_{i=1}^Nx_i^{\theta}y_{N+1-i}^{\theta} \right)^2$: $O(N)$, 
   \item If $v_{\id}\geq v_{\check \id}$,  $\pi^{\star}=\pi_{\id}$ is optimal for $\mathsf{IGW}((\theta^{\intercal}\bdelta)_{\sharp}\mu,(\theta^{\intercal})_{\sharp}\mu)$, else $\pi^{\star}=\pi_{\check\id}$ is optimal (see \cref{thm:UnivariateIGW})
   \item Compute an element of the subdifferential of $G^{\theta}$: $\mathbf T_{\theta}=-4  \theta^{\intercal}\bdelta \mathbf C_{\pi^{\star
   }}\theta  \left(\mathbf C_{\pi^{\star}}\theta\theta^{\intercal}\right)^{\intercal}$, where $\bm C_{\pi^{\star}} = \int xy^{\intercal} d\pi^{\star}(x,y)$, $\theta^{\intercal}\bdelta \mathbf C_{\pi^{\star
   }}\theta$ can be computed in $O(N)$ whereas $\left(\mathbf C_{\pi^{\star}}\theta\theta^{\intercal}\right)^{\intercal}$ can be computed in $O(Nd_x+d_xd_y)$
   \item Compute $\mathbf L_{\theta}=4\left((\theta)^{\intercal} \bdelta \mathbf R_{\mu}\bdelta^{\intercal} \theta\right) \left(\mathbf{R}_{\mu}\bdelta^{\intercal} \theta(\theta)^{\intercal}\right)^{\intercal}$ where $\bm R_{\mu} = \int zz^{\intercal} d\mu(z)$. Again, the constant $\left((\theta)^{\intercal} \bdelta \mathbf R_{\mu}\bdelta^{\intercal} \theta\right)$ can be computed in $O(N)$ whereas $\left(\mathbf{R}_{\mu}\bdelta^{\intercal} \theta(\theta)^{\intercal}\right)^{\intercal}$ can be computed in $O(Nd_x+d_xd_y)$.   
\end{enumerate}
In sum, the overall complexity of computing $\mathbf T_{\theta}+\mathbf R_{\theta}$ is $O(N(d_x+d_y)+d_xd_y +N\log(N))$ and, since $d_x\leq d_y$ and $\partial F(\bdelta)=\frac 1 R\sum_{i=1}^R (\mathbf T_{\theta^{(i)}}+\mathbf R_{\theta^{(i)}})$ for some collection of unit vectors $(\theta^{(i)})_{i=1}^R$, the overall per-iteration complexity for computing one subgradient is given by $O(R\left(N(d_y+\log(N))+d_xd_y \right))$. With this, computing a subgradient of $H$ at $\bm\Delta$ via \cref{prop:subgradient} requires first computing $ A(\mathbf{\Delta}) = \frac{1}{8}\mathbf{\Delta}\left(  15\mathbf{I}_{d_x} - 10\mathbf{\Delta}^{\intercal}\mathbf{\Delta}+3(\mathbf{\Delta}^{\intercal}\mathbf{\Delta})^2 \right)$ whose complexity, $O(d_x^2d_y)$, is dominated by forming $\bm \Delta^{\intercal}\bm \Delta$ (since $d_x\leq d_y$), then computing a subgradient of $F$ at $A(\bm \Delta)$ which we have just analyzed and composing that matrix with $D A_{[\bm \Delta]}$ (which consists of multiplying matrices of shape $d_y\times d_x$ or $d_x\times d_y$). Finally, the term $\beta \bm\Delta(\bm \Delta^{\intercal}\bm \Delta-\mathbf I_{d_x})$ can be computed in $O(d_yd_x^2)$ as argued above. Altogether, the complexity for computing the subgradient of $H$ for distributions on $N$ points satisfying the conditions described previously based on $R$ slices is given by $O(R(N(d_y+\log(N)))+d_x^2d_y)$.

For the sample complexity experiment, we compute the IGW distance between the one-dimensional projections of the empirical measures. If $\hat \mu_n = \frac 1n\sum_{i=1}^n \delta_{X_i}$ where $X_1,\dots, X_n\stackrel{i.i.d.}{\sim}\mu$, then for any $\theta \in \mathbb S^{d_y-1}$ and $\bdelta \in \mathrm{St}(d_x,d_y)$, $(\theta^{\intercal}\bm \Delta)_{\sharp}\hat \mu_n = \frac 1n\sum_{i=1}^n \delta_{\theta^{\intercal}\bm \Delta X_i}$. Thus, as long $\theta^{\intercal}\bm \Delta X_i\neq \theta^{\intercal}\bm \Delta X_j$ for each $i\neq j$ and similarly $\theta^{\intercal}Y_i\neq \theta^{\intercal}Y_j$ for each $i\neq j$ for the samples $Y_1,\dots,Y_n\stackrel{i.i.d.}\sim \nu$,  $\mathsf{IGW}\left((\theta^{\intercal}\bm \Delta)_{\sharp}\hat \mu_n,(\theta^{\intercal})_{\sharp}\hat \nu_n\right)$ can be computed as above. In our numerical implementation, we implement a check to see if the sorted values are pairwise distinct and warn the user if they are not -- in which case the 1d sorting formula may fail to accurately compute the IGW distance between the projected measures. In our experiments this check always passes.  

We use the same hyperparameters as the MC validation experiment for the validation of the sample complexity with a fixed number of slices, $m=3000$. The average runtime for this experiment over the $25$ runs  as a function of the number of samples with the different choices of initialization and algorithm is provided in \cref{fig:emptimes}. For practical purposes, the data was generated by running $5$ instances of the code in parallel, each running $5$ runs with different seeds to generate the samples.

\begin{figure}[!t]
    \centering
    \includegraphics[width=\linewidth]{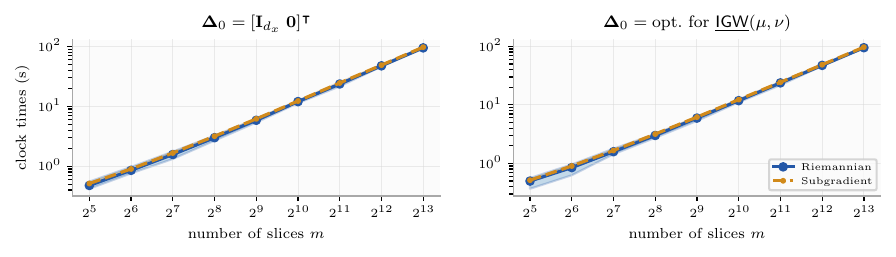}
    \vspace{-2em} 
    \caption{Average time required for the $25$ runs used to generate the data for the Monte Carlo error plot \cref{fig:MCOptim} along with a shaded region for the maximum and minimum runtime over the 25 runs. The total compute time required to generate this data is $19243$ seconds on a laptop with 64 GB of ram and an Intel i7-13620H processor.}
    \label{fig:MCtimes}
\end{figure}

\begin{figure}[!t]
    \centering
    \includegraphics[width=\linewidth]{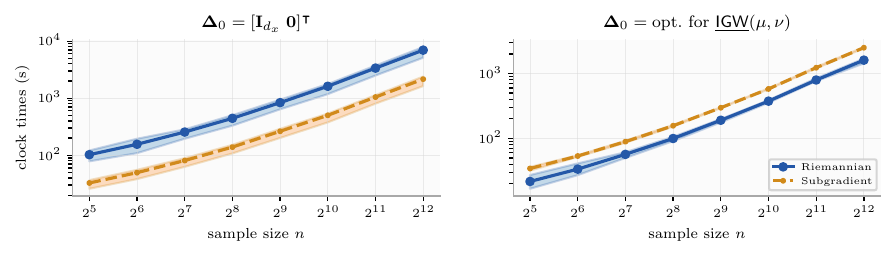}
    \vspace{-2em} 
    \caption{Average time required for the $25$ runs used to generate the data for the sample complexity plot \cref{fig:sampleCompOptim} along with a shaded region for the maximum and minimum runtime over the 25 runs. The total compute time required to generate this data is $617932$ seconds on a laptop with 32 GB of ram and a Ryzen 7 PRO 4750U processor.}
    \label{fig:emptimes}
\end{figure}

\subsection{Riemannian Subgradient Method}
\label{sec:Riemannian}

The practical Riemannian subgradient method used in the experiments is presented in \cref{alg:Riemanniansgm}. Neglecting the additional backtracking step, the per-iteration complexity of this method scales in the same way as the standard subgradient method in $d_x,d_y,N,$ and $R$; see the  discussion in the previous section and note that the $QR$ decomposition can be implemented in $O(d_yd_x^2)$. Though the convergence of this algorithm is not guaranteed, it is seen to yield similar values to \cref{alg:sgm}.

\begin{algorithm}[!h]
\caption{Riemannian subgradient method for sliced IGW}
\label{alg:Riemanniansgm}
\begin{algorithmic}[1]
\Statex Fix initialization $\bdelta_1\in\St(d_x,d_y)$, total number of iterations $0\leq K\in\mathbb N$, maximum number of backtracking steps $0\leq L\in\mathbb N$, backtracking parameter $\alpha>0$, and tolerance $\epsilon>0$.
\For{$k=1,\dots,K$}
\State Choose $\bS_k\in\partial F(\bdelta_k)$
\State $\bG_k\gets \bS_k-\frac{1}{2}\bdelta_k(\bdelta_k^{\intercal}\bS_k+\bS_k^{\intercal}\bdelta_k)$ \hfill(element of $\partial_R F(\bdelta_k)$)
\If{$\|\bG_k\|_{\mathrm F}<\epsilon$}
\State \Return Riemannian $\epsilon$-critical point $\bdelta_k$
\EndIf
\State $\eta\gets 1$ and $\mathrm{accepted}\gets \mathrm{false}$
\For{$l=1,\dots,L$}
    \State $\tilde\bdelta\gets \mathrm{proj}_{\St(d_x,d_y)}(\bdelta_k-\eta\bG_k)$
    \If{$F(\tilde\bdelta)\leq F(\bdelta_k)-\alpha\eta\|\bG_k\|_{\mathrm F}^2$}
        \State $\mathrm{accepted}\gets \mathrm{true}$
        \State \textbf{break}
    \Else
        \State $\eta\gets \eta/2$
    \EndIf
\EndFor
\If{$\mathrm{accepted}=\mathrm{false}$}
\State \Return Optimizer stuck at non-critical point $\bdelta_k$
\EndIf
\State $\bdelta_{k+1}\gets \tilde\bdelta$
\EndFor
\State \Return $\bdelta_{K+1}$
\end{algorithmic}
\end{algorithm}

The operation $\mathrm{proj}_{\St(d_x,d_y)}$ takes a matrix of the form $\bdelta+\bP$, where $\bdelta\in\St(d_x,d_y)$ and $\bP$ is in the tangent space to $\St(d_x,d_y)$ at $\bdelta$, and returns an element of $\St(d_x,d_y)$. This corresponds to the standard QR-based retraction onto the Stiefel manifold; see, for instance, Example 4.1.3 in \cite{absil2009optimization}. The implementation used in our experiments is given in \cref{alg:projection}.

\begin{algorithm}[!h]
\caption{$\mathrm{proj}_{\St(d_x,d_y)}(\bdelta+\bP)$}
\label{alg:projection}
\begin{algorithmic}[1]
\State Compute a QR decomposition $\bdelta+\bP=\mathbf Q\mathbf R$, with $\mathbf R$ having strictly positive diagonal
\State \Return $\mathbf Q$
\end{algorithmic}
\end{algorithm}

\subsection{LLM Representation Comparison}
\label{app:sec:LLM}

The representation-comparison experiment loads the two text datasets, \texttt{ag\_news} and \texttt{amazon\_polarity}, from the Hugging Face Datasets package for Python \cite{lhoest2021datasets}. The former dataset includes no licensing information on Hugging Face, but is stated to be for non-comercial use only on the original dataset's website \cite{agnews2}. The latter dataset is listed under the apache-2.0 license.   

As noted in the text, we use \texttt{bert-base-uncased} as the teacher model and the $23$ distilled BERT student models from \cite{turc2020wellread} as provided in the Hugging Face transformers package for Python \cite{wolf-etal-2020-transformers}. All of these models are listed under the apache-2.0 license on Hugging Face. For each model and each dataset, we embed the same collection of $800$ text passages. Texts are tokenized using the corresponding model tokenizer and represented by attention-masked mean pooling of the last hidden layer. The resulting embeddings are centered before computing the pairwise distances.

For a student-teacher pair, let $X^S$ and $Y^T$ denote the centered embedding matrices, whose rows correspond to the same text passages embedded by the student and teacher models, respectively. We let $\mu^S$ and $\nu^T$ denote the empirical distributions over the rows of $X^S$ and $Y^T$. The empirical sliced IGW value $\aIGW(\mu^S,\nu^T)$ is computed using $m=3000$ slices via the Riemannian subgradient method described above. We choose this method over the standard subgradient method from the text, \cref{alg:sgm}, as it does not require a choice of regularization parameter $\beta$ and does not require a choice of stepsize. The algorithm uses the same implementation for the gradient and distance computations as described in \cref{app:sec:validation} with $m=500$ slices and a maximum number of iterations of $200$.   The initialization is the Gaussian alignment from \cref{ex:aIGWGaussian}, applied to the empirical covariance matrices $\bsigma_S$ and $\bsigma_T$.

We compare the empirical sliced distance with $\aIGW(\cN(0,\bsigma_S),\cN(0,\bsigma_T))$ and $\IGW(\cN(0,\bsigma_S),\cN(0,\bsigma_T))$, which admit the closed form expressions
\[
\aIGW(\cN(0,\bsigma_S),\cN(0,\bsigma_T))^2 = \frac{\big(\Tr( \bsigma_S) - \Tr( \bsigma_T)\big)^2+2\sum_{i=1}^{d_y}\big(\lambda_i(\bsigma_S)-\lambda_i(\bsigma_T)\big)^2}{d_y(d_y+2)},
\]
see \cref{ex:aIGWGaussian}, and 
\[
 \IGW(\mathcal N(0,\bsigma_S),\mathcal N(0,\bsigma_T))^2 = \sum_{i=1}^{d_y}\big(\lambda_i(\bsigma_S)-\lambda_i(\bsigma_T)\big)^2,  
\]
see Corollary 3.2 in \cite{dandapanthula2025optimal}. We recall that, by convention,  $\lambda_i(\bsigma_{\mu}) =0 $ for each $i>d_x$ above. We underscore that $\bsigma_S$ and $\bsigma_T$ are the empirical covariance matrices from the data and so work well when the datasets are approximately normal.  
We also compare with the  centered kernel alignment,
\begin{equation}
\label{eq:CKA}
    \mathrm{CKA}(X^S,Y^T)
    =
    \frac{\|(Y^T)^{\intercal}X^S\|_{\mathrm F}^2}
    {\|(X^S)^{\intercal}X^S\|_{\mathrm F}\|(Y^T)^{\intercal}Y^T\|_{\mathrm F}},
\end{equation}
and report $1-\mathrm{CKA}(X^S,Y^T)$ so that smaller values correspond to stronger agreement with the teacher representation. We report the runtime of each metric for each student-teacher pair and for each choice of dataset in \cref{fig:LLMruntime}.

\begin{figure}
    \centering
    \includegraphics[width=\linewidth]{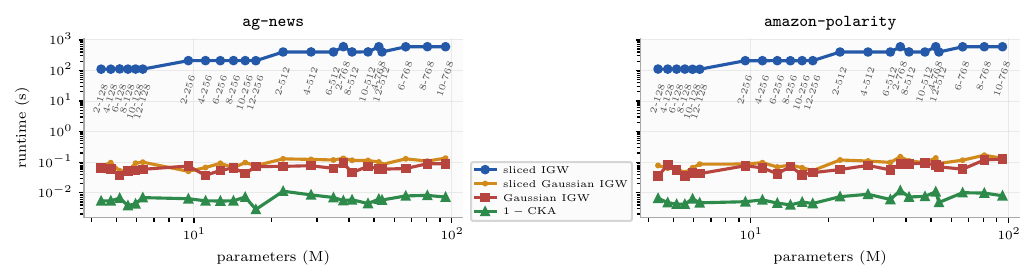}
    \vspace{-2em}
    \caption{Time required to compute the various metrics between the student  and teacher embeddings to generate the data presented in \cref{fig:LLM}. The total compute time required to generate this data is $15643$ seconds on a laptop with 64 GB of ram and an Intel i7-13620H processor; generating the embeddings took $1016$ seconds whereas computing the values of the metrics took $14627$ seconds in total.}
    \label{fig:LLMruntime}
\end{figure}

\subsection{Heterogeneous User Clustering}
\label{app:sec:clustering}

In addition to the \texttt{ag\_news} dataset considered previously, this experiment also uses the \texttt{imdb} and \texttt{arXiv-summarization} text datasets loaded using the  Hugging Face Datasets package for Python \cite{lhoest2021datasets}. The \texttt{arXiv-summarization}   dataset entry on Hugging Face lists no license details whereas the \texttt{imdb} dataset entry lists its license as ``other''. We note that the original repo for \texttt{arXiv-summarization} and the accompanying paper \cite{arxiv}, \url{https://github.com/armancohan/long-summarization}, downloads the datasets directly from the arXiv OpenAcess repository and that all metadata  from arXiv (including the abstracts in the dataset) have the CC0 1.0 license. As for the \texttt{imdb} dataset, the original paper \cite{imdb} states that ``In the interest of providing a benchmark for future work in this area, we release this dataset to the public''; we understand our use of the dataset to be in line with this declaration.  

Regarding the models used to generate the embeddings, we utilize \texttt{all-MiniLM-L6-v2} and \texttt{all-mpnet-base-v2} from the Hugging Face transformers Python package \cite{wolf-etal-2020-transformers}, which generate $384$ and $768$-dimensional embeddings respectively. These models are both listed under apache-2.0 licenses on Hugging Face. The embeddings are generated in the same manner as described in \cref{app:sec:LLM} following the homogeneous and heterogeneous settings described in the text. 

We underscore that the styles are ordered according to an index $0,1,2$ and, to form the users for the style at index $0$, say, the primary-style documents are assigned from a seeded random permutation without replacement. Then, a number, $k$, between $0$ and $10$ is selected uniformly at random which corresponds to the number of samples from the non-primary style $1$, and finally $10-k$ documents are selected from the non-primary style $2$.

With the simulated users in hand, we compute the pairwise distributions between the embedded datasets again using the same implementation of the Riemannian subgradient method described before, but with  
 $m=200$ slices and a maximum number of iterations of $200$ to speed up computation. This enables us to construct a pairwise distance matrix $\mathbf D\in \mathbb R^{24\times 24}$ to which we apply self-tuning spectral clustering as described in the main text to cluster the data. To measure the accuracy of the clustering, we report the adjusted Rand index \cite{hubert1985comparing} which is a standard chance-corrected agreement score between the predicted and ground-truth partitions. We also report purity which is the fraction of users assigned to the majority ground-truth label within their predicted cluster summed over predicted clusters. The MDS coordinates in \cref{fig:hetero_clustering} are used only for visualization and are computed from the sliced-IGW distance matrix in each setting. Clustering is performed from each method's own pairwise distance matrix, not from the MDS coordinates. Prediction labels are permutation-aligned to the ground-truth style labels for display.

 We perform the same procedure using the pairwise distances generated by the sliced IGW between Gaussians whose matrices are given by the empirical covariance matrices by complete analogy with \cref{app:sec:LLM}.

These experiments were performed on a  laptop with 64 GB of ram and an Intel i7-13620H
processor. We note that the time required to compute each distance is not as straightforward to plot as in the previous examples since the dimensions vary, leading to drastically different runtimes. For the three seeds considered, the pairwise sliced $\mathrm{IGW}$ distances took a total of $10453, 34423,$ and $34696$ seconds to compute, respectively, in the homogeneous setting and $84504,71983,$ and $62987$ seconds in the heterogeneous. The corresponding Gaussian IGW distances took a total of   
$6, 39,$ and $36$ seconds in the homogeneous case versus 
$41,
58,$ and $45$ seconds in the heterogeneous case. In addition to these runs, the embeddings took 1344 seconds to generate. In sum, this experiment required a total of $300615$ seconds of compute time. As in the validation of the sample complexity, we run 10 Python~instances in parallel to compute the pairwise distances in order to reduce the  real world time spent on this experiment.
\end{document}